\documentclass{article}

\usepackage{amsmath,amssymb,amsthm,mathtools,bm}
\usepackage[inline]{enumitem}
\usepackage{MnSymbol}
\usepackage[dvipsnames,table]{xcolor}
\usepackage{hyperref}
\usepackage{booktabs,tabularx}
\usepackage{mathtools}
\usepackage{multirow}
\usepackage{makecell}
\usepackage{booktabs}
\usepackage{algorithm}
\usepackage{algorithmic}
\usepackage{subcaption}
\usepackage[labelfont=bf]{caption}
\usepackage{wrapfig}
\usepackage{titletoc}
\usepackage{scalerel}

\PassOptionsToPackage{sort,compress}{natbib}
\usepackage[main,preprint]{neurips_2026}

\usepackage{etoolbox}
\makeatletter
\patchcmd{\hyper@makecurrent}{%
    \ifx\Hy@param\Hy@chapterstring
        \let\Hy@param\Hy@chapapp
    \fi
}{%
    \iftoggle{inappendix}{
        \@checkappendixparam{chapter}%
        \@checkappendixparam{section}%
        \@checkappendixparam{subsection}%
        \@checkappendixparam{subsubsection}%
        \@checkappendixparam{paragraph}%
        \@checkappendixparam{subparagraph}%
    }{}%
}{}{\errmessage{failed to patch}}

\newcommand*{\@checkappendixparam}[1]{%
    \def\@checkappendixparamtmp{#1}%
    \ifx\Hy@param\@checkappendixparamtmp
        \let\Hy@param\Hy@appendixstring
    \fi
}
\makeatletter

\newtoggle{inappendix}
\togglefalse{inappendix}

\apptocmd{\appendix}{\toggletrue{inappendix}}{}{\errmessage{failed to patch}}

\definecolor{darkred}{rgb}{0.55, 0.0, 0.0}
\definecolor{darkblue}{rgb}{0.0, 0.0, 0.55}
\definecolor{lightgreen}{RGB}{217,242,217}
\definecolor{lightred}{RGB}{245,216,216}
\definecolor{lightyellow}{RGB}{250,242,204}
\hypersetup{
    colorlinks=true,
    linkcolor=darkblue,
    citecolor=darkred,
    urlcolor=darkblue
}

\newcolumntype{L}{>{\raggedright\arraybackslash}X}

\DeclareMathOperator*{\argmin}{argmin}
\DeclareMathOperator*{\argmax}{argmax}
\DeclareMathOperator*{\argsup}{argsup}

\newcommand{\leg}[1]{#1^{\dagger}}  
\newcommand{\possexp}{\mathbb{E}^{\star}}  
\newcommand{\mode}[1]{#1^{\star}}  
\newcommand{\loss}{\hat{\ell}}  
\newcommand{\given}{\,|\,}  

\newcommand{\lr}[3]{\left#1#2\right#3}
\newcommand{\rb}[1]{\lr{(}{#1}{)}}
\renewcommand{\sb}[1]{\lr{[}{#1}{]}}
\newcommand{\cb}[1]{\lr{\{}{#1}{\}}}

\newtheorem{theorem}{Theorem}
\newtheorem{proposition}{Proposition}

\theoremstyle{remark}
\newtheorem{remark}{Remark}
\theoremstyle{definition}
\newtheorem{example}{Example}

\newcommand{\graycell}[1]{\cellcolor{gray!10}#1}

\title{Maxitive Donsker-Varadhan Formulation for Possibilistic Variational Inference
}
\author{
Jasraj Singh\thanks{Equal contribution.} \\
Nanyang Technological University \\
\And Shelvia Wongso$^*$ \\
Nanyang Technological University
\AND Jeremie Houssineau \\
Nanyang Technological University \\
\And Badr-Eddine Ch\'erief-Abdellatif \\
Sorbonne Universit\'e
}
\date{}

\begin{document}

\maketitle









\begin{abstract}
    Variational inference (VI) is a cornerstone of modern Bayesian learning, enabling approximate inference in complex models. However, its formulation depends on expectations and divergences defined through high-dimensional integrals, often rendering analytical treatment impossible and necessitating heavy reliance on approximations. Possibility theory, an imprecise probability framework, allows us to directly model epistemic uncertainty instead of relying on a subjective interpretation of probabilities. While this framework provides robustness and interpretability under sparse or imprecise information, adapting VI to the possibilistic setting requires rethinking core concepts such as divergences, which presuppose additivity. In this work, we develop a principled formulation for performing possibilistic VI by establishing a maxitive analogue of the classical Donsker-Varadhan formulation. The resulting framework enables us to derive a learning rule for possibilistic VI with exponential-family candidates and practical update rules for neural-network training, giving rise to a family of optimizers termed \texttt{CBOpt}. Finally, we demonstrate that \texttt{CBOpt}\footnote{An anonymized implementation is available at \href{https://anonymous.4open.science/r/cbopt-80D1}{https://anonymous.4open.science/r/cbopt-80D1} for review purposes; the de-anonymized repository will be released upon publication.} achieves competitive performance on both in-domain and out-of-domain image classification tasks.
\end{abstract}

\section{Introduction}

Variational inference (VI) has become a cornerstone of modern Bayesian learning, enabling approximate inference in complex models that were previously intractable. For many years, purely Bayesian methods were hampered by computational cost, but advances in Monte Carlo and especially variational approximations have brought them within reach \citep{blei2017vi-review,salimans2015mcmc-vi-bridge}. More recently, variational learning methods have been shown to perform well even at large scale and with little additional computational overhead \citep{shen2024variational}. This motivates further development of variational learning methods, given their ability to naturally provide uncertainty estimates. In VI, one typically posits a simple (parametric) variational family $\mathcal{Q}$ and then finds the probability distribution $q \in \mathcal{Q}$ that best approximates the true posterior $
\mode{q}_{\mathrm{add}}$ by maximizing the evidence lower bound (ELBO). 
This framework underlies many tools today, such as variational autoencoders \citep{kingma2022autoencodingvariationalbayes} and Bayesian neural networks \citep{blundell2015weight-uncertainty,kingma2015reparameterization-trick}, and blurs the line between tractable and intractable Bayesian analyses.

Introduced by \citet{zadeh1978fuzzy}, possibility theory, an imprecise probability framework, offers a complementary approach for modelling uncertainty. Unlike probability, possibility does not require additivity: events have a degree of plausibility ($\leq 1$) and a dual \textit{necessity}, with logical rules based on min/max operators. This allows handling epistemic uncertainty naturally. For example, the analogue of probability distributions, referred to as \emph{possibility functions}, can easily be made fully uninformative \citep{hieu2025decoupling}, a difficult endeavour in probability theory. 
The flexibility and scalability of possibility theory makes it suitable for challenging applications such as control problems under uncertainty \citep{xue2025orbit} and point process-based inference \citep{houssineau2021linear}.

A key challenge arises when we try to adapt variational methods to this possibilistic setting. Standard VI relies on notions like Shannon entropy, which do not directly carry over: for example, the ELBO includes an entropy term, but there is no obvious ``possibilistic entropy'' analogous to the Shannon entropy of a probability distribution \citep{shannon1949mathematical}. Similarly, divergences like Kullback–Leibler (KL) require additive measures \citep{kullback1951kld}. As a result, traditional VI objectives cannot be used out-of-the-box. Nonetheless, recent work has sought approximations in the possibilistic context. For instance, \citet{cella2025vi-im} developed a variational-like approximation for inferential models \citep{martin2013im}, and use a Monte Carlo-based strategy to search over the chosen family of possibility functions. In that spirit, our goal is to formulate a VI framework for general possibilistic models: we define an objective that aligns a tractable candidate possibility function with the (target) posterior possibility function, and which can be optimized in lieu of maximizing a probabilistic ELBO.

\paragraph{Contributions.} In this work, we present a novel maxitive analogue of the classical Donsker–Varadhan (DV) formulation (\autoref{thm:maxitive_DV}), which admits two dual variational characterizations whose optima are possibility functions that respectively lower- and upper-bound the possibilistic posterior. The associated objective, the consistency bound (CBO), is built on a maxitive pseudo-divergence that plays a role analogous to the KL divergence in standard VI. Building on this formulation, we develop a practical framework for possibilistic VI within a special class of exponential-family possibility functions, highlighting their structural parallels with the probabilistic case and giving rise to special mathematical structures. From this framework, we derive a learning rule that leads to a new family of optimizers, \texttt{CBOpt}. We empirically validate these methods on both in-domain and out-of-domain image classification tasks, demonstrating competitive predictive performance together with reliable uncertainty estimates. These developments provide a foundation for extending variational reasoning to possibility theory, offering new avenues for inference under epistemic uncertainty. Components of this framework, specifically the maxitive pseudo-divergence and the maxitive posterior under the  uninformative prior, have already been leveraged by \citet{ni2026possibilistic} to derive a second-order predictor for image classification; the underlying theory, including the maxitive DV duality, the dual CBO formulations, and the exponential-family treatment, is developed in full here for the first time.

\paragraph{Outline.} The remainder of the paper is organized as follows. \autoref{sec:primer-possibility} reviews the necessary background on possibility theory, and \autoref{sec:probabilistic-vi} recalls the classical Donsker–Varadhan formulation and its connection to Bayesian variational inference. \autoref{sec:possibilistic-vi} establishes our maxitive DV analogue and the dual CBO formulations. \autoref{sec:ExpFam} develops the exponential-family treatment and derives the learning rule underlying \texttt{CBOpt}. \autoref{sec:experiments} reports our empirical evaluation, and \autoref{sec:discussion} concludes with a discussion of limitations and future directions. Additional related work and notation are provided in \autoref{app:related} and \autoref{app:notation}, respectively.

\section{Primer on Possibility Theory}
\label{sec:primer-possibility}

Possibility theory provides a dedicated representation of epistemic uncertainty where, like probabilities, each event is assigned a degree of possibility between $0$ and $1$ and, unlike probabilities, the main operation on possibilities is optimization. In possibility theory, random variables are replaced by \emph{uncertain variables} and operations like conditioning and marginalization are defined similarly but with a maximum or supremum rather than a sum or integral, as stated more formally in \autoref{app:possibilityTheory}.

The expected value for a possibility function $f$ is $\possexp_{f}[\bm{\theta}] \doteq \argmax_{\theta \in \Theta} f(\theta)$, which is a set in general, and satisfies $\possexp_f[T(\bm{\theta})] = T( \possexp_f[\bm{\theta}] )$ for any mapping $T$, a property shared with the maximum likelihood estimate. When $\possexp_{f}[\bm{\theta}]$ is a singleton $\{\mode{\theta} \}$, we do not make a distinction between this singleton and the element $\mode{\theta}$. If $f$ is twice differentiable at $\mode{\theta}$, we can define the precision as
$\mathcal{I}_f( \bm{\theta} ) \doteq \possexp_f[ -\nabla^2 \log f(\bm{\theta}) ] = -\nabla^2 \log f(\theta)|_{\theta = \mode{\theta}}$, where the operator $\nabla^2$ is the Hessian and \emph{precision} is to be understood as an inverse covariance matrix. The \emph{normal possibility function} with expected value $\mu$ and covariance matrix $\Sigma \succ 0$ is defined as
$f(\theta) = \overline{\mathrm{N}}\left(\theta;\mu,\Sigma \right) \doteq \exp( -\tfrac{1}{2} (\theta-\mu)^{\intercal}\Sigma^{-1}(\theta-\mu) )$. It verifies $\possexp_f[\bm{\theta}] = \mu$ and $\mathcal{I}_f(\bm{\theta}) = \Sigma^{-1}$, as is leveraged in the Gaussian (Laplace) approximation. 

We define the set $\mathcal{F}(\Theta) \doteq \{ f : \Theta \to [0,1] \; : \; \sup_{\theta \in \Theta} f(\theta) = 1 \}$ of possibility functions over $\Theta$. As opposed to the set $\mathcal{P}(\Theta)$ of probability distributions, $\mathcal{F}(\Theta)$ is a pre-ordered set when equipped with the partial order $\preceq$ defined as
$f \preceq g \iff f(\theta) \leq g(\theta), \forall \theta \in \Theta$. 
There is a greatest element in $\mathcal{F}(\Theta)$, namely the function equal to $1$ everywhere, which we denote by $\bm{1}$. For any subset $\mathcal{G}$ of $\mathcal{F}(\Theta)$, $\max \mathcal{G}$ denotes the maximal element of $\mathcal{G}$. The minimal element $\min \mathcal{G}$ can also be considered when it exists. 
\autoref{tab:notations} summarizes other key notations used in this paper.

\section{Bayesian Inference and the Donsker-Varadhan Variational Formula}
\label{sec:probabilistic-vi}

We first provide a review of the probabilistic variational formulation of Bayesian inference and its variational approximations. We consider the problem of learning about an unknown parameter $\theta_0$ in a set $\Theta$, which can be thought of either as the true parameter in a statistical procedure or as the solution of an optimization problem. In this context, it is usual to alternate between
\begin{enumerate*}[label=\roman*)]
    \item the optimization viewpoint where the main objects are a loss $\ell$ and a regularizer $R$ on $\Theta$, and
    \item the inference viewpoint where the main objects are a likelihood $L \propto \exp(- \ell)$, that is, $L(\theta) \propto \exp(- \ell(\theta))$ for all $\theta \in \Theta$, and a prior $\pi$ with density $\propto \exp(- R)$ with respect to Lebesgue's measure.
\end{enumerate*}
This leads to the (generalized) Bayesian posterior $q^{\star}_{\textnormal{add}}$ for the regularized loss $\loss \doteq \ell + R$ with respect to $\pi$ as
\begin{equation}
q^{\star}_{\textnormal{add}}(\mathrm{d}\theta) \;\doteq\; \dfrac{\exp(-\ell(\theta))\,\pi(\mathrm{d}\theta)}{\int \exp(-\ell(\theta'))\pi(\mathrm{d}\theta')} = \dfrac{\exp\left(-\loss(\theta)\right)\mathrm{d}\theta}{\int \exp\left(-\loss(\theta')\right) \mathrm{d}\theta'}.
\end{equation}

We start with the following classical variational formula, known since at least the 1950s (see e.g.\ \citet[Exercise 8.28]{kullback1959information} for the finite case), generally attributed in the general setting to \citet{donsker1983vi}, and later rediscovered in statistics by \citet{zellner1988optimal}. We refer the interested reader to \citet[Page~160]{catoni2004} for a proof.

\begin{theorem}[Donsker and Varadhan's variational formula]
\label{additive_DV_general}
Let $(\Theta,\mathcal{T})$ be a measurable space and let $\nu$ be a probability measure on $\Theta$.
For any measurable function $h:\Theta\to\mathbb{R}$ such that $\int e^{h} \mathrm{d}\nu<+\infty$, we have
\begin{equation}
\log \int e^{h} \mathrm{d}\nu
= \sup_{\rho\in\mathcal{P}(\Theta)}\left\{  \int h \mathrm{d}\rho - \mathrm{KL}\left(\rho\Vert\nu\right)\right\},
\end{equation}
where $\mathcal{P}(\Theta)$ denotes the set of probability measures on $\Theta$ and
$\mathrm{KL}(\rho\Vert\nu)$ is the Kullback--Leibler divergence of $\rho$ with respect to $\nu$, with the convention $\infty-\infty=-\infty$. Moreover, the supremum on the right-hand side is achieved exactly at the Gibbs measure $\nu_h$ whose density with respect to $\nu$ is
\begin{align}
\frac{\mathrm{d}\nu_h}{\mathrm{d}\nu}(\theta)
\doteq \frac{e^{h(\theta)}}{\int e^{h} \mathrm{d}\nu} .
\end{align}
\end{theorem}

\paragraph{(Generalized) Bayesian inference via Donsker-Varadhan.}
\autoref{additive_DV_general} provides an information-theoretic characterization of the Bayesian posterior. By setting $h=-\ell$ and $\nu=\pi$ (the prior), we obtain that the normalizing constant $Z_{\textnormal{add}} = \int \exp(-\ell)\,\mathrm{d}\pi$ (also called the \emph{evidence} in the statistical setting) satisfies
\begin{align}
\log Z_{\textnormal{add}} &
= \sup_{q\in\mathcal{P}(\Theta)} \mathrm{ELBO}(q) \, , 
\qquad \text{with} \qquad
\mathrm{ELBO}(q) 
\doteq - \mathbb{E}_{\theta\sim q}[\ell(\theta)] - \mathrm{KL}(q \Vert \pi) .
\end{align}
The ELBO (\emph{Evidence Lower Bound}) provides a lower bound on the log-evidence; the exact posterior can then be expressed as the solution of the infinite-dimensional optimization problem
\begin{align}
\label{eq:bayes_variational}
q^\star_{\textnormal{add}} = \argmin_{q\in\mathcal{P}(\Theta)} \Big\{ \mathbb{E}_{\theta\sim q}[\ell(\theta)] + \mathrm{KL}(q\Vert \pi) \Big\} .
\end{align}
Thus, Donsker and Varadhan's lemma simultaneously identifies the Bayesian posterior as the optimizer of a regularized expected loss (the negative ELBO) and expresses the log-evidence as the supremum of the ELBO over all probability measures.

\paragraph{(Generalized) Variational inference via Donsker-Varadhan.}
In practice, computing $q^\star_{\textnormal{add}}$ exactly is generally intractable. \emph{Generalized Variational Inference} (GVI) consists in restricting the optimization to a tractable family $\mathcal{Q}\subset\mathcal{P}(\Theta)$, turning the infinite-dimensional problem \eqref{eq:bayes_variational} into a simpler, typically parametric, optimization problem.  

This restriction naturally induces a gap between the true log-evidence and the ELBO achievable within $\mathcal{Q}$, which can be quantified via $\log Z_{\textnormal{add}} = \mathrm{ELBO}(q) + \mathrm{KL}(q \,\Vert\, q^\star_{\textnormal{add}})$, for all $q\in\mathcal{Q}$. Maximizing the ELBO within $\mathcal{Q}$ therefore corresponds exactly to minimizing the KL divergence to the exact posterior restricted to the chosen family. In this sense, GVI can be interpreted both as the best tractable approximation (in the reverse KL sense) of the exact posterior within $\mathcal{Q}$, and as the optimal information-processing rule within $\mathcal{Q}$.

\section{Possibilistic Inference and a Maxitive Donsker-Varadhan Analogue}
\label{sec:possibilistic-vi}


In this section, we establish a maxitive analogue of the Donsker–Varadhan variational formula and use it to derive a possibilistic inference framework. We first observe that many standard regularizers naturally attain their minimum at $0$, so that the prior $\pi = \exp(- R)$ is directly a possibility function whereas $\pi$ could be improper as a probability distribution if $\exp(-R)$ is not integrable. This leads to the following formula for the (generalized) posterior $g^{\star}_{\textnormal{max}}$ in possibility theory:
\begin{equation}
\label{def:gibbs:max}
g^{\star}_{\textnormal{max}}(\theta) \;\doteq\; \dfrac{\exp(-\ell(\theta))\pi(\theta)}{\sup_{\theta'} \exp(-\ell(\theta')) \pi(\theta')} = \dfrac{\exp(-\loss(\theta))}{\sup_{\theta'} \exp(-\loss(\theta'))} \, .
\end{equation}

The function $\pi(\theta)$ now denotes a prior possibility function over $\Theta$, typically $\pi(\theta)=\exp(-R(\theta))$.  
The normalizing constant in the maxitive posterior is defined as
\begin{equation}
Z_{\textnormal{max}} \doteq \sup_{\theta'\in\Theta} \exp(-\ell(\theta')) \pi(\theta') \, ,
\end{equation}
and is referred to as the maxitive marginal likelihood. 
Contrary to its probabilistic analogue $Z_{\textnormal{add}}$, which measures the overall fit of the model to the data, the quantity $Z_{\textnormal{max}}$ quantifies the \emph{consistency} between the prior and the likelihood. This notion is central in robust inference and can be used, for instance, to detect outliers \citep{houssineau2022robust}. 


\begin{theorem}[Maxitive Donsker-Varadhan formula]
\label{thm:maxitive_DV}
Let $\pi:\Theta\to[0,1]$ be a possibility function taking positive values. For any function $\ell:\Theta\to\mathbb{R}$, we have
\begin{subequations}
\begin{align}
\label{eq:dv:max:supinf}
\log \sup_{\theta \in \Theta} e^{-\ell(\theta)} \pi(\theta)
&= \sup_{g\in\mathcal{F}(\Theta)} \inf_{\theta \in \Theta}
\left\{ -\ell(\theta) - \log\!\left(\frac{g(\theta)}{\pi(\theta)}\right) \right\},\\
\label{eq:dv:max:infsup}
&= \inf_{g\in\mathcal{F}(\Theta)} \sup_{\theta \in \Theta}
\left\{ -\ell(\theta) - \log\!\left(\frac{g(\theta)}{\pi(\theta)}\right) \right\} \, .
\end{align}
\end{subequations}

Moreover, the supremum in \eqref{eq:dv:max:supinf} is achieved by any possibility function lower-bounding the posterior $g^{\star}_{\textnormal{max}}$, that is
\begin{equation}
\label{eq:gibbs:max:supinf}
\argmax_{g\in\mathcal{F}(\Theta)}
\inf_{\theta\in\Theta}
\left\{ -\ell(\theta) - \log\!\left(\frac{g(\theta)}{\pi(\theta)}\right) \right\}
= \big\{ g \in \mathcal{F}(\Theta) : g \preceq g^{\star}_{\textnormal{max}} \big\}.
\end{equation}
Similarly, the infimum in \eqref{eq:dv:max:infsup} is achieved by any possibility function upper-bounding $g^{\star}_{\textnormal{max}}$:
\begin{equation}
\label{eq:gibbs:max:infsup}
\argmin_{g\in\mathcal{F}(\Theta)}
\sup_{\theta\in\Theta}
\left\{ -\ell(\theta) - \log\!\left(\frac{g(\theta)}{\pi(\theta)}\right) \right\}
= \big\{ g \in \mathcal{F}(\Theta) : g^{\star}_{\textnormal{max}} \preceq g \big\}.
\end{equation}
\end{theorem}

The proof is provided in \autoref{proof:thm:maxitive_DV}. Note that if $\pi(\theta)=0$ for some values of $\theta$, then \eqref{eq:dv:max:supinf} and \eqref{eq:dv:max:infsup} still hold when just considering suprema/infima over the support $\Theta_\pi$ of $\pi$ instead of $\Theta$, and over the set $\mathcal{F}_\pi(\Theta)$ of possibility functions whose support is included in $\Theta_\pi$ instead of $\mathcal{F}(\Theta)$.


On the left-hand side of \eqref{eq:dv:max:supinf} appears the log-consistency $\log Z_{\textnormal{max}}$, while the quantity optimized on the right-hand side defines a \emph{consistency bound} (CBO).  
Two dual versions can be introduced:
\begin{align}
\underline{\mathrm{CBO}}(g)
\doteq \inf_\theta \left\{ -\ell(\theta) - \log \frac{g(\theta)}{\pi(\theta)} \right\}, 
\qquad 
\overline{\mathrm{CBO}}(g)
\doteq \sup_\theta \left\{ -\ell(\theta) - \log \frac{g(\theta)}{\pi(\theta)} \right\}.
\end{align}
These provide lower and upper bounds on the log-consistency, as shown in \autoref{thm:maxitive_DV}.  
For any $g\in\mathcal{F}(\Theta)$, we even have the decomposition
\begin{subequations}
\label{eq:min-divergence}
\begin{align}
    \log Z_{\textnormal{max}}
    &= \underline{\mathrm{CBO}}(g)
    + D_{\textnormal{max}}(g \,\|\, g^{\star}_{\textnormal{max}}) \\ 
    &= \overline{\mathrm{CBO}}(g)
    - D_{\textnormal{max}}(g^{\star}_{\textnormal{max}} \,\|\, g),
\end{align}
\end{subequations}
where the \emph{max-relative entropy} between two possibility functions is defined as
\begin{align}
\label{eq:max-rel-ent}
D_{\textnormal{max}}(g \,\|\, f)
\doteq \sup_{\theta \in \Theta} \log \frac{g(\theta)}{f(\theta)} \geq 0,
\end{align}
with the convention that the supremum is taken over all $\theta$ such that $f(\theta)>0$, and that $D_{\textnormal{max}}(g \,\|\, f)=+\infty$ whenever $\operatorname{supp}(g)\nsubseteq\operatorname{supp}(f)$.
The maxitive pseudo-divergence $D_{\max}$ has subsequently been used by \cite{ni2026possibilistic} to derive a second-order predictor for image classification.

Maximizing (resp.\ minimizing) the lower (resp.\ upper) CBO with respect to $g$ is equivalent to minimizing the max-relative entropy to the Gibbs possibility function $g^{\star}_{\textnormal{max}}$.  
Any maximizer $g$ in \eqref{eq:dv:max:supinf} satisfies $g \preceq g^{\star}_{\textnormal{max}}$, whereas any minimizer $g$ in \eqref{eq:dv:max:infsup} satisfies $g \succeq g^{\star}_{\textnormal{max}}$.  
In particular, the standard posterior $g^{\star}_{\textnormal{max}}$ is always an optimizer, and pointwise bounds all other solutions (below or above, depending on the formulation). We visualize the behaviour of max-relative entropy in \autoref{app:intuition}.

\textbf{Canonical choice of maximizer:} Among all $g$ that optimize the CBO, a natural canonical representative is the standard possibilistic posterior $g^{\star}_{\textnormal{max}}$. However, one must be careful with terminology: $g^{\star}_{\textnormal{max}}$ is an extremum of the CBO, which means it assigns the largest/smallest possibility degrees and is therefore the most/least \emph{permissive} (i.e. informative) representative of the equivalence class of solutions. Choosing $g^{\star}_{\textnormal{max}}$ is natural because it is the sup/inf-envelope of all maximizers/minimizers, stable under maxitive combination, and
\begin{enumerate*}[label=\roman*)]
\item \emph{for the lower CBO}, does not introduce any extra, arbitrary information,
\item \emph{for the upper CBO}, does not forgo any of the available information in the likelihood and/or in the prior.
\end{enumerate*}

\textbf{Recovering the posterior:} Possibilistic Bayesian inference can thus be understood as the optimization of a possibilistic analogue of the ELBO: the CBO. Unlike the probabilistic case, this optimization generally admits multiple solutions, but the standard maxitive posterior $g^{\star}_{\textnormal{max}}$ stands out as the natural choice -- it is an extremum among the optimizers and serves as the canonical representative of the updated possibility distribution.

Recall that the standard Bayesian posterior is
    $\left\{ q^{\star}_{\textnormal{add}} \right\} = \argmax_{q\in\mathcal{P}(\Theta)} \mathrm{ELBO}(q)$,
while the two characterizations of the Bayesian posterior in possibility theory are:
\begin{subequations}
\label{eq:maxELBO}
\begin{align}
\label{eq:maxLowerELBO}
g^{\star}_{\textnormal{max}} & = \max\Big( \argmax_{g\in\mathcal{F}(\Theta)} \underline{\mathrm{CBO}}(g) \Big), \\
\label{eq:minUpperELBO}
& = \min\Big( \argmin_{g\in\mathcal{F}(\Theta)} \overline{\mathrm{CBO}}(g) \Big).
\end{align}
\end{subequations}
The additional $\max$/$\min$ operations in \eqref{eq:maxLowerELBO} and \eqref{eq:minUpperELBO} add a layer of complexity to the recovery of the posterior possibility function. However, it holds that $g^{\star}_{\textnormal{max}}$ is the only element in
\begin{align}
\argmax_{g\in\mathcal{F}(\Theta)} \underline{\mathrm{CBO}}(g) \cap \argmin_{g\in\mathcal{F}(\Theta)} \overline{\mathrm{CBO}}(g),
\end{align}
that is, the posterior possibility function is the unique element selected simultaneously by the lower and upper characterisations.


\textbf{Relationship with VI:} These formulations, \eqref{eq:maxLowerELBO} and \eqref{eq:minUpperELBO}, to recover the posterior possibility function from an optimization problem \emph{on the set of possibility functions}, provide two different schemes for defining approximations of the posterior when we restrict the optimization to be on a subset $\mathcal{G}$ of $\mathcal{F}(\Theta)$. While \eqref{eq:maxLowerELBO} is qualitatively related to standard VI and will typically underestimate the possibility when restricted to $\mathcal{G}$, \eqref{eq:minUpperELBO} provides an alternative that will instead overestimate possibilities. This latter behaviour is more in line with general statistical principles which tend to prefer pessimistic uncertainty estimates over optimistic ones. 

\textbf{Relationship with Generalized VI:} The lower and upper CBOs provide two variational approximations with opposite pointwise behaviour, the lower formulation tends to underestimate possibility degrees, whereas the upper formulation tends to overestimate them. Conceptually, this mirrors the way different probabilistic variational bounds induce different approximation biases. A more detailed comparison with generalized VI bounds is deferred to \autoref{app:GVI}.

\section{Possibilistic VI with Exponential Families}
\label{sec:ExpFam}


We develop a general framework for possibilistic variational inference based on exponential families by considering their conjugate priors. All the proofs for this section can be found in \autoref{proof:ExpFam}. We begin by defining the structure of exponential families in the possibilistic setting. Assuming that $\Theta$ is a subset of $\mathbb{R}^{d_{\theta}}$ and $\Lambda$ is a convex subset of $\mathbb{R}^{d_{\lambda}}$,   we follow the probabilistic approach and define an exponential family as a set $\mathcal{G} = \{g_{\lambda} \in \mathcal{F}(\Theta): \lambda \in \Lambda\}$ where $g_{\lambda}$ is of the form
\begin{align}
\label{eq:possExpFam}
g_{\lambda}(\theta) = \exp\big( \lambda^{\intercal} T(\theta) - A(\lambda) - B(\theta) \big) 
\end{align}
for a given base measure $B$ and sufficient statistics $T$ on $\Theta$ and the corresponding log-partition function $A$ on $\Lambda$, which ensures proper normalisation, that is
\begin{align}
A(\lambda) &= \log \sup_{\theta \in \Theta} \exp\big( \lambda^{\intercal} T(\theta) - B(\theta) \big) 
= \sup_{\theta \in \Theta} \lambda^{\intercal} T(\theta) - B(\theta).
\end{align}
More details about these general exponential families are given in \autoref{app:ExpFamInGeneral}.

\subsection{Bayesian Inference with Conjugate Family}
\label{sec:ExpFamBayes}

We consider a likelihood $p(\cdot \given  \theta)$, $\theta \in \Theta$, that is part of a regular and minimal exponential family with log-partition $A$, that is 
$p(x\given \theta) = \exp( \theta^{\intercal}T(x) - A(\theta) - B(x) )$. 
The following proposition rephrases known results on the relationship between exponential families and the Bregman divergence \citep{banerjee2005clustering} in the context of possibilistic inference.

\begin{proposition}
\label{prop:posteriorPossibilty}
The posterior possibility function under the uninformative prior on $\Theta$ is of the form
\begin{subequations}
\begin{align}
\label{eq:possConjPrior}
g_{\lambda}(\theta) & = \exp\big( \lambda^{\intercal}\theta - A^{\dagger}(\lambda) - A(\theta) \big) \\
\label{eq:PosteriorAsBregman}
& = \exp\big (-D_A(\theta \| \mode{\theta}(\lambda)) \big),
\end{align}
\end{subequations}
with $\mode{\theta}(\lambda) \doteq \possexp_{g_{\lambda}}[\bm{\theta}]$ the maximum likelihood estimator (MLE), $A^{\dagger}$ the Legendre transform of the log-partition $A$, and $D_A(\theta \| \theta') = A(\theta) - A(\theta') - \nabla_{\theta} A(\theta')^{\intercal}(\theta -\theta')$ the Bregman divergence.
\end{proposition}

Identifying the Bregman divergence in $g_{\lambda}$ allows us to leverage its properties. For instance, it is obvious from \eqref{eq:PosteriorAsBregman} that $\pi(\cdot\given  \lambda)$ is a possibility function with expected value $\mode{\theta}(\lambda)$. We can also use the known relationship between $D_A$ and the KL divergence to see that $-\log g_{\lambda}(\theta)$ is the KL divergence between the likelihood at $\theta$ and the one at the MLE. We now highlight some important properties of possibility functions of the same form as $g_{\lambda}$ in the following proposition.

\begin{proposition}
\label{prop:PropertiesOfConjugatePriors}
Let $\mathcal{G}_A(\Theta)$ be the subset of $\mathcal{F}(\Theta)$ defined as
\begin{align}
\mathcal{G}_A(\Theta) \doteq \big\{ \theta \mapsto \exp(\lambda^{\intercal}\theta - A^{\dagger}(\lambda) -A(\theta) ) : \lambda \in \Lambda \big\},
\end{align}
where $\Lambda \doteq \nabla_{\theta}A(\Theta)$ is the image of $\Theta$ by $\nabla_{\theta}A$, then, for any $g_{\lambda} \in \mathcal{G}_A(\Theta)$, it holds that
\begin{enumerate}[itemsep=.1em,parsep=.1em,topsep=.1em]
    \item The dual possibility function $f_{\theta}:\lambda \mapsto g_{\lambda}(\theta)$ is in $\mathcal{G}_{A^{\dagger}}(\Lambda)$.
    \item For any $\nu \geq 0$, $g_{\lambda}^{\nu} \in \mathcal{G}_{\nu A}(\Theta)$. Furthermore, it is a conjugate prior for the likelihood $p(\cdot \given  \theta)$, 
    and the posterior for $n$ i.i.d.\ observations $\{x_i\}_{i=1}^n$ is $g_{\lambda'}^{\nu'}$, where
    \begin{align}
        \lambda' = \frac{\nu\lambda + \sum_{i=1}^n T(x_i)}{\nu+n}, \quad \nu' = \nu + n.
    \end{align}
    That is, $\lambda$ is updated as the running mean of the sufficient statistics, and $\nu$ as the number of observations.
\end{enumerate}
\end{proposition}

The first point in \autoref{prop:PropertiesOfConjugatePriors} is a specific property of the considered class of conjugate prior possibility functions. Although it has additional properties, the possibilistic conjugate prior is often simply the renormalised version of its probabilistic counterpart. The second point is a possibilistic analogue of the result of \citet{diaconis1979conjugate}, with the additional features that the log-partition function is simply the Legendre transform of $A$ and that the $\nu$ parameter is a discount factor rather than an arbitrary parameter. This property is intuitive: $\nu$ can be interpreted as how many observations the prior is worth in terms of information content -- as in the probabilistic case; yet, the possibilistic treatment allows for considering $\nu = 0$, in which case $g_{\lambda}^{\nu} = \bm{1}$ and the prior is uninformative. 

We have the following properties of $g_{\lambda} \in \mathcal{G}_A(\Theta)$ 
as a possibility function describing the uncertain variable $\bm{\theta}$ in $\Theta$:
    (i) The expected value can be expressed as $\possexp_{g_{\lambda}}[\bm{\theta}] = \nabla_{\lambda} \leg{A}(\lambda)$ and is assumed to be the singleton $\{\mode{\theta}(\lambda)\}$, this is a weak assumption often verified by exponential families except when they become fully uninformative, i.e.\ $g_{\lambda} = \bm{1}$, and
    (ii) the precision can be expressed as
    \begin{align}
        \mathcal{I}_{g_{\lambda}}(\bm{\theta}) &\doteq \possexp_{g_{\lambda}}\sb{-\nabla^2_{\theta} \log g_{\lambda}\rb{\bm{\theta}}} 
        = \nabla_{\theta}^2 A\rb{\theta}|_{\theta = \mode{\theta}(\lambda)} = (\nabla^2_{\lambda} \leg{A}(\lambda))^{-1}.
    \end{align}

\subsection{Variational Inference with Conjugate Family}

We now demonstrate how to perform possibilistic VI with exponential family candidates of the form in \eqref{eq:possConjPrior}. For compactness, we write both consistency bounds using a sign variable
\(s \in \{+1,-1\}\). Define
\[
\mathrm{CBO}_s(g) = \mathop{\mathrm{ext}}_{\theta \in \Theta} \Bigl\{ -\ell(\theta)
- \log\frac{g(\theta)}{\pi(\theta)} \Bigr\}, \qquad \mathop{\mathrm{ext}} =
\begin{cases}
\sup, & s=+1,\\
\inf, & s=-1.
\end{cases}
\]
Thus $s=+1$ corresponds to the upper CBO, whereas $s=-1$ corresponds to the lower CBO. The associated objective is to minimize $s\,\mathrm{CBO}_s(g)$.
We 
derive a simple update rule for optimizing arbitrary minimization objectives, e.g.\ $\loss(\theta) = \ell(\theta) - \log\pi(\theta)$, for a loss function $\ell$ and a prior $\pi$.

In the following proposition, we reparameterize $g_{\lambda}$ in terms of its mode, using the form in \eqref{eq:PosteriorAsBregman}, and establish a Newton-like update for the corresponding parameter.

\begin{proposition}
\label{prop:updateUpperCBO}
Consider a candidate $g_\eta$ of the form $g_\eta(\theta)=\exp(-D_A(\theta\|\eta))$ and let $H_s(\eta) \doteq s(\nabla^2\hat\ell(\eta)-\nabla^2A(\eta))$. Assume $H_s(\eta)\succ 0$, so that the local quadratic approximation of the inner problem is concave for $s=+1$ and convex for $s=-1$. Danskin's theorem then gives
\begin{align}
    \nabla \mathrm{CBO}_s(\eta) = \nabla^2 A(\eta)\,(\eta - \theta_s),
\end{align}
where $\theta_s \doteq \mathop{\mathrm{argext}}_{\theta\in\Theta} \{-\hat\ell(\theta)-\log g_\eta(\theta)\}$. Replacing the inner objective by its 2\textsuperscript{nd}-order Taylor expansion around $\eta$ and minimizing the resulting surrogate via natural gradient descent yields the Newton-like update
\begin{align}
\label{eq:upperCBOupdate_general}
    \eta_{t+1} = \eta_t - \rho_t H_s(\eta_t)^{-1} \nabla\hat\ell(\eta_t),
\end{align}
with step size $\rho_t > 0$.
\end{proposition}

\begin{remark}
    The use of natural gradient for the update follows similar motivation from \citet{khan2023bayesian}. In fact, the update rule in \eqref{eq:upperCBOupdate_general} is equivalent to taking a mirror descent step defined in the space of the dual coordinate $\lambda=\nabla_{\eta}A(\eta)$. The natural gradient emerges naturally from the Legendre duality between the spaces of $\eta$ and $\lambda$ \citep{wainwright2008graphical}.
\end{remark}

\begin{remark}
While $H_s(\eta)\succ 0$ is necessary if $\nabla \loss\rb{\eta} \neq 0$, once the optimization converges and the equality is achieved, this constraint can be relaxed to $H_s(\eta)\succeq 0$. In particular, convergence would correspond to matching the first moment with that of the posterior as defined in \eqref{def:gibbs:max} and the second moment is matched when $\nabla^2 A(\eta^\thinstar) = \nabla^2\loss(\eta^\thinstar)$; this is in the spirit of the Laplace approximation used for fitting Gaussian candidates.
\end{remark}

\paragraph{Practical diagonal approximation.} We now specialise the signed update to a Gaussian candidate $g_\eta(\theta) = \overline{\mathrm{N}}(\theta;\eta,\vartheta^{-1}I)\), so that $\nabla^2 A(\eta)=\vartheta I$. We also use an isotropic Gaussian prior with precision $\delta$, equivalently an $L_2$ regularization term $\delta\|\eta\|^2/2$, which gives $\nabla\hat{\ell}(\eta) = \nabla\ell(\eta)+\delta\eta$ and $\nabla^2\hat{\ell}(\eta)=\nabla^2\ell(\eta)+\delta I$. For \(s\in\{+1,-1\}\), the signed curvature gap is therefore $H_s(\eta) = s(\nabla^2\ell(\eta)+(\delta-\vartheta)I )$. Substituting these expressions into the signed CBO update gives
\begin{align}
    \eta_{t+1} = \eta_t - \rho_t\,s\, \bigl(\nabla^2\ell(\eta_t) + (\delta-\vartheta) I\bigr)^{-1} \bigl(\nabla\ell(\eta_t)+\delta\eta_t\bigr).
\end{align}
For neural-network objectives, the Hessian is not formed explicitly. We therefore use the standard diagonal squared-gradient approximation $\nabla^2\ell(\eta_t) \approx \operatorname{diag}\bigl( \nabla\ell(\eta_t)\odot\nabla\ell(\eta_t) \bigr)$, and replace the stochastic gradient and squared-gradient estimates by exponential moving averages (EMAs): with 
$g_t = \operatorname{EMA}(\hat{\nabla}\ell(\eta_t))$ and $h_t = \operatorname{EMA}(
\hat{\nabla}\ell(\eta_t)\odot \hat{\nabla}\ell(\eta_t))$, where $\odot$ denotes the Hadamard product, we obtain an implementable signed update (see Alg.~\ref{alg:ucbopt} for the case $s=+1$) as
\begin{align}
    \label{eqn:ucbopt-ivon}
    \eta_{t+1} = \eta_t - \rho_t\,s\,\frac{\bar g_t+\delta\eta_t}{h_t+\delta-\vartheta},
\end{align}
with elementwise division, and where $\bar{g}_t$ the bias-corrected first moment based on $g_t$. In the case $s=+1$, the condition $\vartheta \leq \delta$ ensures that the diagonal preconditioner remains positive whenever $h_t$ is non-negative, and the limiting case $\vartheta \to 0^+$ corresponds to ignoring the curvature of the Gaussian candidate. This diagonal squared-gradient approximation follows the same heuristic as Vprop and Vadam \citep{pmlr-v80-khan18a}, while the use of exponential moving averages makes the resulting update reminiscent of Adam \citep{KingBa15} and IVON \citep{shen2024variational}. For $s=-1$, the analogous curvature condition is instead $h_t+\delta-\vartheta < 0$ elementwise, corresponding to the opposite local curvature regime associated with the lower CBO.

\section{Experiments}
\label{sec:experiments}

\begin{wrapfigure}{r}{0.55\textwidth}
\begin{minipage}{\linewidth}
\vspace{-2.43em}
\begin{algorithm}[H]
  \caption{\textbf{u}pper \textbf{CB}O \textbf{Opt}imisation (\texttt{uCBOpt})}
  \label{alg:ucbopt}
  \begin{algorithmic}[1]
    \STATE \textbf{Require:} Learning rates $\{\rho_t\}$, weight decay $\delta>0$, curvature parameter $\vartheta\in(0,\delta]$, momentum parameters $\beta_1,\beta_2\in[0,1)$, and Hessian init $h_0>0$.
    \STATE \textbf{Init:} $\eta \leftarrow$ (NN-weights), $h\leftarrow h_0$,
    $g\leftarrow 0$.
    \FOR{$t = 1,2,\ldots$}
      \STATE $\hat{g} \leftarrow \widehat{\nabla}\,\ell(\eta)$ \hfill\textcolor{Gray}{\# Stochastic gradient}
      \STATE $\hat{h} \leftarrow \hat{g} \odot \hat{g}$
      \STATE $g \leftarrow \beta_1 g + (1-\beta_1)\hat{g}$
      \STATE $h \leftarrow \beta_2 h + (1-\beta_2)\hat{h}$
      \STATE $\bar{g} \leftarrow g/(1-\beta_1^t)$
      \STATE $\eta \leftarrow \eta - \rho_t(\bar{g} + \delta \eta)/(h+\delta-\vartheta)$
    \ENDFOR
    \RETURN $\eta$
  \end{algorithmic}
\end{algorithm}
\end{minipage}
\end{wrapfigure}


In this section, we evaluate the efficacy of the CBO update rule derived in \autoref{prop:updateUpperCBO}. We consider three variants: (1) \textbf{u}pper \textbf{CB}O \textbf{Opt}imization, denoted by \texttt{uCBOpt} (Alg.~\ref{alg:ucbopt}); (2) \texttt{uCBOpt} with adaptive curvature, denoted by \texttt{uCBOpt-adapt} (Alg.~\ref{alg:ucbopt-adapt}); and (3) \textbf{l}ower \textbf{CB}O \textbf{Opt}imization with adaptive curvature, denoted by \texttt{lCBOpt-adapt} (Alg.~\ref{alg:lcbopt-adapt}). In \texttt{uCBOpt}, we use a fixed curvature parameter $\vartheta$. In \texttt{uCBOpt-adapt} and \texttt{lCBOpt-adapt}, we instead make the curvature parameter adaptive by setting $\vartheta_t=\gamma(h_t+\delta)$, where $h_t$ is the running curvature estimate and $\delta$ is the prior curvature (weight-decay). We provide additional details on the algorithms in \autoref{sec:algo}. 


We evaluate our methods by training neural networks on image classification tasks, closely following a subset of the experiments in \citet{shen2024variational}. We perform extensive hyperparameter tuning only on Fashion-MNIST, and use the hyperparameters from \citet{shen2024variational} for the larger datasets. Additional training details are provided in \autoref{sec:exp_details}. For all experiments, we run each method with three random seeds and report the mean and standard deviation. We present the main results in \autoref{tab:main-results}, while the full results, including additional metrics, final-epoch performance, and training and calibration plots, are presented in \autoref{sec:full_results}. Ablation studies are also included in \autoref{sec:ablation}.

\textbf{Baselines.} We evaluate our methods against six baseline approaches: SGD, AdamW \citep{loshchilov2018decoupled}, IVON \citep{shen2024variational}, MC Dropout (MC-D) \citep{gal16mcd}, linearized last-layer Laplace (referred to as Laplace) \citep{daxberger2021laplace}, and SWAG \citep{maddox_2019_simple}. AdamW and SGD serve as deterministic optimization baselines, whereas IVON, MC-D, Laplace, and SWAG serve as uncertainty-aware baselines; these are discussed in \autoref{app:related}. For all uncertainty-aware baselines, we report predictive averages over 64 stochastic forward passes. For IVON, we additionally report predictions obtained from the posterior mean parameters, denoted by IVON@mean. For CIFAR-100, we consider only four baselines, as our primary objective is to demonstrate that our optimizers can scale effectively to larger models and datasets.

\textbf{Experiments.} As in-domain tasks, we train LeNet on Fashion-MNIST, ResNet-20 on CIFAR-10, and DenseNet-121 on CIFAR-100. For the out-of-domain (OOD) tasks, we use EMNIST as the OOD dataset for LeNet trained on Fashion-MNIST,
SVHN for ResNet-20 trained on CIFAR-10, 
and TinyImageNet for DenseNet-121 trained on CIFAR-100.



\textbf{Results.} On Fashion-MNIST, the \texttt{CBOpt} variants achieve the strongest in-domain performance overall. In particular, \texttt{lCBOpt-adapt} obtains the best accuracy (0.913) and NLL (0.252), while \texttt{uCBOpt-adapt} performs comparably. For OOD detection on EMNIST, \texttt{uCBOpt} also matches MC Dropout for the lowest FPR@95. On CIFAR-10, the \texttt{CBOpt} variants remain competitive with the strongest baselines, reaching accuracies around 0.910, comparable to SGD and Laplace, and achieving OOD detection performance close to SWAG and MC Dropout, although they do not surpass them. It is worth noting that MC Dropout uses additional dropout layers, which may provide extra regularization, and that all the uncertainty-aware baselines require multiple stochastic forward passes or sampled models at inference time. In contrast, our method does not require such sampling at test time. Nevertheless, as shown in \autoref{sec:mcdrop-cbopt}, combining \texttt{CBOpt} with MC Dropout preserves the in-domain performance of MC Dropout while improving OOD performance. The benefits of \texttt{CBOpt} become more pronounced on CIFAR-100. In this larger-scale setting, all three \texttt{CBOpt} variants outperform the baselines on both in-domain and out-of-domain tasks, with the exception of ECE, where IVON and MC dropout achieve better calibration through multiple stochastic forward passes at test time. These results suggest that \texttt{CBOpt} scales effectively to larger models and datasets. Overall, the \texttt{CBOpt} family consistently matches or improves upon IVON, the most directly comparable variational baseline, across both in-domain and OOD metrics, suggesting that the maxitive formulation can provide competitive uncertainty estimates without sacrificing predictive accuracy.


\begin{table*}[t!]
\caption{Experimental results on in-domain and out-of-domain tasks over three runs. The best result and all results within one standard deviation of the best-performing method are bolded.}
\label{tab:main-results}
\centering
\setlength{\tabcolsep}{3pt}
\renewcommand{\arraystretch}{1.12}
\resizebox{\linewidth}{!}{
\begin{tabular}{p{3cm}lcccccc}
\toprule
& & \multicolumn{3}{c}{In-domain} & \multicolumn{3}{c}{Out-of-domain} \\
\cmidrule(lr){3-5} \cmidrule(lr){6-8}
Dataset & Optimizer
& Acc. $\uparrow$
& NLL $\downarrow$
& ECE $\downarrow$
& FPR@95 $\downarrow$
& AUROC $\uparrow$
& AUPR-Out $\uparrow$ \\
\midrule

\multirow[c]{10}{=}{
\begin{tabular}[t]{@{}l@{}}
\textbf{In:} Fashion-MNIST\\
\textbf{Out:} EMNIST
\end{tabular}
}
& AdamW
& $0.900_{\pm 0.005}$ & $0.294_{\pm 0.007}$ & $0.026_{\pm 0.004}$
& $0.321_{\pm 0.008}$ & $0.745_{\pm 0.009}$ & $0.627_{\pm 0.022}$ \\

& SGD
& $0.896_{\pm 0.005}$ & $0.293_{\pm 0.005}$ & $0.022_{\pm 0.004}$
& $0.297_{\pm 0.020}$ & $0.767_{\pm 0.022}$ & $0.845_{\pm 0.012}$ \\

& IVON@mean
& $0.907_{\pm 0.001}$ & $0.267_{\pm 0.005}$ & $0.018_{\pm 0.008}$
& $0.266_{\pm 0.024}$ & $0.787_{\pm 0.029}$ & $0.861_{\pm 0.019}$ \\

& IVON
& $\bm{0.911}_{\pm 0.003}$ & $\bm{0.258}_{\pm 0.001}$ & $\bm{0.013}_{\pm 0.004}$
& $\bm{0.249}_{\pm 0.024}$ & $0.804_{\pm 0.030}$ & $0.873_{\pm 0.020}$ \\

& MC-D
& $0.897_{\pm 0.003}$ & $0.281_{\pm 0.007}$ & $\bm{0.013}_{\pm 0.003}$
& $\bm{0.244}_{\pm 0.011}$ & $\bm{0.821}_{\pm 0.008}$ & $\bm{0.885}_{\pm 0.007}$ \\

& Laplace
& $0.895_{\pm 0.004}$ & $0.295_{\pm 0.009}$ & $0.024_{\pm 0.003}$
& $0.296_{\pm 0.047}$ & $0.770_{\pm 0.039}$ & $0.848_{\pm 0.031}$ \\

& SWAG
& $0.900_{\pm 0.003}$ & $0.342_{\pm 0.009}$ & $0.043_{\pm 0.003}$
& $0.378_{\pm 0.005}$ & $0.731_{\pm 0.005}$ & $0.820_{\pm 0.002}$ \\

& \graycell{\texttt{uCBOpt}}
& \graycell{}$0.907_{\pm 0.002}$
& \graycell{}$0.270_{\pm 0.002}$
& \graycell{}$0.023_{\pm 0.000}$
& \graycell{}$\bm{0.245}_{\pm 0.010}$
& \graycell{}$0.810_{\pm 0.012}$
& \graycell{}$\bm{0.879}_{\pm 0.010}$ \\

\rowcolor{gray!10} \cellcolor{white}
& \texttt{uCBOpt-adapt}
& $\bm{0.910}_{\pm 0.001}$ & $\bm{0.256}_{\pm 0.003}$ & $\bm{0.016}_{\pm 0.001}$
& $0.272_{\pm 0.011}$ & $0.773_{\pm 0.015}$ & $0.857_{\pm 0.009}$ \\

\rowcolor{gray!10} \cellcolor{white}
& \texttt{lCBOpt-adapt}
& $\bm{0.913}_{\pm 0.002}$ & $\bm{0.252}_{\pm 0.004}$ & $0.020_{\pm 0.000}$
& $0.269_{\pm 0.015}$ & $0.784_{\pm 0.019}$ & $0.859_{\pm 0.011}$ \\

\cmidrule{1-8}

\multirow[c]{10}{=}{
\begin{tabular}[c]{@{}l@{}}
\textbf{In:} CIFAR-10\\
\textbf{Out:} SVHN
\end{tabular}
}
& AdamW
& $0.872_{\pm 0.001}$ & $0.406_{\pm 0.002}$ & $0.042_{\pm 0.003}$
& $0.271_{\pm 0.033}$ & $0.820_{\pm 0.031}$ & $0.881_{\pm 0.021}$ \\

& SGD
& $0.909_{\pm 0.003}$ & $0.308_{\pm 0.011}$ & $0.039_{\pm 0.008}$
& $\bm{0.207}_{\pm 0.019}$ & $0.864_{\pm 0.013}$ & $\bm{0.918}_{\pm 0.011}$ \\

& IVON@mean
& $0.892_{\pm 0.003}$ & $0.358_{\pm 0.011}$ & $0.045_{\pm 0.008}$
& $0.255_{\pm 0.013}$ & $0.830_{\pm 0.011}$ & $0.890_{\pm 0.010}$ \\

& IVON
& $0.899_{\pm 0.004}$ & $0.317_{\pm 0.004}$ & $0.028_{\pm 0.005}$
& $0.252_{\pm 0.020}$ & $0.836_{\pm 0.014}$ & $0.894_{\pm 0.012}$ \\

& MC-D
& $\bm{0.923}_{\pm 0.002}$ & $\bm{0.230}_{\pm 0.006}$ & $\bm{0.010}_{\pm 0.002}$
& $0.230_{\pm 0.030}$ & $0.847_{\pm 0.024}$ & $0.902_{\pm 0.017}$ \\

& Laplace
& $0.909_{\pm 0.001}$ & $0.303_{\pm 0.008}$ & $0.035_{\pm 0.006}$ 
& $0.230_{\pm 0.021}$ & $0.852_{\pm 0.019}$ & $0.909_{\pm 0.013}$ \\

& SWAG
& $0.918_{\pm 0.001}$ & $0.251_{\pm 0.004}$ & $\bm{0.009}_{\pm 0.000}$ 
& $\bm{0.197}_{\pm 0.012}$ & $\bm{0.879}_{\pm 0.007}$ & $\bm{0.924}_{\pm 0.007}$ \\

& \graycell{}\texttt{uCBOpt}
& \graycell{}$0.910_{\pm 0.004}$
& \graycell{}$0.296_{\pm 0.009}$
& \graycell{}$0.038_{\pm 0.001}$
& \graycell{}$\bm{0.208}_{\pm 0.016}$
& \graycell{}$0.869_{\pm 0.012}$
& \graycell{}$\bm{0.918}_{\pm 0.010}$ \\

\rowcolor{gray!10} \cellcolor{white}
& \texttt{uCBOpt-adapt}
& $0.909_{\pm 0.005}$ & $0.300_{\pm 0.016}$ & $0.038_{\pm 0.004}$
& $0.210_{\pm 0.016}$ & $0.866_{\pm 0.010}$ & $0.915_{\pm 0.007}$ \\

\rowcolor{gray!10} \cellcolor{white}
& \texttt{lCBOpt-adapt}
& $0.909_{\pm 0.005}$ & $0.299_{\pm 0.002}$ & $0.037_{\pm 0.002}$
& $0.213_{\pm 0.012}$ & $0.861_{\pm 0.006}$ & $0.913_{\pm 0.003}$ \\

\cmidrule{1-8}

\multirow[c]{7}{=}{
\begin{tabular}[t]{@{}l@{}}
\textbf{In:} CIFAR-100\\
\textbf{Out:} TinyImageNet
\end{tabular}
}

& SGD
& $0.649_{\pm 0.004}$ & $1.334_{\pm 0.006}$ & $0.096_{\pm 0.006}$
& $0.357_{\pm 0.005}$ & $0.720_{\pm 0.008}$ & $0.684_{\pm 0.006}$ \\

& IVON@mean
& $0.599_{\pm 0.001}$ & $1.521_{\pm 0.008}$ & $0.092_{\pm 0.006}$
& $0.368_{\pm 0.006}$ & $0.702_{\pm 0.008}$ & $0.663_{\pm 0.007}$ \\

& IVON
& $0.607_{\pm 0.002}$ & $1.458_{\pm 0.012}$ & $\bm{0.071}_{\pm 0.004}$
& $0.368_{\pm 0.006}$ & $0.702_{\pm 0.008}$ & $0.663_{\pm 0.007}$ \\

& MC-D
& $0.649_{\pm 0.016}$ & $1.298_{\pm 0.016}$ & $\bm{0.074}_{\pm 0.020}$
& $0.353_{\pm 0.008}$ & $0.718_{\pm 0.008}$ & $0.681_{\pm 0.009}$ \\

& \graycell{\texttt{uCBOpt}}
& \graycell{}$\bm{0.671}_{\pm 0.014}$
& \graycell{}$\bm{1.231}_{\pm 0.005}$
& \graycell{}$0.095_{\pm 0.014}$
& \graycell{}$\bm{0.344}_{\pm 0.004}$
& \graycell{}$\bm{0.727}_{\pm 0.003}$
& \graycell{}$0.693_{\pm 0.003}$ \\

\rowcolor{gray!10} \cellcolor{white}
& \texttt{uCBOpt-adapt}
& $0.655_{\pm 0.011}$ & $1.298_{\pm 0.011}$ & $0.096_{\pm 0.021}$
& $\bm{0.343}_{\pm 0.004}$ & $\bm{0.729}_{\pm 0.002}$ & $\bm{0.696}_{\pm 0.001}$ \\

\rowcolor{gray!10} \cellcolor{white}
& \texttt{lCBOpt-adapt}
& $\bm{0.674}_{\pm 0.007}$ & $1.268_{\pm 0.018}$ & $0.110_{\pm 0.002}$
& $\bm{0.346}_{\pm 0.004}$ & $\bm{0.729}_{\pm 0.005}$ & $0.692_{\pm 0.006}$ \\

\bottomrule
\end{tabular}
}
\vskip -0.1in
\end{table*}

\section{Discussion}
\label{sec:discussion}

Possibilistic Bayesian inference is already an optimization problem, which seems to limit the interest in the development of possibilistic VI at first sight. On the contrary, we have shown that there is a surprising amount of subtleties to be understood and of insights to be gained when exploring the properties of possibilistic VI in general as well as specifically through the lens of a possibilistic version of exponential families and conjugate priors. And indeed, even when closely following existing treatments of VI and its applications, many differences emerge, such as regularities in conjugate families and a new route leading to novel optimization methods.

In particular, the \texttt{CBOpt} family of optimizers developed from possibilistic VI demonstrates that this framework is not merely a reformulation, but can lead to practically useful algorithms. From a practical viewpoint, the scalar parameter $\vartheta$ can be interpreted as modifying the effective curvature or regularization strength appearing in the denominator of the update. In the adaptive variants, this contribution changes throughout training. Our theory therefore separates this effect from the standard weight-decay term, providing additional flexibility in the update. Empirically, this flexibility leads to competitive performance across in-domain and out-of-domain tasks, suggesting that possibilistic VI offers not only a conceptually distinct perspective on uncertainty representation, but also a practical foundation for developing scalable uncertainty-aware learning algorithms.

\paragraph{Limitations.} The Newton-like update of \autoref{prop:updateUpperCBO} is based on a second-order Taylor expansion of the inner extremum problem and on the curvature condition $H_s \succ 0$, which couples the admissible range of the candidate precision $\vartheta$ to the local geometry of $\hat{\ell}$. For the lower CBO this condition is non-trivial to maintain pointwise and underlies the need for adaptive variants. Our exponential-family treatment further assumes that $A$ is proper, closed, and convex with $A^{\dagger\dagger} = A$, and that $\mathbb{E}^\star_{g_\lambda}[\theta]$ is a singleton, excluding the fully uninformative limit. Empirically, the framework is instantiated only for isotropic Gaussian candidates with a diagonal empirical-Fisher surrogate, leaving the anisotropic and non-Gaussian members of the construction in \autoref{sec:ExpFam} untested.



\bibliographystyle{apalike}
\bibliography{bibliography}

@article{
    blei2017vi-review,
    author = {David M. Blei and Alp Kucukelbir and Jon D. McAuliffe},
    title = {Variational Inference: A Review for Statisticians},
    journal = {Journal of the American Statistical Association},
    volume = {112},
    number = {518},
    pages = {859--877},
    year = {2017},
    publisher = {ASA Website},
    doi = {10.1080/01621459.2017.1285773},
}

@InProceedings{
    salimans2015mcmc-vi-bridge,
    title = 	 {Markov Chain {M}onte {C}arlo and Variational Inference: Bridging the Gap},
    author = 	 {Salimans, Tim and Kingma, Diederik and Welling, Max},
    booktitle = 	 {Proceedings of the 32nd International Conference on Machine Learning},
    pages = 	 {1218--1226},
    year = 	 {2015},
    editor = 	 {Bach, Francis and Blei, David},
    volume = 	 {37},
    series = 	 {Proceedings of Machine Learning Research},
    address = 	 {Lille, France},
    month = 	 {07--09 Jul},
    publisher =    {PMLR}
}

@InProceedings{
    kingma2022autoencodingvariationalbayes,
    title={Auto-Encoding Variational {B}ayes}, 
    author={Diederik P Kingma and Max Welling},
    booktitle={International Conference on Learning Representations},
    year={2013}
}

@InProceedings{
    blundell2015weight-uncertainty,
    title = 	 {Weight Uncertainty in Neural Network},
    author = 	 {Blundell, Charles and Cornebise, Julien and Kavukcuoglu, Koray and Wierstra, Daan},
    booktitle = 	 {Proceedings of the 32nd International Conference on Machine Learning},
    pages = 	 {1613--1622},
    year = 	 {2015},
    editor = 	 {Bach, Francis and Blei, David},
    volume = 	 {37},
    series = 	 {Proceedings of Machine Learning Research},
    address = 	 {Lille, France},
    month = 	 {07--09 Jul},
    publisher =    {PMLR}
}

@inproceedings{
    kingma2015reparameterization-trick,
    author = {Kingma, Durk P and Salimans, Tim and Welling, Max},
    booktitle = {Advances in Neural Information Processing Systems},
    editor = {C. Cortes and N. Lawrence and D. Lee and M. Sugiyama and R. Garnett},
    pages = {},
    publisher = {Curran Associates, Inc.},
    title = {Variational Dropout and the Local Reparameterization Trick},
    volume = {28},
    year = {2015}
}

@article{
    zadeh1978fuzzy,
    title = {Fuzzy sets as a basis for a theory of possibility},
    journal = {Fuzzy Sets and Systems},
    volume = {1},
    number = {1},
    pages = {3-28},
    year = {1978},
    issn = {0165-0114},
    doi = {https://doi.org/10.1016/0165-0114(78)90029-5},
    author = {Zadeh, Lotfi Aliasger},
}

@article{
    cella2025vi-im,
    title = {Computationally efficient variational-like approximations of possibilistic inferential models},
    journal = {International Journal of Approximate Reasoning},
    volume = {186},
    pages = {109506},
    year = {2025},
    issn = {0888-613X},
    doi = {https://doi.org/10.1016/j.ijar.2025.109506},
    author = {Leonardo Cella and Ryan Martin},
}

@article{
    martin2013im,
    author = {Ryan Martin and Chuanhai Liu},
    title = {Inferential Models: A Framework for Prior-Free Posterior Probabilistic Inference},
    journal = {Journal of the American Statistical Association},
    volume = {108},
    number = {501},
    pages = {301--313},
    year = {2013},
    publisher = {ASA Website},
    doi = {10.1080/01621459.2012.747960},
    URL = {https://doi.org/10.1080/01621459.2012.747960},
    eprint = {https://doi.org/10.1080/01621459.2012.747960}
}

@book{shannon1949mathematical,
  title     = {The Mathematical Theory of Communication},
  author    = {Shannon, Claude E. and Weaver, Warren},
  year      = {1949},
  publisher = {University of Illinois Press},
  address   = {Urbana, IL}
}

@article{
    kullback1951kld,
    author = {Kullback, Solomon and Leibler, Richard},
    title = {{On Information and Sufficiency}},
    volume = {22},
    journal = {The Annals of Mathematical Statistics},
    number = {1},
    publisher = {Institute of Mathematical Statistics},
    pages = {79 -- 86},
    year = {1951},
    doi = {10.1214/aoms/1177729694},
}

@article{
    donsker1983vi,
    author = {Donsker, Monroe David and Varadhan, S. R. Srinivasa},
    title = {Asymptotic evaluation of certain {M}arkov process expectations for large time. III},
    journal = {Communications on Pure and Applied Mathematics},
    volume = {28},
    pages = {389–461},
    year = {1976}
}

@article{catoni2004,
  title={Statistical Learning Theory and Stochastic Optimization. {S}aint-{F}lour Summer School on Probability Theory 2001 ({J}ean {P}icard ed.)},
  author={Catoni, Olivier},
  journal={Lecture Notes in Mathematics. Springer},
  volume={2},
  pages={10},
  year={2004}
}

@book{kullback1959information,
  title={Information Theory and Statistics},
  author={Kullback, Solomon},
  year={1959},
  publisher={John Wiley \& Sons}
}

@article{li2016renyi,
  title={R{\'e}nyi divergence variational inference},
  author={Li, Yingzhen and Turner, Richard E},
  journal={Advances in neural information processing systems},
  volume={29},
  year={2016}
}

@article{dieng2017variational,
  title={Variational Inference via $\chi$ Upper Bound Minimization},
  author={Dieng, Adji Bousso and Tran, Dustin and Ranganath, Rajesh and Paisley, John and Blei, David},
  journal={Advances in Neural Information Processing Systems},
  volume={30},
  year={2017}
}

@article{zellner1988optimal,
  title={Optimal information processing and {B}ayes's theorem},
  author={Zellner, Arnold},
  journal={The American Statistician},
  volume={42},
  number={4},
  pages={278--280},
  year={1988},
  publisher={Taylor \& Francis}
}

@article{houssineau2022robust,
  title={Robust {B}ayesian inference in complex models with possibility theory},
  author={Houssineau, Jeremie and Nott, David J},
  journal={arXiv preprint arXiv:2204.06911},
  year={2022}
}

@article{banerjee2005clustering,
  title={Clustering with {B}regman divergences},
  author={Banerjee, Arindam and Merugu, Srujana and Dhillon, Inderjit S and Ghosh, Joydeep},
  journal={Journal of machine learning research},
  volume={6},
  number={Oct},
  pages={1705--1749},
  year={2005}
}

@article{diaconis1979conjugate,
  title={Conjugate priors for exponential families},
  author={Diaconis, Persi and Ylvisaker, Donald},
  journal={The Annals of statistics},
  pages={269--281},
  year={1979},
  publisher={JSTOR}
}

@article{tracking2020ristic,
title = {Target tracking in the framework of possibility theory: The possibilistic {B}ernoulli filter},
journal = {Information Fusion},
volume = {62},
pages = {81-88},
year = {2020},
issn = {1566-2535},
doi = {https://doi.org/10.1016/j.inffus.2020.04.008},
url = {https://www.sciencedirect.com/science/article/pii/S1566253520302633},
author = {Branko Ristic and Jeremie Houssineau and Sanjeev Arulampalam},
}

@article{khan2023bayesian,
  title={The {B}ayesian learning rule},
  author={Khan, Mohammad Emtiyaz and Rue, H{\aa}vard},
  journal={Journal of Machine Learning Research},
  volume={24},
  number={281},
  pages={1--46},
  year={2023}
}

@InProceedings{
    KingBa15,
    author    = {Kingma, Diederik and Ba, Jimmy},
    booktitle = {International Conference on Learning Representations (ICLR)},
    title     = {Adam: A Method for Stochastic Optimization},
    year      = {2015},
    address   = {San Diega, CA, USA},
    optmonth  = {12},
}

@inproceedings{hieu2025decoupling,
  title={Decoupling epistemic and aleatoric uncertainties with possibility theory},
  author={Hieu, Nong Minh and Houssineau, Jeremie and Chada, Neil K and Delande, Emmanuel},
  booktitle={The 28th International Conference on Artificial Intelligence and Statistics},
  pages={2899--2907},
  year={2025},
  organization={ML Research Press}
}

@ARTICLE{xue2025orbit,
  author={Xue, Zhirun and Cai, Han and Houssineau, Jeremie and Zhang, Jingrui},
  journal={IEEE Transactions on Aerospace and Electronic Systems}, 
  title={Orbit-Attitude Coupled Control for Multitarget Tracking Based on Partition Pattern Search}, 
  year={2025},
  volume={61},
  number={4},
  pages={10855-10867},
  doi={10.1109/TAES.2025.3549026}}

@article{houssineau2021linear,
  title={A linear algorithm for multi-target tracking in the context of possibility theory},
  author={Houssineau, Jeremie},
  journal={IEEE Transactions on Signal Processing},
  volume={69},
  pages={2740--2751},
  year={2021},
  publisher={IEEE}
}

@inproceedings{shen2024variational,
  title={Variational Learning is Effective for Large Deep Networks},
  author={Shen, Yuesong and Daheim, Nico and Cong, Bai and Nickl, Peter and Marconi, Gian Maria and Raoul, Bazan Clement Emile Marcel and Yokota, Rio and Gurevych, Iryna and Cremers, Daniel and Khan, Mohammad Emtiyaz and others},
  booktitle={International Conference on Machine Learning},
  pages={44665--44686},
  year={2024},
  organization={PMLR}
}

@InProceedings{
    pmlr-v80-khan18a,
    title = 	 {Fast and Scalable {B}ayesian Deep Learning by Weight-Perturbation in {A}dam},
    author =       {Khan, Mohammad and Nielsen, Didrik and Tangkaratt, Voot and Lin, Wu and Gal, Yarin and Srivastava, Akash},
    booktitle = 	 {Proceedings of the 35th International Conference on Machine Learning},
    pages = 	 {2611--2620},
    year = 	 {2018},
    editor = 	 {Dy, Jennifer and Krause, Andreas},
    volume = 	 {80},
    series = 	 {Proceedings of Machine Learning Research},
    month = 	 {10--15 Jul},
    publisher =    {PMLR},
    url = 	 {https://proceedings.mlr.press/v80/khan18a.html},
}

@inproceedings{
    loshchilov2018decoupled,
    title={Decoupled Weight Decay Regularization},
    author={Ilya Loshchilov and Frank Hutter},
    booktitle={International Conference on Learning Representations},
    year={2019},
    url={https://openreview.net/forum?id=Bkg6RiCqY7},
}

@InProceedings{
    gal16mcd,
    title = 	 {Dropout as a {B}ayesian Approximation: Representing Model Uncertainty in Deep Learning},
    author = 	 {Gal, Yarin and Ghahramani, Zoubin},
    booktitle = 	 {Proceedings of The 33rd International Conference on Machine Learning},
    pages = 	 {1050--1059},
    year = 	 {2016},
    editor = 	 {Balcan, Maria Florina and Weinberger, Kilian Q.},
    volume = 	 {48},
    series = 	 {Proceedings of Machine Learning Research},
    address = 	 {New York, New York, USA},
    month = 	 {20--22 Jun},
    publisher =    {PMLR},
    url = 	 {https://proceedings.mlr.press/v48/gal16.html},
}

@article{wainwright2008graphical,
  title={Graphical models, exponential families, and variational inference},
  author={Wainwright, Martin J and Jordan, Michael I and others},
  journal={Foundations and Trends{\textregistered} in Machine Learning},
  volume={1},
  number={1--2},
  pages={1--305},
  year={2008},
  publisher={Now Publishers, Inc.}
}

@inproceedings{daxberger2021laplace,
 author = {Daxberger, Erik and Kristiadi, Agustinus and Immer, Alexander and Eschenhagen, Runa and Bauer, Matthias and Hennig, Philipp},
 booktitle = {Advances in Neural Information Processing Systems},
 editor = {M. Ranzato and A. Beygelzimer and Y. Dauphin and P.S. Liang and J. Wortman Vaughan},
 pages = {20089--20103},
 publisher = {Curran Associates, Inc.},
 title = {Laplace Redux - Effortless {B}ayesian Deep Learning},
 url = {https://proceedings.neurips.cc/paper_files/paper/2021/file/a7c9585703d275249f30a088cebba0ad-Paper.pdf},
 volume = {34},
 year = {2021}
}

@article{
    martin2026imreview,
    author = {Ryan Martin},
    title = {Possibilistic Inferential Models: A Review},
    journal = {Journal of the American Statistical Association},
    volume = {121},
    number = {553},
    pages = {807--826},
    year = {2026},
    publisher = {Taylor \& Francis},
    doi = {10.1080/01621459.2025.2606127},
    URL = {https://doi.org/10.1080/01621459.2025.2606127},
}

@article{martin2022valid,
  title={Valid and efficient imprecise-probabilistic inference with partial priors, {I}. {F}irst results},
  author={Martin, Ryan},
  journal={arXiv preprint arXiv:2203.06703},
  year={2022}
}

@inproceedings{zhang2025open,
  title={Open-world objectness modeling unifies novel object detection},
  author={Zhang, Shan and Ni, Yao and Du, Jinhao and Xue, Yuan and Torr, Philip and Koniusz, Piotr and Van Den Hengel, Anton},
  booktitle={Proceedings of the Computer Vision and Pattern Recognition Conference},
  pages={30332--30342},
  year={2025}
}

@article{krueger2017bayesian,
  title={Bayesian hypernetworks},
  author={Krueger, David and Huang, Chin-Wei and Islam, Riashat and Turner, Ryan and Lacoste, Alexandre and Courville, Aaron},
  journal={arXiv preprint arXiv:1710.04759},
  year={2017}
}

@article{rudner2022tractable,
  title={Tractable function-space variational inference in {B}ayesian neural networks},
  author={Rudner, Tim GJ and Chen, Zonghao and Teh, Yee Whye and Gal, Yarin},
  journal={Advances in Neural Information Processing Systems},
  volume={35},
  pages={22686--22698},
  year={2022}
}

@article{mobiny2021dropconnect,
  title={Dropconnect is effective in modeling uncertainty of {B}ayesian deep networks},
  author={Mobiny, Aryan and Yuan, Pengyu and Moulik, Supratik K and Garg, Naveen and Wu, Carol C and Van Nguyen, Hien},
  journal={Scientific reports},
  volume={11},
  number={1},
  pages={5458},
  year={2021},
  publisher={Nature Publishing Group UK London}
}

@article{lakshminarayanan2017simple,
  title={Simple and scalable predictive uncertainty estimation using deep ensembles},
  author={Lakshminarayanan, Balaji and Pritzel, Alexander and Blundell, Charles},
  journal={Advances in neural information processing systems},
  volume={30},
  year={2017}
}

@inproceedings{
lohr2025credal,
title={Credal Prediction based on Relative Likelihood},
author={Timo L{\"o}hr and Paul Hofman and Felix Mohr and Eyke H{\"u}llermeier},
booktitle={The Thirty-ninth Annual Conference on Neural Information Processing Systems},
year={2025},
url={https://openreview.net/forum?id=rKM3oqruN3}
}

@article{malinin2018predictive,
  title={Predictive uncertainty estimation via prior networks},
  author={Malinin, Andrey and Gales, Mark},
  journal={Advances in neural information processing systems},
  volume={31},
  year={2018}
}

@article{charpentier2020posterior,
  title={Posterior network: Uncertainty estimation without ood samples via density-based pseudo-counts},
  author={Charpentier, Bertrand and Z{\"u}gner, Daniel and G{\"u}nnemann, Stephan},
  journal={Advances in neural information processing systems},
  volume={33},
  pages={1356--1367},
  year={2020}
}

@inproceedings{
charpentier2022natural,
title={Natural Posterior Network: Deep {B}ayesian Predictive Uncertainty for Exponential Family Distributions},
author={Bertrand Charpentier and Oliver Borchert and Daniel Z{\"u}gner and Simon Geisler and Stephan G{\"u}nnemann},
booktitle={International Conference on Learning Representations},
year={2022},
url={https://openreview.net/forum?id=tV3N0DWMxCg}
}

@article{sensoy2018evidential,
  title={Evidential deep learning to quantify classification uncertainty},
  author={Sensoy, Murat and Kaplan, Lance and Kandemir, Melih},
  journal={Advances in neural information processing systems},
  volume={31},
  year={2018}
}

@book{josang2016subjective,
  title={Subjective logic},
  author={J{\o}sang, Audun},
  volume={3},
  year={2016},
  publisher={Springer}
}

@article{dempster1968generalization,
  title={A generalization of {B}ayesian inference},
  author={Dempster, Arthur P},
  journal={Journal of the Royal Statistical Society: Series B (Methodological)},
  volume={30},
  number={2},
  pages={205--232},
  year={1968},
  publisher={Wiley Online Library}
}

@book{shafer1976mathematical,
 ISBN = {9780691100425},
 URL = {http://www.jstor.org/stable/j.ctv10vm1qb},
 author = {Glenn Shafer},
 publisher = {Princeton University Press},
 title = {A Mathematical Theory of Evidence},
 urldate = {2026-01-26},
 year = {1976}
}

@inproceedings{iedl,
  title={Uncertainty estimation by fisher information-based evidential deep learning},
  author={Deng, Danruo and Chen, Guangyong and Yu, Yang and Liu, Furui and Heng, Pheng-Ann},
  booktitle={International conference on machine learning},
  pages={7596--7616},
  year={2023},
  organization={PMLR}
}

@inproceedings{
fedl,
title={Uncertainty Estimation by Flexible Evidential Deep Learning},
author={Taeseong Yoon and Heeyoung Kim},
booktitle={The Thirty-ninth Annual Conference on Neural Information Processing Systems},
year={2025},
url={https://openreview.net/forum?id=N6ujq5Yfwa}
}

@inproceedings{
bengs2022pitfalls,
title={Pitfalls of Epistemic Uncertainty Quantification through Loss Minimisation},
author={Viktor Bengs and Eyke H{\"u}llermeier and Willem Waegeman},
booktitle={Advances in Neural Information Processing Systems},
editor={Alice H. Oh and Alekh Agarwal and Danielle Belgrave and Kyunghyun Cho},
year={2022},
url={https://openreview.net/forum?id=epjxT_ARZW5}
}

@article{verdoja2020notesmcd,
  author       = {Francesco Verdoja and
                  Ville Kyrki},
  title        = {Notes on the Behavior of {MC} Dropout},
  journal      = {CoRR},
  volume       = {abs/2008.02627},
  year         = {2020},
  url          = {https://arxiv.org/abs/2008.02627},
  eprinttype   = {arXiv},
  eprint       = {2008.02627},
  bibsource    = {dblp computer science bibliography, https://dblp.org}
}

@inproceedings{ni2026possibilistic,
  title={Possibilistic Predictive Uncertainty for Deep Learning},
  author={Ni, Yao and Houssineau, Jeremie and Ong, Yew Soon and Koniusz, Piotr},
  booktitle={Proceedings of the 43rd International Conference on Machine Learning},
  series={Proceedings of Machine Learning Research},
  year={2026},
  note={To appear},
  publisher={PMLR}
}

@article{caprio2024credall,
  title={Credal learning theory},
  author={Caprio, Michele and Sultana, Maryam and Elia, Eleni G and Cuzzolin, Fabio},
  journal={Advances in Neural Information Processing Systems},
  volume={37},
  pages={38665--38694},
  year={2024}
}

@article{wang2024credal,
  title={Credal deep ensembles for uncertainty quantification},
  author={Wang, Kaizheng and Cuzzolin, Fabio and Manchingal, Shireen K and Shariatmadar, Keivan and Moens, David and Hallez, Hans},
  journal={Advances in Neural Information Processing Systems},
  volume={37},
  pages={79540--79572},
  year={2024}
}

@article{
caprio2024credalb,
title={Credal {B}ayesian Deep Learning},
author={Michele Caprio and Souradeep Dutta and Kuk Jin Jang and Vivian Lin and Radoslav Ivanov and Oleg Sokolsky and Insup Lee},
journal={Transactions on Machine Learning Research},
issn={2835-8856},
year={2024},
url={https://openreview.net/forum?id=4NHF9AC5ui},
note={}
}

@inproceedings{
chau2025integral,
title={Integral Imprecise Probability Metrics},
author={Siu Lun Chau and Michele Caprio and Krikamol Muandet},
booktitle={The Thirty-ninth Annual Conference on Neural Information Processing Systems},
year={2025},
url={https://openreview.net/forum?id=KM2XzHq2Rm}
}

@inproceedings{maddox_2019_simple,
  title={A simple baseline for {B}ayesian uncertainty in deep learning},
  author={Maddox, Wesley J and Izmailov, Pavel and Garipov, Timur and Vetrov, Dmitry P and Wilson, Andrew Gordon},
  booktitle={Advances in Neural Information Processing Systems},
  pages={13153--13164},
  year={2019}
}

\newpage
\onecolumn
\appendix

\section*{Supplementary Material}
\addcontentsline{toc}{section}{Supplementary Material}

\startcontents[appendix]
\printcontents[appendix]{}{1}{\section*{Table of Contents}}

\newpage

\section{Related Work}
\label{app:related}

In this section, we position our work with respect to current literature. 

\textbf{Possibilistic Variational Inference.} Our work is closest to \citet{cella2025vi-im}, which proposes a ``variational-like'' Monte-Carlo--driven strategy for approximating the $\alpha$-cut of an Inferential Model's \citep{martin2013im} possibility contour with a $100(1-\alpha)\%$ credible set of a candidate family; see \citet{martin2026imreview} for a review on IMs. Their framework ensures that the resulting approximations inherit the finite-sample validity properties of IMs \citep{martin2022valid}. In contrast, we take a more standard approach, framing approximate inference as minimizing divergence from the Bayes posterior. Specifically, we derive a maxitive analogue of the Donsker-Varadhan formulation, which enables direct optimization over a chosen candidate family, analogous to probabilistic VI. Importantly, the optimization problem is well-suited for modern autograd engines designed for large-scale compositional models, like neural networks.

\textbf{Epistemic Uncertainty in Deep Learning.} Neural networks are notoriously susceptible to overconfidence, presenting a significant hurdle for their deployment in safety-critical applications. Standard models often assign high predictive certainty to incorrect outputs, particularly when encountering out-of-distribution (OOD) data \citep{zhang2025open}. Accordingly, \emph{Epistemic Uncertainty} (EU) quantification, which seeks to represent the model's lack of knowledge rather than the inherent noise in the data, has garnered attention from the research community. While Bayesian neural networks offer a theoretically grounded approach to this problem \citep{blundell2015weight-uncertainty, krueger2017bayesian, rudner2022tractable, gal16mcd, mobiny2021dropconnect, lakshminarayanan2017simple}, their high computational complexity has spurred the development of more efficient second-order predictive models. These methods, such as Prior Networks \citep{malinin2018predictive}, PostNet \citep{charpentier2020posterior}, and Natural Posterior Networks \citep{charpentier2022natural}, bypass parameter-space inference by directly modelling distributions over predictions.

A prominent subset of these second-order approaches is Evidential Deep Learning (EDL) \citep{sensoy2018evidential}. EDL draws upon Subjective Logic \citep{josang2016subjective} and the Dempster-Shafer Theory \citep{dempster1968generalization, shafer1976mathematical} to treat predictions as ``opinions'' supported by evidence. While various iterations have sought to refine this framework \citep{iedl,fedl}, many EDL variants rely on heuristic objectives that lack a rigorous probabilistic derivation \citep{bengs2022pitfalls}. Consequently, establishing a principled connection between these evidential objectives and formal epistemic uncertainty remains an active area of research.

This gap has motivated recent efforts to ground second-order prediction in possibility theory rather than in subjective logic or belief functions. Concurrent work by \citet{ni2026possibilistic} takes this direction by leveraging the maxitive pseudo-divergence and the maxitive Bayes posterior under the uninformative possibilistic prior, both introduced in this paper, to derive a Dirichlet-based loss for image classification. Their formulation specializes the present framework to the probability simplex, projecting parameter-space possibility onto prediction space via the supremum-based change of variable, and yields a closed-form inner maximizer that is competitive with state-of-the-art EDL methods. Their work and ours are complementary: where they take selected components and apply them to a specific second-order prediction problem, we develop the underlying variational principle, including the maxitive Donsker-Varadhan duality, the dual CBO formulations, and the exponential-family treatment, and study its consequences for general possibilistic inference and learning.

\textbf{Imprecise probability.} There exists a formal correspondence between possibility measures and credal sets: a given possibility measure characterizes a unique, closed, and convex set of probability measures that it upper bounds. This establishes possibilistic modelling as a mathematically rigorous instance of imprecise probability theory, a framework that has recently attracted growing interest in machine learning for principled epistemic uncertainty quantification, including credal generalization bounds \citep{caprio2024credall}, credal-set neural networks for out-of-distribution detection \citep{wang2024credal}, imprecise Bayesian deep learning \citep{caprio2024credalb}, likelihood-based credal prediction \citep{lohr2025credal}, and metrics over imprecise probabilistic objects \citep{chau2025integral}. Accordingly, our maxitive-VI formulation (\autoref{thm:maxitive_DV}) opens up a complementary line of research within this broader programme, providing a variational principle through which the comprehensive set of tools from imprecise probability theory can be applied to epistemic uncertainty quantification in deep learning.

\textbf{Computational Efficiency.} Bayesian deep learning provides a principled framework for uncertainty quantification, yet its practical application is often bottlenecked by the high computational cost of marginalizing over the parameters. Exact inference remains intractable for modern neural networks, necessitating the use of approximations that balance statistical rigour with hardware constraints. Here, we discuss a few popular strategies that we benchmark against:

\begin{itemize}
    \item Monte-Carlo Dropout (MC-D) \citep{gal16mcd} interprets the use of dropout at test time as a form of VI, simulating a model ensemble. It is highly efficient during training as it uses multiple forward passes of standard dropout, but has been know to provide unreliable uncertainty estimates \citep{verdoja2020notesmcd}.
    \item Last-layer Laplace approximation (Laplace) \citep{daxberger2021laplace} fits a Gaussian approximation to the posterior over the final-layer weights after maximum \textit{a posteriori} (MAP) training. It is computationally efficient and often improves uncertainty estimation, but its posterior uncertainty is restricted to the last layer.
    \item Stochastic Weight Averaging Gaussian (SWAG) \citep{maddox_2019_simple} fits a low-rank Gaussian distribution to the trajectory of iterates generated by SGD. This is computationally efficient because it captures weight-space statistics during a standard training run.
    \item Improved Variational Online Newton (IVON) \citep{shen2024variational} is a second-order variational optimizer that approximates a Gaussian posterior. It is computationally efficient, matching the complexity of first-order optimizers, while providing well-calibrated uncertainty estimates.
\end{itemize}

In comparison to these approaches, our \texttt{CBOpt} models the network weights using a possibility function. Similar to the diagonal Gaussian approximations used in SWAG-diagonal and IVON, we adopt a mean-field approximation for performing VI, which is computationally efficient and enables scaling our methodology to large-scale learning problems.

\section{Notations}
\label{app:notation}

In \autoref{tab:notations}, we present some of the notations used in this work.

\begingroup
\renewcommand{\arraystretch}{1.3}   
\setlength{\tabcolsep}{7pt}
\begin{table}[ht]
\caption{Summary of key notations}
\centering
\begin{tabularx}{0.8\linewidth}{@{} >{$}l<{$} L @{}}
\toprule
\text{Notation} & \text{Description} \\
\midrule
\possexp_{f} & Possibilistic expectation (mode) of $f$ \\
\mathcal{I}_f & Possibilistic precision of $f$ \\
A^{\dagger} & Legendre transform of $A$ \\
\loss & Regularized loss: $\ell(\theta) - \log \pi(\theta)$ \\
\overline{\mathrm{CBO}} & Upper consistency bound \\
\underline{\mathrm{CBO}} & Lower consistency bound \\
\mode{\theta}(\lambda) &  Mode / expected value of $g_\lambda$, that is $\possexp_{g_{\lambda}}[\bm{\theta}]$\\
\bottomrule
\label{tab:notations}
\end{tabularx}
\end{table}
\endgroup

\section{More on Possibility Theory}
\label{app:possibilityTheory}

To define epistemic uncertainty formally, we consider a sample space $\Omega$ whose elements characterise all the possible values of the relevant unknown quantities. Instead of equipping $\Omega$ with a probabilistic structure, we simply describe an unknown parameter $\theta_0$ in set $\Theta$ via a (deterministic) \emph{uncertain variable} $\bm{\theta} : \Omega \to \Theta$. If an element $\omega \in \Omega$ were the correct one, then $\bm{\theta}(\omega)$ would be true value of the parameter $\theta_0$. To \emph{describe} the available information about $\bm{\theta}$, we define a \emph{possibility function} $f_{\bm{\theta}}$ (a.k.a.\ possibility distribution) on $\Theta$ as a non-negative function that verifies $\sup_{\theta \in \Theta} f_{\bm{\theta}}(\theta)=1$. The possibility of an event $\bm{\theta} \in A$ for some $A \subseteq \Theta$ is then $\sup_{\theta \in A} f_{\bm{\theta}}(\theta)$.

\paragraph{Values and support.} Since $f_{\bm{\theta}}$ is non-negative and $\sup_{\theta\in\Theta} f_{\bm{\theta}}(\theta) = 1$, the values of any possibility function lie in $[0,1]$, with the upper bound attained somewhere on $\Theta$. The support of $f$ is the set $\operatorname{supp}(f_{\bm{\theta}}) \doteq \{\theta \in \Theta : f_{\bm{\theta}}(\theta) > 0\}$ of values that are not deemed impossible. The convention $f_{\bm{\theta}}(\theta) > 0$ everywhere is sometimes convenient and is assumed implicitly when, for instance, the ratio of two possibility functions is considered, as in the definition of the max-relative entropy in~\eqref{eq:max-rel-ent}.

\paragraph{Necessity.} Each possibility function $f_\theta$ comes equipped with a dual necessity measure, defined for any $A \subseteq \Theta$ by
\begin{equation}
    N_{\bm{\theta}}(A) \doteq 1 - \sup_{\theta \in A^c} f_{\bm{\theta}}(\theta).
    \label{eq:necessity}
\end{equation}
Whereas $\sup_{\theta \in A} f_{\bm{\theta}}(\theta)$ quantifies how plausible the event $\theta \in A$ is, $N_{\bm{\theta}}(A)$ quantifies how strongly the available information rules out its complement. An event with low necessity is not necessarily implausible, only insufficiently supported by the available information, a key distinction for the conservative treatment of epistemic uncertainty in possibility theory.

\paragraph{Independence.} Two uncertain variables $\bm{\theta}$ and $\bm{\psi}$ on $\Theta$ and $\Psi$ are said to be independent (or independently described) when their joint possibility function factorises as
\begin{equation}
    f_{\bm{\theta},\bm{\psi}}(\theta, \psi) = f_{\bm{\theta}}(\theta) f_{\bm{\psi}}(\psi), \qquad \forall (\theta, \psi) \in \Theta \times \Psi.
    \label{eq:independence}
\end{equation}
This convention mirrors the probabilistic notion of independence and is the one adopted throughout the paper. A weaker, more conservative alternative occasionally found in the literature replaces the product by a minimum but is not used here.

\paragraph{Marginalization and conditioning.} Possibility functions behave similar to probability mass functions, except that summation is replaced by a maximum. For instance, if $\bm{\theta}$ and another uncertain variable $\bm{\psi}$ on a set $\Psi$ are jointly described by a possibility function $f_{\bm{\theta},\bm{\psi}}$, then the \emph{marginal} possibility function describing $\bm{\theta}$ is characterised by
\begin{equation}
f_{\bm{\theta}}(\theta) =\sup_{\psi \in \Psi} f_{\bm{\theta},\bm{\psi}}(\theta,\psi), \qquad \forall \theta \in \Theta.
\end{equation}
Similarly, for a fixed $\psi \in \Psi$ satisfying $f_{\bm{\psi}}(\psi) > 0$, the \emph{conditional} possibility function of $\bm{\theta}$ given $\bm{\psi} = \psi$ is characterised by
\begin{equation}
f_{\bm{\theta}}(\theta \given \bm{\psi} = \psi) = \frac{f_{\bm{\theta},\bm{\psi}}(\theta,\psi)}{f_{\bm{\psi}}(\psi)}, \qquad \forall \theta \in \Theta. \label{eq:conditioning}
\end{equation}
As is standard with probability distributions, we will often omit which uncertain variable is being described by a possibility function and simply write, e.g. $f$ instead of $f_{\bm{\theta}}$. 

\paragraph{Bayesian inference.} Given a likelihood $L(\theta) \propto p(x \mid \theta)$, regarded as a possibility function over $\theta$, and a prior possibility function $\pi$, applying the conditioning rule~\eqref{eq:conditioning} to the joint possibility function $f_{\theta, X}(\theta, x) \propto p(x \mid \theta)\, \pi(\theta)$ directly yields
\begin{equation}
    g^\star_{\max}(\theta) = \frac{p(x \mid \theta)\, \pi(\theta)}{\sup_{\theta'} p(x \mid \theta')\, \pi(\theta')},
    \label{eq:posterior-derivation}
\end{equation}
which is the posterior used as the central object of~\eqref{def:gibbs:max}. The denominator plays the role of a marginal likelihood and is referred to in the main text as the maxitive marginal likelihood $Z_{\max}$. Unlike its probabilistic counterpart, it measures the consistency between the prior and the likelihood rather than an averaged model fit.

\paragraph{Uninformative priors.} The function $\pi \equiv \mathbf{1}$ is itself a valid possibility function and represents the complete absence of prior information, in the sense that every value of $\theta$ is deemed fully plausible. More generally, any prior of the form $\pi(\theta) = \exp(-R(\theta))$ with $R \geq 0$ is automatically a possibility function once renormalised by $\sup_\theta \exp(-R(\theta))$, regardless of whether $\int \exp(-R(\theta))\, \mathrm{d}\theta$ is finite. This contrasts with the probabilistic setting, where standard regularizers such as a constant or the indicator of an unbounded set fail to integrate and yield improper priors.

\paragraph{Connection to imprecise probability.} A possibility function $f$ can also be interpreted as an upper envelope over a family of probability distributions, namely the credal set $\{P : P(A) \leq \sup_{\theta \in A} f(\theta) \text{ for all measurable } A\}$. This perspective is not developed here, and we refer the interested reader to \citet{hieu2025decoupling} for further discussion.


\section{Proofs}
\label{app:proofs}

In this section, we present the proofs for the main theoretical results in the main text.
    
\subsection{Proof of \autoref{thm:maxitive_DV}}
\label{proof:thm:maxitive_DV}

We prove each equation of the theorem one after the other, with the notation $h = -\ell$ for conciseness.

\begin{proof}[Proof of \autoref{eq:dv:max:supinf}]
    Consider some fixed possibility function $g$. 
    
    Let us first assume that there exists some $\mode{\theta}_g\in\argsup_{\theta\in\Theta} g(\theta)$. We then have $g(\mode{\theta}_g)=1$, so that:
    \begin{subequations}
    \begin{align}
    \log \sup_{\theta \in \Theta} e^{h(\theta)} \pi(\theta) & \geq \log e^{h(\mode{\theta}_g)} \pi(\mode{\theta}_g) \\ 
    & = \log \frac{e^{h(\mode{\theta}_g)} \pi(\mode{\theta}_g)}{g(\mode{\theta}_g)} \\ 
    & \geq \inf_{\theta \in \Theta} \log \frac{e^{h(\theta)} \pi(\theta)}{g(\theta)} \\ 
    &= \inf_{\theta \in \Theta} \Big\{ h(\theta) - \log \frac{g(\theta)}{\pi(\theta)} \Big\} \, .
    \end{align}
    \end{subequations}
    Now, if $\argsup_{\theta\in\Theta} g(\theta)=\emptyset$, then one can still define a sequence $(\mode{\theta}_{g,n})_{n=1}^{\infty}\in\Theta^\mathbb{N}$ such that for any integer $n>0$, $g(\mode{\theta}_{g,n})\geq 1-1/n$. We can then write:
    \begin{subequations}
    \begin{align}
        \log \sup_{\theta \in \Theta} e^{h(\theta)} \pi(\theta) & \geq \log e^{h(\mode{\theta}_{g,n})} \pi(\mode{\theta}_{g,n}) \\ 
        & = \log \frac{e^{h(\mode{\theta}_{g,n})} \pi(\mode{\theta}_{g,n})}{1-\frac{1}{n}} + \log \left(1-\frac{1}{n}\right) \\
        & \geq \log \frac{e^{h(\mode{\theta}_{g,n})} \pi(\mode{\theta}_{g,n})}{g(\mode{\theta}_{g,n})} + \log \left(1-\frac{1}{n}\right) \\ 
        & \geq \inf_{\theta \in \Theta} \log \frac{e^{h(\theta)} \pi(\theta)}{g(\theta)} + \log \left(1-\frac{1}{n}\right) \\ 
        & = \inf_{\theta \in \Theta} \Big\{ h(\theta) - \log \frac{g(\theta)}{\pi(\theta)} \Big\} + \log \left(1-\frac{1}{n}\right) \, ,
    \end{align}
    \end{subequations}
    so by letting $n\to+\infty$, we have
    \begin{align}
    \log \sup_{\theta \in \Theta} e^{h(\theta)} \pi(\theta) \geq \inf_{\theta \in \Theta} \Big\{ h(\theta) - \log \frac{g(\theta)}{\pi(\theta)} \Big\} \, .
    \end{align}
    Consequently, the inequality above holds for any possibility function $g\in\mathcal{F}(\Theta)$, and by taking the supremum over $g\in\mathcal{F}(\Theta)$ in the right-hand side leads to:
    \begin{align}
    \log \sup_{\theta \in \Theta} e^{h(\theta)} \pi(\theta) \geq \sup_{g\in\mathcal{F}(\Theta)} \inf_{\theta \in \Theta} \Big\{ h(\theta) - \log \frac{g(\theta)}{\pi(\theta)} \Big\} \, .
    \end{align}
    Furthermore, the choice of possibility function $g(\theta)=g^{\star}_{\textnormal{max}}(\theta) = {e^{h(\theta)} \pi(\theta)}/{\sup_{\theta'} e^{h(\theta')} \pi(\theta')}$ transforms the inequality into an equality, which finally gives \autoref{eq:dv:max:supinf}:
    \begin{align}
        \log \sup_{\theta \in \Theta} e^{h(\theta)} \pi(\theta) = \sup_{g\in\mathcal{F}(\Theta)} \inf_{\theta \in \Theta} \Big\{ h(\theta) - \log \frac{g(\theta)}{\pi(\theta)} \Big\} \, .
    \end{align}
\end{proof}
  
\begin{proof}[Proof of \autoref{eq:dv:max:infsup}] 
    Any possibility function $g$ takes its values in $[0,1]$, so that for any possibility function $g$ and any parameter $\theta\in\Theta$, we have $\log g(\theta)\leq 0$. Thus, for any possibility function $g$ and any parameter $\theta$,
    \begin{align}
    \log e^{h(\theta)} \pi(\theta) \leq h(\theta) - \log \frac{g(\theta)}{\pi(\theta)} \, ,
    \end{align}
    so taking the supremum over $\theta$ in both sides leads for any $g$ to:
    \begin{align}
    \log \sup_{\theta \in \Theta} e^{h(\theta)} \pi(\theta) \leq \sup_{\theta \in \Theta} \Big\{ h(\theta) - \log \frac{g(\theta)}{\pi(\theta)} \Big\} \, ,
    \end{align}
    and taking the infimum over possibility functions $g$ in the right-hand side leads to
    \begin{align}
    \log \sup_{\theta \in \Theta} e^{h(\theta)} \pi(\theta) \leq \inf_{g\in\mathcal{F}(\Theta)} \sup_{\theta \in \Theta} \Big\{ h(\theta) - \log \frac{g(\theta)}{\pi(\theta)} \Big\} \, ,
    \end{align}
    Once again, $g=g^{\star}_{\textnormal{max}}$ transforms the inequality into an equality, which finally gives \autoref{eq:dv:max:infsup}
    \begin{align}
    \log \sup_{\theta \in \Theta} e^{h(\theta)} \pi(\theta) = \inf_{g\in\mathcal{F}(\Theta)} \sup_{\theta \in \Theta} \Big\{  h(\theta) - \log \frac{g(\theta)}{\pi(\theta)} \Big\} \, .
    \end{align}
\end{proof}

\begin{proof}[Proof of \autoref{eq:gibbs:max:supinf}]
    We have already mentioned in the proof of \autoref{eq:dv:max:supinf} that
    \begin{align}
    g^{\star}_{\textnormal{max}}(\theta) \doteq \dfrac{e^{h(\theta)} \pi(\theta)}{\sup\limits_{\theta'\in\Theta} e^{h(\theta')} \pi(\theta')} \in \argmax_{g\in\mathcal{F}(\Theta)} \inf_{\theta \in \Theta} \left\{ h(\theta) - \log\left(\frac{g(\theta)}{\pi(\theta)}\right) \right\} \, , 
    \end{align}
    so that
    \begin{align}
    \inf_{\theta \in \Theta} \left\{ h(\theta) - \log\left(\frac{g^{\star}_{\textnormal{max}}(\theta)}{\pi(\theta)}\right) \right\} = \sup_{g\in\mathcal{F}(\Theta)}\inf_{\theta \in \Theta} \left\{ h(\theta) - \log\left(\frac{g(\theta)}{\pi(\theta)}\right) \right\} \, .
    \end{align}
    Furthermore, any possibility function $g'\in\mathcal{F}(\Theta)$ such that $g'\preceq g^{\star}_{\textnormal{max}}$ satisfies by monotonicity of the logarithmic function:
    \begin{align}
    \inf_{\theta \in \Theta} \left\{ h(\theta) - \log\left(\frac{g'(\theta)}{\pi(\theta)}\right) \right\} \geq \inf_{\theta \in \Theta} \left\{ h(\theta) - \log\left(\frac{g^{\star}_{\textnormal{max}}(\theta)}{\pi(\theta)}\right) \right\} \, .
    \end{align}
    Hence, combining the two lines above, we get for any possibility function $g'\in\mathcal{F}(\Theta)$ such that $g'\preceq g^{\star}_{\textnormal{max}}$:
    \begin{align}
    \inf_{\theta \in \Theta} \left\{ h(\theta) - \log\left(\frac{g'(\theta)}{\pi(\theta)}\right) \right\} \geq \sup_{g\in\mathcal{F}(\Theta)}\inf_{\theta \in \Theta} \left\{ h(\theta) - \log\left(\frac{g(\theta)}{\pi(\theta)}\right) \right\} \, .
    \end{align}
    Since by definition of the supremum, we have the reverse inequality:
    \begin{align}
    \inf_{\theta \in \Theta} \left\{ h(\theta) - \log\left(\frac{g'(\theta)}{\pi(\theta)}\right) \right\} \leq \sup_{g\in\mathcal{F}(\Theta)}\inf_{\theta \in \Theta} \left\{ h(\theta) - \log\left(\frac{g(\theta)}{\pi(\theta)}\right) \right\} \, ,
    \end{align}
    we finally have for any possibility function $g'\in\mathcal{F}(\Theta)$ such that $g'\preceq g^{\star}_{\textnormal{max}}$:
    \begin{align}
    \inf_{\theta \in \Theta} \left\{ h(\theta) - \log\left(\frac{g'(\theta)}{\pi(\theta)}\right) \right\} = \sup_{g\in\mathcal{F}(\Theta)}\inf_{\theta \in \Theta} \left\{ h(\theta) - \log\left(\frac{g(\theta)}{\pi(\theta)}\right) \right\} \, .
    \end{align}
    This provides half of the proof, namely
    \begin{align}
    \left\{ g\in\mathcal{F}(\Theta) \, : \, g\preceq g^{\star}_{\textnormal{max}} \right\} \subset \argmax_{g\in\mathcal{F}(\Theta)} \inf_{\theta \in \Theta} \left\{ h(\theta) - \log\left(\frac{g(\theta)}{\pi(\theta)}\right) \right\} \, .
    \end{align}
    To get the equality, we still have to show that any possibility function $g$ not satisfying $g\preceq g^{\star}_{\textnormal{max}}$ does not belong to the argmax.
    To show this, let us consider some possibility function $g'\notin \left\{ g\in\mathcal{F}(\Theta) \, : \, g\preceq g^{\star}_{\textnormal{max}} \right\}$, and show that
    \begin{align}
    \inf_{\theta \in \Theta} \left\{ h(\theta) - \log\left(\frac{g'(\theta)}{\pi(\theta)}\right) \right\} < \sup_{g\in\mathcal{F}(\Theta)}\inf_{\theta \in \Theta} \left\{ h(\theta) - \log\left(\frac{g(\theta)}{\pi(\theta)}\right) \right\} \, .
    \end{align}
    To see this, recall that since $g' \not\preceq g^{\star}_{\textnormal{max}}$, there exists at least one $\theta' \in \Theta$ such that $g'(\theta') > g^{\star}_{\textnormal{max}}(\theta')$. Therefore, 
    \begin{align}
    h(\theta') - \log\!\left(\frac{g'(\theta')}{\pi(\theta')}\right) < h(\theta') - \log\!\left(\frac{g^{\star}_{\textnormal{max}}(\theta')}{\pi(\theta')}\right), 
    \end{align}
    Using the definition of the infimum, we have 
    \begin{align}
    \inf_{\theta\in\Theta}\left\{h(\theta) - \log\!\left(\frac{g'(\theta)}{\pi(\theta)}\right)\right\} \le h(\theta') - \log\!\left(\frac{g'(\theta')}{\pi(\theta')}\right) < h(\theta') - \log\!\left(\frac{g^{\star}_{\textnormal{max}}(\theta')}{\pi(\theta')}\right) . 
    \end{align}
    However, notice that by definition of $g^{\star}_{\textnormal{max}}$, the quantity in the right-hand side above does not depend on $\theta'$ since: 
    \begin{align}
    h(\theta') - \log\!\left(\frac{g^{\star}_{\textnormal{max}}(\theta')}{\pi(\theta')}\right) = h(\theta') - \log\!\left(\frac{\dfrac{e^{h(\theta')} \pi(\theta')}{\sup e^{h(\cdot)} \pi(\cdot)}}{\pi(\theta')}\right) = \log \sup_{\theta\in\Theta} e^{h(\theta)} \pi(\theta) ,
    \end{align}
    so combining the two lines above, we get
    \begin{align}
    \inf_{\theta\in\Theta}\left\{h(\theta) - \log\!\left(\frac{g'(\theta)}{\pi(\theta)}\right)\right\} < \log \sup_{\theta\in\Theta} e^{h(\theta)} \pi(\theta) ,
    \end{align}
    By using \autoref{eq:dv:max:supinf}, we can rewrite the quantity in the right-hand side:
    \begin{align}
    \inf_{\theta\in\Theta}\left\{h(\theta) - \log\!\left(\frac{g'(\theta)}{\pi(\theta)}\right)\right\} < \sup_{g\in\mathcal{F}(\Theta)} \inf_{\theta\in\Theta}\left\{h(\theta) - \log\!\left(\frac{g(\theta)}{\pi(\theta)}\right)\right\}, 
    \end{align}
    which is exactly what we wanted to show. Hence, $g'$ cannot belong to the set of maximisers. This proves that 
    \begin{align}
    \argmax_{g\in\mathcal{F}(\Theta)} \inf_{\theta\in\Theta} \left\{h(\theta) - \log\!\left(\frac{g(\theta)}{\pi(\theta)}\right)\right\} = \left\{ g\in\mathcal{F}(\Theta) : g\preceq g^{\star}_{\textnormal{max}} \right\}, 
    \end{align}
    which concludes the proof.
    \end{proof}

\begin{proof}[Proof of \autoref{eq:gibbs:max:infsup}]
    The proof of \autoref{eq:gibbs:max:infsup} is very similar to the proof of \autoref{eq:gibbs:max:supinf}, and is only provided for the sake of completeness.
    Once again, we start from the following fact mentioned in the proof \eqref{eq:dv:max:infsup}:
    \begin{align}
    g^{\star}_{\textnormal{max}}(\theta) \doteq \dfrac{e^{h(\theta)} \pi(\theta)}{\sup\limits_{\theta'\in\Theta} e^{h(\theta')} \pi(\theta')} \in \argmin_{g\in\mathcal{F}(\Theta)} \sup_{\theta \in \Theta} \left\{ h(\theta) - \log\left(\frac{g(\theta)}{\pi(\theta)}\right) \right\} \, , 
    \end{align}
    so that
    \begin{align}
    \sup_{\theta \in \Theta} \left\{ h(\theta) - \log\left(\frac{g^{\star}_{\textnormal{max}}(\theta)}{\pi(\theta)}\right) \right\} = \inf_{g\in\mathcal{F}(\Theta)}\sup_{\theta \in \Theta} \left\{ h(\theta) - \log\left(\frac{g(\theta)}{\pi(\theta)}\right) \right\} \, .
    \end{align}
    Furthermore, any possibility function $g'\in\mathcal{F}(\Theta)$ such that $g^{\star}_{\textnormal{max}} \preceq g'$ satisfies by monotonicity of the logarithmic function:
    \begin{align}
    \sup_{\theta \in \Theta} \left\{ h(\theta) - \log\left(\frac{g'(\theta)}{\pi(\theta)}\right) \right\} \leq \sup_{\theta \in \Theta} \left\{ h(\theta) - \log\left(\frac{g^{\star}_{\textnormal{max}}(\theta)}{\pi(\theta)}\right) \right\} \, .
    \end{align}
    Hence, combining the two lines above, we get for any possibility function $g'\in\mathcal{F}(\Theta)$ such that $g^{\star}_{\textnormal{max}} \preceq g'$:
    \begin{align}
    \sup_{\theta \in \Theta} \left\{ h(\theta) - \log\left(\frac{g'(\theta)}{\pi(\theta)}\right) \right\} \leq \inf_{g\in\mathcal{F}(\Theta)}\sup_{\theta \in \Theta} \left\{ h(\theta) - \log\left(\frac{g(\theta)}{\pi(\theta)}\right) \right\} \, .
    \end{align}
    Since by definition of the infimum, we have the reverse inequality:
    \begin{align}
    \sup_{\theta \in \Theta} \left\{ h(\theta) - \log\left(\frac{g'(\theta)}{\pi(\theta)}\right) \right\} \geq \inf_{g\in\mathcal{F}(\Theta)}\sup_{\theta \in \Theta} \left\{ h(\theta) - \log\left(\frac{g(\theta)}{\pi(\theta)}\right) \right\} \, ,
    \end{align}
    we finally have for any possibility function $g'\in\mathcal{F}(\Theta)$ such that $g^{\star}_{\textnormal{max}}\preceq g'$:
    \begin{align}
    \sup_{\theta \in \Theta} \left\{ h(\theta) - \log\left(\frac{g'(\theta)}{\pi(\theta)}\right) \right\} = \inf_{g\in\mathcal{F}(\Theta)}\sup_{\theta \in \Theta} \left\{ h(\theta) - \log\left(\frac{g(\theta)}{\pi(\theta)}\right) \right\} \, .
    \end{align}
    This provides half of the proof, namely
    \begin{align}
    \left\{ g\in\mathcal{F}(\Theta) \, : \, g^{\star}_{\textnormal{max}} \preceq g \right\} \subset \argmin_{g\in\mathcal{F}(\Theta)} \sup_{\theta \in \Theta} \left\{ h(\theta) - \log\left(\frac{g(\theta)}{\pi(\theta)}\right) \right\} \, .
    \end{align}
    To get the equality, we now have to show that any possibility function $g$ not satisfying $g^{\star}_{\textnormal{max}}\preceq g$ does not belong to the argmin.
    To show this, let us consider some possibility function $g'\notin \left\{ g\in\mathcal{F}(\Theta) \, : \, g^{\star}_{\textnormal{max}} \preceq g \right\}$, and show that
    \begin{align}
    \sup_{\theta \in \Theta} \left\{ h(\theta) - \log\left(\frac{g'(\theta)}{\pi(\theta)}\right) \right\} > \inf_{g\in\mathcal{F}(\Theta)}\sup_{\theta \in \Theta} \left\{ h(\theta) - \log\left(\frac{g(\theta)}{\pi(\theta)}\right) \right\} \, .
    \end{align}
    To see this, recall that since $g^{\star}_{\textnormal{max}} \not\preceq g'$, there exists at least one $\theta' \in \Theta$ such that $g'(\theta') < g^{\star}_{\textnormal{max}}(\theta')$. Therefore, 
    \begin{align}
    h(\theta') - \log\!\left(\frac{g'(\theta')}{\pi(\theta')}\right) > h(\theta') - \log\!\left(\frac{g^{\star}_{\textnormal{max}}(\theta')}{\pi(\theta')}\right), 
    \end{align}
    Using the definition of the supremum, we now have 
    \begin{align}
    \sup_{\theta\in\Theta}\left\{h(\theta) - \log\!\left(\frac{g'(\theta)}{\pi(\theta)}\right)\right\} \ge h(\theta') - \log\!\left(\frac{g'(\theta')}{\pi(\theta')}\right) > h(\theta') - \log\!\left(\frac{g^{\star}_{\textnormal{max}}(\theta')}{\pi(\theta')}\right) . 
    \end{align}
    Again, the quantity in the right-hand side above does not depend on $\theta'$: 
    \begin{align}
    h(\theta') - \log\!\left(\frac{g^{\star}_{\textnormal{max}}(\theta')}{\pi(\theta')}\right) = h(\theta') - \log\!\left(\frac{\dfrac{e^{h(\theta')} \pi(\theta')}{\sup e^{h(\cdot)} \pi(\cdot)}}{\pi(\theta')}\right) = \log \sup_{\theta\in\Theta} e^{h(\theta)} \pi(\theta) ,
    \end{align}
    so combining the two lines above leads to
    \begin{align}
    \sup_{\theta\in\Theta}\left\{h(\theta) - \log\!\left(\frac{g'(\theta)}{\pi(\theta)}\right)\right\} > \log \sup_{\theta\in\Theta} e^{h(\theta)} \pi(\theta) ,
    \end{align}
    Now using \autoref{eq:dv:max:infsup}, we can rewrite the quantity in the right-hand side:
    \begin{align}
    \sup_{\theta\in\Theta}\left\{h(\theta) - \log\!\left(\frac{g'(\theta)}{\pi(\theta)}\right)\right\} > \inf_{g\in\mathcal{F}(\Theta)} \sup_{\theta\in\Theta}\left\{h(\theta) - \log\!\left(\frac{g(\theta)}{\pi(\theta)}\right)\right\}, 
    \end{align}
    which is precisely what we wanted to show. Hence, $g'$ cannot belong to the set of minimisers. This proves that 
    \begin{align}
    \argmin_{g\in\mathcal{F}(\Theta)} \sup_{\theta\in\Theta} \left\{h(\theta) - \log\!\left(\frac{g(\theta)}{\pi(\theta)}\right)\right\} = \left\{ g\in\mathcal{F}(\Theta) : g^{\star}_{\textnormal{max}} \preceq g \right\}, 
    \end{align}
    which concludes the proof.
    \end{proof}


\subsection{Proofs of Results in \autoref{sec:ExpFam}}
\label{proof:ExpFam}


\begin{proof}[Proof of \autoref{prop:posteriorPossibilty}]
    Given the probabilistic likelihood $p(x\given \theta) = \exp( \theta^{\intercal}T(x) - A(\theta) - B(x) )$, under the uninformative prior on $\Theta$, $\pi\rb{\theta} = 1$, the posterior possibility is given by
    \begin{subequations}
    \begin{align}
        \pi\rb{\theta \given x} &= \frac{p(x\given \theta) \pi\rb{\theta}}{\max_{\theta'} p(x\given \theta') \pi\rb{\theta'}} = \frac{p(x\given \theta)}{\max_{\theta'} p(x\given \theta')} \\
        &= \exp\rb{\rb{\theta-\mode{\theta}}^\intercal T(x) - \rb{A\rb{\theta}-A\rb{\mode{\theta}}}},
    \end{align}
    \end{subequations}
    where we denote the MLE by $\mode{\theta} \doteq \argmax_{\theta} p\rb{x\given \theta} = \argmax_{\theta} \theta^{\intercal}T(x) - A(\theta)$. The first-order condition for the maximum tells us that $T(x) = \nabla_{\theta} A\rb{\theta^{\thinstar}}$. Plugging this in the equation above, we get the result in \eqref{eq:PosteriorAsBregman}:
    \begin{subequations}
    \begin{align}
        \pi\rb{\theta \given x} &= \exp\rb{\rb{\theta-\mode{\theta}}^\intercal \nabla_{\theta} A\rb{\theta^{\thinstar}} - \rb{A\rb{\theta}-A\rb{\mode{\theta}}}} \\
        &= \exp\rb{-D_A\rb{\theta \| \mode{\theta}}}.
    \end{align}
    \end{subequations}
    Next, we get the result in \eqref{eq:possConjPrior} by rewriting the posterior as
    \begin{subequations}
    \begin{align}
        \pi\rb{\theta \given x} &= \exp\rb{\theta^\intercal T(x) - A\rb{\theta} -\rb{\theta^{\thinstar\intercal} T(x) - A\rb{\mode{\theta}}}} \\
        &= \exp\rb{\theta^\intercal \lambda - A\rb{\theta} - \leg{A} \rb{\lambda}},
    \end{align}
    \end{subequations}
    where $\lambda \doteq T(x)$ and $\leg{A}=\max_{\theta}\theta^\intercal T(x) - A(\theta)=\theta^{\thinstar\intercal} T(x) - A\rb{\mode{\theta}}$ is the Legendre transform of $A$. The resulting posterior is indeed of the same form as $g_{\lambda}$ in \eqref{eq:possExpFam}.
\end{proof}

\begin{proof}[Proof of \autoref{prop:PropertiesOfConjugatePriors}]
    First, we show that $f_{\theta}: \lambda \to g_{\lambda}\rb{\theta}$ is in $\mathcal{G}_{\leg{A}}\rb{\Lambda}\doteq \big\{ \lambda \mapsto \exp(\theta^\intercal \lambda - A^{\dagger\dagger}\rb{\theta} - A^{\dagger}\rb{\lambda} ) : \theta \in \Theta \big\}$. Since $A^{\dagger\dagger}=A$ (assuming $A$ is proper, closed, and convex), we have
    \begin{subequations}
    \begin{align}
        g_{\lambda}\rb{\theta} &= \exp\rb{\lambda^\intercal\theta - \leg{A}\rb{\lambda} - A\rb{\theta}} \\
        &= \exp\rb{\theta^\intercal \lambda - A^{\dagger\dagger}\rb{\theta} - A^{\dagger}\rb{\lambda}}. 
    \end{align}
    \end{subequations}
    Therefore, the mapping $f_{\theta}: \lambda \to g_{\lambda}\rb{\theta}$, which can be written as
    \begin{align}
        f_{\theta}: \lambda \to \exp\rb{\theta^\intercal \lambda - A^{\dagger\dagger}\rb{\theta} - A^{\dagger}\rb{\lambda}},
    \end{align}
    is in $\mathcal{G}_{\leg{A}}\rb{\Lambda}$.

    Next, we show that $g_{\lambda}^{\nu} \in \mathcal{G}_{\nu A}(\Theta)$. 
    It holds that
    \begin{subequations}
    \begin{align}
        g_{\lambda}^{\nu}\rb{\theta} &= \exp\rb{\nu\lambda^\intercal\theta - \nu \leg{A}\rb{\lambda} - \nu A\rb{\theta}} \\
        &= \exp\rb{\nu\lambda^\intercal\theta - \leg{\rb{\nu A}}\rb{\nu\lambda} - \nu A\rb{\theta}},
    \end{align}
    \end{subequations}
    where the second equality follows from the following property of the convex conjugate:
    \begin{align}
        \nu\leg{A}\rb{\lambda} &= \nu\sup_{\theta} \lambda^\intercal\theta - A\rb{\theta} = \sup_{\theta} \nu\lambda^\intercal\theta - \nu A\rb{\theta} = \leg{\rb{\nu A}}\rb{\nu\lambda}.
    \end{align}
    We conclude that $g_{\lambda}^{\nu} \in \mathcal{G}_{\nu A}(\Theta)$, as required.
    
    Finally, consider the likelihood $p(x\given \theta) = \exp( \theta^{\intercal}T(x) - A(\theta) - B(x) )$, and a choice of prior $g_{\lambda, \nu} \doteq g_{\lambda}^{\nu}$, with $g_{\lambda} \in \mathcal{G}_A(\Theta)$ defined as $g_{\lambda}\rb{\theta} = \exp\rb{\lambda^{\intercal}\theta - \leg{A}\rb{\lambda} - A\rb{\theta}}$, for some $\lambda \in \Lambda$. 
    The posterior possibility is given by
    \begin{subequations}
    \begin{align}
        g_{\lambda,\nu}\rb{\theta\given x} &= \frac{p(x\given \theta) g_{\lambda}^{\nu}\rb{\theta}}{\max_{\theta'} p(x\given \theta') g_{\lambda}^{\nu}\rb{\theta'}} \\
        &= \frac{\exp\rb{\theta^\intercal\rb{T(x)+\nu\lambda} - \rb{\nu+1} A\rb{\theta}}}{\max_{\theta'} \exp\rb{\theta'^\intercal\rb{T(x)+\nu\lambda} - \rb{\nu+1} A\rb{\theta'}}} \\
        &= \exp\rb{\theta^\intercal\rb{T(x)+\nu\lambda} - \rb{\nu+1} A\rb{\theta} - \rb{\nu+1} \max_{\theta'} \rb{\theta'^\intercal\rb{\frac{T\rb{x}+\nu\lambda}{\nu+1}} - A\rb{\theta'}}} \\
        &= \exp\rb{\theta^\intercal\rb{T(x)+\nu\lambda} - \rb{\nu+1} A\rb{\theta} - \rb{\nu+1}\leg{A}\rb{\frac{T\rb{x}+\nu\lambda}{\nu+1}}} \\
        &= g_{\frac{T(x)+\nu\lambda}{\nu+1},\nu+1}\rb{\theta}.
    \end{align}
    \end{subequations}
    Therefore, $g_{\lambda,\nu}$ is a valid conjugate prior for the likelihood $p(\cdot \given \theta)$, where the hyperparameters are updated as
    \begin{align}
        \lambda' = \frac{\nu\lambda + T(x)}{\nu+1}, \quad \nu' = \nu + 1.
    \end{align}

    As a sanity check, note that with $\nu=0$, $g$ is the uninformative prior, and we recover the posterior in \autoref{prop:posteriorPossibilty}. 

    For $n$ i.i.d.\ observations $\{x_i\}_{i=1}^n$, the likelihood product is $\prod_{i=1}^n p(x_i \given \theta)\propto \exp \rb{\theta^{\intercal}\sum_{i=1}^{n}T(x_i)-nA(\theta)}$. In this case, following the same derivation as above, the posterior possibility is given by
    \begin{align}
        g_{\lambda,\nu}\rb{\theta\given x_{1:n}} = g_{\frac{\sum_{i=1}^{n}T(x_i)+\nu\lambda}{\nu+n},\nu+n}\rb{\theta},
    \end{align}
    where the hyperparameters are now updated as
    \begin{align}
        \lambda' = \frac{\nu\lambda + \sum_{i=1}^{n}T(x_i)}{\nu+n}, \quad \nu' = \nu + n.
    \end{align}
    
\end{proof}

\begin{proof}[Proof of \autoref{prop:updateUpperCBO}]
    Recall that $\overline{\mathrm{CBO}}(\eta) \doteq \sup_{\theta\in\Theta} - \loss(\theta) - \log g_{\eta}(\theta)$, where the candidate is given by
    \begin{align}
        &g_{\eta}(\theta) = \exp\rb{-D_A(\theta\|\eta)} = \exp\big( \nabla A(\eta)^\intercal(\theta -\eta) + A(\eta) - A(\theta) \big) \\
        \implies\;&\!\!- \loss(\theta) - \log g_{\eta}(\theta) = - \loss(\theta) - \nabla A(\eta)^\intercal(\theta -\eta) - A(\eta) + A(\theta)
    \end{align}
    Danskin's theorem gives us $\nabla \overline{\mathrm{CBO}}(\eta) = - \nabla^2 A(\eta) \rb{\overline{\theta} - \eta}$, where $\overline{\theta} = \argsup_{\theta\in\Theta} - \loss(\theta) - \log g_{\eta}(\theta)$. The 2\textsuperscript{nd}-order Taylor approximation of the objective around $\eta$ is given by
    \begin{align}
        - \loss(\theta) - \log g_{\eta}(\theta) \approx - \loss(\eta) - \rb{\theta-\eta}^\intercal \nabla \loss(\eta) - \frac{1}{2} \rb{\theta-\eta}^{\intercal} \rb{\nabla^2 \loss(\eta) - \nabla^2 A(\eta)} \rb{\theta-\eta}
    \end{align}
    Under this approximation, assuming $\nabla^2\loss(\eta) - \nabla^2 A(\eta) \succ 0$ (so that $\overline{\mathrm{CBO}}(\eta)$ is finite),
    \begin{align}
        \overline{\theta} \approx \eta - \rb{\nabla^2 \loss\rb{\eta} - \nabla^2 A\rb{\eta}}^{-1} \nabla \loss\rb{\eta}
    \end{align}
    Minimizing $\overline{\mathrm{CBO}}(\eta)$ using natural gradients yields a Newton-like update rule:
    \begin{subequations}
    \begin{align}
        \eta_{t+1} &= \eta_t - \rho_t \rb{\nabla^2 A\rb{\eta_t}}^{-1} \nabla \overline{\mathrm{CBO}}(\eta_t) \\
        &\approx \eta_t - \rho_t \rb{\nabla^2 \loss\rb{\eta_t} - \nabla^2 A\rb{\eta_t}}^{-1} \nabla \loss\rb{\eta_t}
    \end{align}
    \end{subequations}
    where $\rho_t>0$ is the step size.
\end{proof}

\section{Discussion on Max-Relative Entropy}

\subsection{Relationship with Generalized VI}
\label{app:GVI}

In classical VI, the optimization is typically framed through the decomposition of the ELBO, which involves the reverse KL divergence, $\mathrm{KL}\!\left(q\Vert q^{\star}_{\textnormal{add}}\right)$. This yields a lower bound on the model evidence and has well-known properties such as the mode-seeking effect. Yet, in the probabilistic literature, many alternative divergences and bound constructions have been explored, leading to a variety of upper and lower bounds on the marginal likelihood. Notable examples include the $\chi$-divergence upper bounds \citep{dieng2017variational} and variational R\'{e}nyi bounds \citep{li2016renyi}. In each of these cases, one can define dual optimization problems -- minimisation to the left or to the right of the divergence -- corresponding to lower or upper bounds 
providing a sandwiching of the true model evidence.

Remarkably, in the possibilistic framework, an analogous structure arises naturally, but with the max-relative entropy replacing the KL divergence. The lower and upper consistency bounds (CBOs) correspond to two dual optimization perspectives: maximizing the lower CBO or minimizing the upper CBO, which yield, respectively, under- and over-estimation of the posterior possibility degrees. This is conceptually analogous to probabilistic VI, where reverse or forward KL leads to mode-seeking versus mass-covering behaviours. In fact, one can interpret the upper (lower) CBO as a limiting case of a Rényi- or $\chi^n$-type bound when the divergence order tends to infinity, producing an extreme \textit{mass-covering} (resp.\ \textit{mode-seeking}) effect.




\subsection{Intuition Through Visualization}
\label{app:intuition}

Recall the definition of max-relative entropy in \autoref{eq:max-rel-ent}: $D_{\textnormal{max}}(g \,\|\, f)\doteq \sup_{\theta \in \Theta} \log \frac{g(\theta)}{f(\theta)}$.
The supremum is effectively computed over the points where $g\geq f$, and $D_{\textnormal{max}}(g \,\|\, f) = 0$ if and only if $g$ lower bounds $f$. To illustrate, we present two sinusoidal possibility functions in \autoref{fig:active-region}, and visualize the \textit{active region} for computing $D_{\textnormal{max}}(g \,\|\, f)$, which is marked in darker shades.

Looking at \autoref{eq:min-divergence}, we note that maximizing $\underline{\mathrm{CBO}}\rb{g}$ amounts to minimizing $D_{\textnormal{max}}(g \,\|\, g^\thinstar_{\max})$. Therefore, the set of maximizers consists of possibility functions lower bounding the true posterior, $g^\thinstar_{\max}$. Similarly, minimizing $\overline{\mathrm{CBO}}\rb{g}$ amounts to minimizing $D_{\textnormal{max}}(g^\thinstar_{\max} \,\|\, g)$, which is achieved by possibility functions upper bounding the true posterior. This is visualized in \autoref{fig:bayeslr-opts} for Gaussian candidates; see next subsection for a description of the inference problem.

\subsection{Example: Bayesian Logistic Regression}
\label{app:ex-bayeslr}

We perform Bayesian logistic regression for a simple problem to visualize what the confidence bounds look like for different candidate distributions, and how the true posterior compares with the optimizers of upper/lower CBO.

\begin{figure}
\centering
    \subcaptionbox{\label{fig:active-region}}{\vstretch{1.225}{\includegraphics[width=0.48\linewidth]{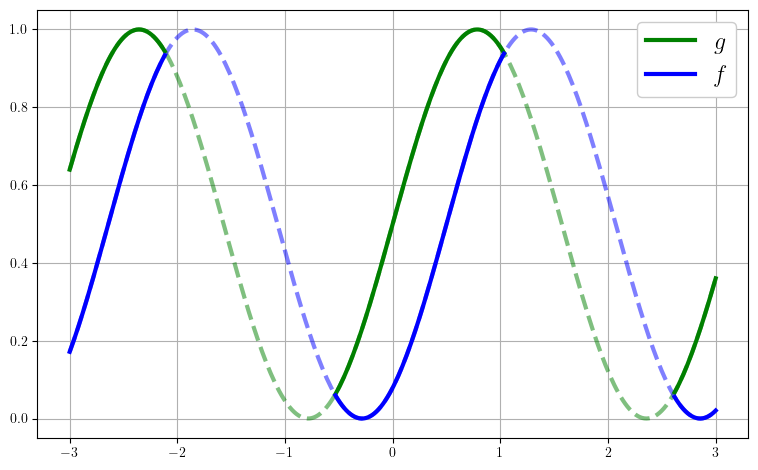}}}
    \subcaptionbox{\label{fig:bayeslr-ucbo}}{\includegraphics[width=0.48\linewidth]{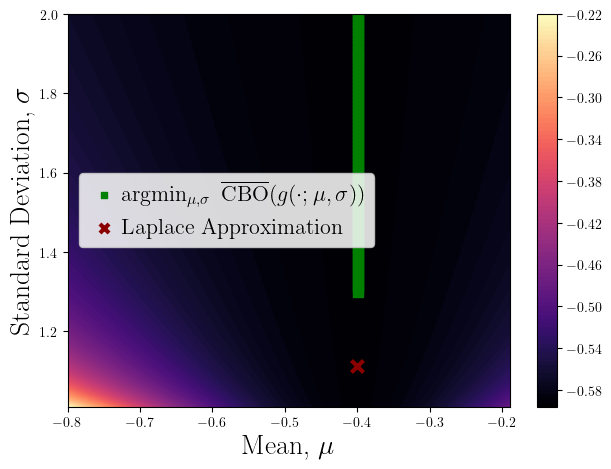}} \\ \vspace{2mm}
    \subcaptionbox{\label{fig:bayeslr-lcbo}}{\includegraphics[width=0.48\linewidth]{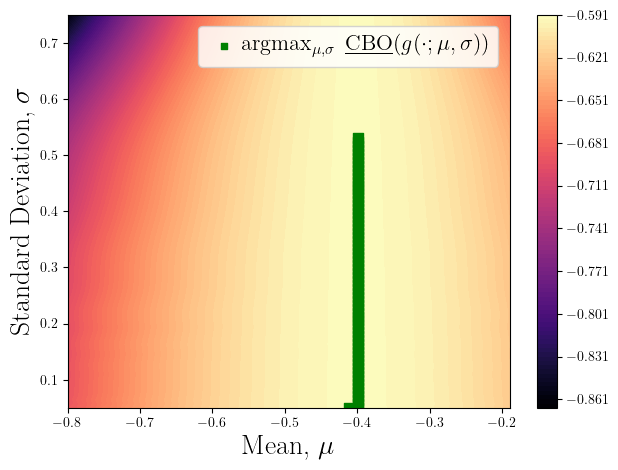}}
    \subcaptionbox{\label{fig:bayeslr-opts}}{\includegraphics[width=0.48\linewidth]{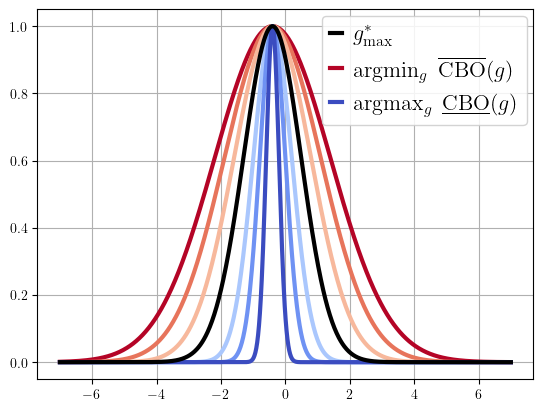}}
\caption{Building an intuition for Maxitive-VI. \textit{Top Left:} Solid lines depict active regions in the supremum computation of $D_{\textnormal{max}}(g \,\|\, f)$ for two sinusoidal possibility functions; dashed lines depict inactive regions. \textit{Top Right:} Minimizers of $\overline{\mathrm{CBO}}$ for the Bayesian LR problem described in \autoref{app:ex-bayeslr}, along with the Laplace approximation, showing consistency in the mean computation while allowing for varying precision levels. \textit{Bottom Left:} Maximizers of $\underline{\mathrm{CBO}}$. \textit{Bottom Right:} Numerically computed maxitive posterior, $g^\star_{\max}$, and the CBO optimizers, showing how they bound $g^\star_{\max}$ (from either side, depending on the formulation).}
\end{figure}

Consider an observable random Bernoulli variable, $\mathbf{y} \sim \mathrm{Bern}(\sigma(\bm{f}))$, where $\sigma$ is the sigmoid function, $\sigma(x) = (1+\exp(-x))^{-1}$. We assume for $\bm{f}$ to be a deterministic uncertain variable, and put a Gaussian possibilistic prior on it, i.e.\ $\bm{f}$ is described by $\overline{\mathrm{N}}(0,1)$. The maxitive posterior is then given by
\begin{align}
    g^\thinstar_{\max}(f) = \frac{p(y|f)\pi(f)}{\sup_{f'}p(y|f')\pi(f')} = \frac{\sigma(yf)\exp\left(-\frac{1}{2}f^2\right)}{\sup_{f'}\sigma(yf')\exp\left(-\frac{1}{2}f'^2\right)}
\end{align}

The denominator is not analytically tractable, hence, we perform VI with the family of Gaussian distributions, $\mathcal{G} = \{\overline{\mathrm{N}}(\mu,\sigma^2): \mu\in\mathbb{R}, \sigma\in\mathbb{R}^+\}$.

Say $y=-1$. In \autoref{fig:bayeslr-ucbo}, we show the minimizers of $\overline{\mathrm{CBO}}$, along with the solution computed using Laplace approximation, i.e.\ values of $\mu$ and $\sigma$ that match the first and second (possibilistic) moments of $g^\thinstar_{\max}$. We note that the minimizers match the mean of the Laplace approximation, while allowing arbitrarily large values for the precision of the approximate-posterior, each of which yields a model consistent with the data. Similarly, in \autoref{fig:bayeslr-lcbo}, we show the maximizers of $\underline{\mathrm{CBO}}$, which also match the mean, while allowing for the solution to be arbitrarily precise (lower $\sigma$). Finally, in \autoref{fig:bayeslr-opts}, we show the true maxitive posterior, computed numerically, as well as the optimizers computed using the VI formulation, showing how the minimizers of $\overline{\mathrm{CBO}}$ upper bound the true posterior, while the maximizers of $\underline{\mathrm{CBO}}$ lower bound it.

\section{Possibilistic Exponential Family}
\label{app:ExpFamInGeneral}

We start by clarifying some properties of possibilistic exponential families: 
\begin{enumerate}[label=\textbf{\;\;Property \arabic*:}, nosep,leftmargin=0pt, itemindent=20pt, align=left]
    \item There is no general correspondence between the components $T_+$ and $B_+$ of a probabilistic (canonical) exponential family and their analogue in possibility theory.
    \item Although possibility functions are not densities and, therefore, do not require the definition of a reference measure, the base measure of a possibilistic exponentially family need not be an indicator function for the support of the family.
    \item The log-partition function $A$ is simpler when compared to the expression of the log-partition $A_+$ for a probabilistic exponential family, that is
    \begin{align}
    A_+(\lambda) = \log \int \exp\big( \lambda^{\intercal} T_+(\theta) - B_+(\theta) \big) \mathrm{d} \theta,
    \end{align}
    since $\sup$ and $\log$ can be exchanged when defining $A$. 
\end{enumerate}

These properties are illustrated in the following examples. 

\begin{example}
Consider a Bernoulli-style possibility function: set $\Theta = \{0,1\}$ and define the possibility of $\bm{\theta} = 0$ as $\alpha_0$ and the possibility of $\bm{\theta} = 1$ as $\alpha_1$, with $\max\cb{\alpha_0,\alpha_1} = 1$ by construction. The corresponding possibility function is $f(\theta \given \alpha_0, \alpha_1)= \alpha_0^{1-\theta}\alpha_1^{\theta}$. This can be expressed in an exponential family form as
\begin{align}
g_{\lambda}(\theta) = \exp\big( \lambda^{\intercal}T(\theta) - A(\lambda)
\big),
\end{align}
with $T(\theta) = (1-\theta, \theta)$, $\lambda = (\lambda_1,\lambda_2) \in \Lambda = \mathbb{R}^2$ and
\begin{align}
    A(\lambda) = \sup_{\theta \in \{0,1\}} \lambda^{\intercal} T(\theta) = \max\{\lambda_1, \lambda_2\}.
\end{align}
These components differ significantly from those of a Bernoulli distribution. Although the parameters $\alpha_0$ and $\alpha_1$ are not directly related, the fact that their maximum is equal to $1$  us to express this exponential family with a single variable: let $\eta = \lambda_2 - \lambda_1$, then another possible exponential-family form for this Bernoulli-style possibility function is
\begin{align}
    g_{\eta}(\theta) = \exp(\eta\theta - A(\eta)),
\end{align}
with $A(\eta) = \max\{0, \eta\}$.
Despite its simplicity, the Bernoulli possibility function is useful to express, for instance, possibilities of detection in tracking problems, see e.g.\ \citet{tracking2020ristic}.
\end{example}

\begin{example}
Consider a Poisson-style possibility function: set $\Theta = \mathbb{N}_0$ and consider a parameter $\alpha \in \mathbb{R}^+$. The corresponding possibility function is
\begin{align}
f(\theta \given \alpha) = \frac{\alpha^{\theta - \lfloor\alpha\rfloor}\lfloor{\alpha}\rfloor!}{\theta!},
\end{align}
which can be expressed in an exponential family form as
\begin{align}
g_{\lambda}(\theta) = \exp\big( \lambda^{\intercal} \theta - A(\lambda) - B(\theta) \big).
\end{align}
with $\lambda = \log \alpha$, $B(\theta) = \log(\theta!)$, and $A(\lambda) = \lfloor e^{\lambda}\rfloor\lambda - \log\rb{\lfloor e^{\lambda}\rfloor!}$.
\end{example}

\begin{example}
Consider the univariate normal possibility function with known variance $\sigma^2$, $\overline{\mathrm{N}}(\theta;\mu,\sigma^2)$, for which we have $B(\theta) = \theta^2/(2\sigma^2)$, whereas the base measure of the corresponding univariate normal distribution is
\begin{align}
\frac{\theta^2}{2\sigma^2} + \frac{1}{2}\log (2\pi\sigma^2). 
\end{align}
\end{example}

\begin{example}
Consider the case where $\Theta = [0,1]$, $B(\theta) = 0$ and $T(\theta) = (T_1(\theta), T_2(\theta))$, with $T_1(\theta) = \theta$ and $T_2(\theta) = \bm{1}_0(\theta)$, i.e., the indicator of the point $0$. The corresponding possibilistic exponential family is of the form
\begin{align}
    g_{\lambda}(\theta) = \exp(\lambda_1 \theta + \lambda_2 \bm{1}_0(\theta) - \max\{0,\lambda_1,\lambda_2\}),
\end{align}
which takes the component $T_2$ into account. If we considered the analogous probabilistic exponential family, the obtained density would almost everywhere equal to
\begin{align}
q_{\lambda_1}(\theta) \propto \exp\bigg( \lambda_1 \theta - \log \int \exp(\lambda_1 \theta) \mathrm{d}\theta \bigg),
\end{align}
since the term $\bm{1}_0(\theta)$ only affects a single point (of measure zero). This highlights how minimal representations of exponential families can differ between the probabilistic and possibilistic case, due to the respective behaviours of integration and optimization.
\end{example}

\section{Algorithms}
\label{sec:algo}

In the main paper, we present the algorithm for \texttt{uCBOpt} (Alg.~\ref{alg:ucbopt}). Here, we present the algorithms for \texttt{uCBOpt-adapt} and \texttt{lCBOpt-adapt} in Algorithms~\ref{alg:ucbopt-adapt} and~\ref{alg:lcbopt-adapt}, respectively (main differences are highlighted in blue). Both methods track a data-plus-prior curvature estimate $\tilde{h}_t = h_t + \delta$, where $h_t$ is an exponential moving average of the squared stochastic gradients and $\delta$ is the weight-decay/prior curvature. The variable $c_t$ then forms a one-sided candidate curvature envelope of $\tilde{h}_t$: \texttt{uCBOpt-adapt} uses a growing running minimum, while \texttt{lCBOpt-adapt} uses a decayed running maximum.

In \texttt{uCBOpt-adapt}, $c_t$ tracks the data-plus-prior curvature through $c_t = \min\{\beta_3 c_{t-1}, \bar h_t\}$ with $\beta_3 > 1$. This behaves like a running minimum that can gradually increase over time when the observed curvature grows. The update divides by the residual curvature $\tilde{h}_t - \gamma c_t$, which can be interpreted as the part of the data-plus-prior sharpness that is not accounted for by the candidate curvature
$c_t$.

In \texttt{lCBOpt-adapt}, $c_t$ is instead updated as $c_t = \max\{\beta_3 c_{t-1}, \tilde{h}_t\}$ with $\beta_3 \in [0,1)$. This behaves like a running maximum that can gradually decay over time when the observed curvature decreases. The update divides by $\gamma c_t - \tilde{h}_t$, with $\gamma > 1$, which measures the gap between the scaled upper candidate curvature and the current data-plus-prior curvature estimate. Thus, the two
adaptive variants differ mainly in whether the curvature correction is constructed from a lower or upper envelope of $\tilde{h}_t$.

Below we provide some guidelines for choosing the hyperparameters in \texttt{uCBOpt}, \texttt{uCBOpt-adapt} and \texttt{lCBOpt-adapt}:

\paragraph{Learning rate schedule $\rho_t$.} We found that the learning-rate schedule used for SGD and AdamW is also a good starting point for the \texttt{CBOpt} family. In all experiments, we use a short linear warmup followed by cosine annealing to zero. For \texttt{uCBOpt}, learning rates comparable to SGD can work well on larger datasets, while smaller learning rates are generally safer for the adaptive variants \texttt{uCBOpt-adapt} and \texttt{lCBOpt-adapt}.

\paragraph{Weight decay $\delta$.} The weight decay can be chosen similarly to SGD or AdamW, since it plays the role of a prior curvature term in the data-plus-prior curvature estimate. Weight-decay values that work well for IVON also tend to work well for the \texttt{CBOpt} family.

\paragraph{Gradient momentum $\beta_1$.} We set $\beta_1=0.9$ throughout, following the standard momentum value used in SGD and Adam. We did not find it necessary to tune this parameter extensively.

\paragraph{Hessian momentum $\beta_2$.} The Hessian momentum should be set close to one so that the curvature estimate changes smoothly. We use $\beta_2=0.99999$ in all \texttt{CBOpt} experiments, following the setting commonly used in IVON. Smaller values may make the curvature estimate more responsive, but can also introduce instability, especially when combined with adaptive curvature tracking.

\paragraph{Curvature parameters $\vartheta,\gamma,\beta_3$.} For \texttt{uCBOpt}, the curvature parameter $\vartheta$ controls the fixed candidate curvature. A small value is usually preferable, often one to three orders of magnitude smaller than the weight decay. If $\vartheta$ is too small, the method behaves similarly to the base curvature update, whereas overly large values can excessively distort the effective preconditioner and degrade performance. For \texttt{uCBOpt-adapt} and \texttt{lCBOpt-adapt}, $\gamma$ controls the strength of the curvature correction. For the adaptive variants, $\gamma$ controls the strength of the curvature correction, while $\beta_3$ controls the evolution of the candidate curvature. In \texttt{uCBOpt-adapt}, $\beta_3>1$ controls the growth rate of the running-minimum candidate; in \texttt{lCBOpt-adapt}, $\beta_3<1$ controls the decay rate of the running-maximum candidate. We find that $\gamma$ values close to one work well, using $\gamma=0.9$ for \texttt{uCBOpt-adapt} and $\gamma=1.02$ for \texttt{lCBOpt-adapt}. We find $\beta_3=1.001$ for \texttt{uCBOpt-adapt} and $\beta_3=0.999$ for \texttt{lCBOpt-adapt} to be good default choices.

\paragraph{Hessian initialization $h_0$.} The Hessian initialization controls the initial scale of the curvature estimate. Smaller values allow more adaptive behaviour early in training, while larger values make the initial posterior more concentrated and can stabilize training. In our experiments, we find that $h_0 \in \{0.05, 0.1\}$ tends to work well for the \texttt{CBOpt} family.

\paragraph{Stability term $\epsilon$.} In our implementation of \texttt{uCBOpt-adapt} and \texttt{lCBOpt-adapt}, we add a small numerical stability constant $\epsilon = 10^{-8}$ to the denominator, following Adam, to prevent division by near-zero curvature estimates.

\paragraph{Batch size, training epochs.} Batch sizes and training budgets that work well for SGD and AdamW generally also work well for the \texttt{CBOpt} family. For the adaptive variants, since curvature estimates are accumulated during training, longer training can improve their stability, especially for larger models and datasets.

\begin{figure*}
\begin{minipage}{0.48\textwidth}
\hrule
\vspace{0.4em}
\captionsetup{type=algorithm}
\caption{\textbf{u}pper \textbf{CB}O Optimization with adaptive curvature (\texttt{uCBOpt-adapt})}
\label{alg:ucbopt-adapt}
\vspace{-0.2em}
\hrule
\vspace{0.4em}

\begin{algorithmic}[1]
    \STATE \textbf{Require:} Learning rates $\{\rho_t\}$, weight decay $\delta>0$,
    curvature parameters \textcolor{blue}{$\gamma\in[0,1)$ and $\beta_3>1$},
    momentum parameters $\beta_1,\beta_2\in[0,1)$, and Hessian init $h_0>0$.
    \STATE \textbf{Init:} $\eta \leftarrow$ (NN-weights), $m\leftarrow 0$,
    $h\leftarrow h_0$, \textcolor{blue}{$c\leftarrow +\infty$}.
    \FOR{$t = 1,2,\ldots$}
        \STATE $\hat{g} \leftarrow \widehat{\nabla}\ell(\eta)$
        \STATE $\hat{h} \leftarrow \hat{g}\odot\hat{g}$
        \STATE $m \leftarrow \beta_1 m + (1-\beta_1)\hat{g}$
        \STATE $h \leftarrow \beta_2 h + (1-\beta_2)\hat{h}$
        \STATE $\tilde{h} \leftarrow h+\delta$
        \STATE \textcolor{blue}{$c \leftarrow \min\{\beta_3 c,\, \tilde{h}\}$}
        \STATE $\bar{m} \leftarrow m/(1-\beta_1^t)$
        \STATE $\eta \leftarrow \eta - \rho_t(\bar{m}+\delta\eta)/\textcolor{blue}{(\tilde{h}-\gamma c)}$
    \ENDFOR
    \RETURN $\eta$
\end{algorithmic}

\vspace{0.4em}
\hrule
\end{minipage}
\hfill
\begin{minipage}{0.48\textwidth}
\hrule
\vspace{0.4em}
\captionsetup{type=algorithm}
\caption{\textbf{l}ower \textbf{CB}O Optimization with adaptive curvature (\texttt{lCBOpt-adapt})}
\label{alg:lcbopt-adapt}
\vspace{-0.2em}
\hrule
\vspace{0.4em}

\begin{algorithmic}[1]
    \STATE \textbf{Require:} Learning rates $\{\rho_t\}$, weight decay $\delta>0$,
    curvature parameters \textcolor{blue}{$\gamma>1$ and $\beta_3\in[0,1)$},
    momentum parameters $\beta_1,\beta_2\in[0,1)$, and Hessian init $h_0>0$.
    \STATE \textbf{Init:} $\eta \leftarrow$ (NN-weights), $m\leftarrow 0$,
    $h\leftarrow h_0$, \textcolor{blue}{$c\leftarrow h_0+\delta$}.
    \FOR{$t = 1,2,\ldots$}
        \STATE $\hat{g} \leftarrow \widehat{\nabla}\ell(\eta)$
        \STATE $\hat{h} \leftarrow \hat{g}\odot\hat{g}$
        \STATE $m \leftarrow \beta_1 m + (1-\beta_1)\hat{g}$
        \STATE $h \leftarrow \beta_2 h + (1-\beta_2)\hat{h}$
        \STATE $\tilde{h} \leftarrow h+\delta$
        \STATE \textcolor{blue}{$c \leftarrow \max\{\beta_3 c,\, \tilde{h}\}$}
        \STATE $\bar{m} \leftarrow m/(1-\beta_1^t)$
        \STATE $\eta \leftarrow \eta - \rho_t(\bar{m}+\delta\eta)/\textcolor{blue}{(\gamma c-\tilde{h})}$
    \ENDFOR
    \RETURN $\eta$
\end{algorithmic}

\vspace{0.4em}
\hrule
\end{minipage}
\end{figure*}

\section{Experimental Details}
\label{sec:exp_details}

\subsection{Computational Efficiency}
\label{subsec:comp_efficiency}

In \autoref{tab:runtime-memory}, we report the average runtime over the full training run: 100 epochs for Fashion-MNIST and 200 epochs for CIFAR-10 and CIFAR-100. We also report the peak allocated and reserved GPU memory. All experiments are conducted on a single NVIDIA H100 GPU, with three random seeds run in parallel. The \texttt{CBOpt} variants include \texttt{uCBOpt}, \texttt{uCBOpt-adapt}, and \texttt{lCBOpt-adapt}. The runtime and memory are similar to existing optimizers such as SGD and Adam, as \texttt{CBOpt} only requires maintaining per-parameter gradient and scalar curvature estimates, similar to the moment estimates used in Adam.

\begin{table*}[ht]
\centering
\caption{Runtime and memory usage for classification experiments.}
\label{tab:runtime-memory}
\setlength{\tabcolsep}{4pt}
\renewcommand{\arraystretch}{1.1}
\resizebox{0.85\linewidth}{!}{\begin{tabular}{llcccc}
\toprule
Dataset / Model & Metric & AdamW & SGD & IVON & \texttt{CBOpt} variants \\
\midrule

\multirow{3}{*}{\makecell[l]{Fashion-MNIST\\LeNet}}
& Runtime (seconds)          & 3.26 & 3.52 & 3.30 & 3.28 \\
& Memory allocated (GiB) & 0.07 & 0.07 & 0.07 & 0.07 \\
& Memory reserved (GiB)  & 0.09 & 0.09 & 0.09 & 0.09 \\
\midrule

\multirow{3}{*}{\makecell[l]{CIFAR-10\\ResNet-20}}
& Runtime (seconds)          & 13.7 & 13.7 & 13.8 & 13.7 \\
& Memory allocated (GiB) & 0.60 & 0.59 & 0.60 & 0.60 \\
& Memory reserved (GiB)  & 0.83 & 0.82 & 0.82 & 0.83 \\
\midrule

\multirow{3}{*}{\makecell[l]{CIFAR-100\\DenseNet-121}}
& Runtime (seconds)          & 84.5 & 83.2 & 83.5 & 86.9 \\
& Memory allocated (GiB) & 4.52 & 4.52 & 4.53 & 4.52 \\
& Memory reserved (GiB)  & 4.95 & 4.95 & 4.96 & 4.95 \\

\bottomrule
\end{tabular}}
\end{table*}

\subsection{Architectures and Datasets}
\label{subsec:arch_dataset}

\textbf{In-domain datasets.} For in-domain tasks, we consider three standard image classification benchmarks: \textbf{Fashion-MNIST}, consisting of grayscale images from 10 clothing categories; \textbf{CIFAR-10}, consisting of RGB natural images from 10 classes; and \textbf{CIFAR-100}, a more fine-grained RGB natural image dataset with 100 classes.

\textbf{Out-of-domain datasets.} For out-of-domain evaluation, we consider three datasets: \textbf{EMNIST}, consisting of grayscale images of handwritten English letters (26 classes); \textbf{SVHN}, consisting of RGB images of real-world house numbers cropped from Google Street View (10 classes); and \textbf{TinyImageNet}, consisting of RGB natural images from 200 object classes.

\textbf{Data preprocessing.}
All images were normalized channel-wise using standard normalization constants. During training, CIFAR-10 and CIFAR-100 images were augmented with random horizontal flips and random crops with padding, while Fashion-MNIST used no data augmentation.

\textbf{Architectures.}
We consider two image classification models: LeNet and ResNet-20. On Fashion-MNIST, we train a LeNet; on CIFAR-10 we train a ResNet-20; and on CIFAR-100 we train a DenseNet-121.
\begin{itemize}[leftmargin=*, topsep=-2pt, noitemsep]
\item \textbf{LeNet.} Consists of two convolutional layers with ReLU activations and max-pooling, followed by three fully connected layers. The model has 44,426 trainable parameters for Fashion-MNIST.
\item \textbf{ResNet-20.} Consists of an initial convolutional layer followed by three residual stages with channel widths 16, 32, and 64, using downsampling at the beginning of the latter two stages. Global average pooling is applied before the final linear classifier. We use Filter Response Normalization (FRN) in the convolutional blocks, which normalizes each filter response using its spatial second moment and applies learned affine and threshold parameters. The model has 274,042 trainable parameters for CIFAR-10.
\item \textbf{DenseNet-121.} Consists of an initial convolutional layer followed by four dense blocks with 6, 12, 24, and 16 bottleneck layers, respectively, using transition layers between blocks for downsampling and channel compression. Each bottleneck uses FRN-normalized $1 \times 1$ and $3 \times 3$ convolutions, and concatenates its output with previous feature maps to promote feature reuse. Global average pooling is applied before the final linear classifier. We use growth rate 12 and compression factor 0.5. The model has 1,050,928 trainable parameters for CIFAR-100.
\end{itemize}
For MC Dropout variants, we use the same backbone architectures but keep dropout active at inference time to obtain stochastic predictions. In LeNet, dropout is applied after the first two fully connected hidden layers. In ResNet-20, it is applied within the residual blocks after the FRN-normalized convolutional responses and residual addition. In DenseNet-121, dropout is applied inside the dense bottleneck layers after the FRN-normalized convolutional transformations and before concatenating the newly produced feature maps with the accumulated dense-block features.

\subsection{Hyperparameters}

Across all experiments, we use a 90/10 train-validation split. Each model is trained for a fixed number of epochs, and test performance is reported using the checkpoint with the lowest validation loss. We use a batch size of 128 for training and 256 for validation and testing. All methods use a linear warmup for the first 5 epochs, followed by cosine learning-rate annealing to zero. For the baselines, hyperparameters are tuned only on Fashion-MNIST, while those for the remaining datasets follow \citet{shen2024variational}. For the \texttt{CBOpt} family, we extensively tune the curvature parameters only on Fashion-MNIST. For CIFAR-10, we tune the learning rate and, for \texttt{uCBOpt}, the curvature parameter $\vartheta$. For \texttt{uCBOpt-adapt} and \texttt{lCBOpt-adapt}, we reuse the values of $\gamma$ and $\beta_3$ that perform well on Fashion-MNIST. For CIFAR-100, we use the same hyperparameter settings as CIFAR-10.

\textbf{Fashion-MNIST.} We train all models for 100 epochs. Hyperparameters are selected based on the lowest validation loss over 30 epochs. Details of the hyperparameter selection are provided below.

\begin{itemize}[leftmargin=*, noitemsep]
\item \textbf{AdamW.} We tune the learning rate $\rho$ over a grid from $10^{-3}$ to $10^{-2}$ and weight decay $\delta$ from $10^{-5}$ to $5\cdot10^{-3}$. We use $\beta_1=0.9$ and $\beta_2=0.999$. The chosen hyperparameters are $\rho=10^{-3}$, $\delta=10^{-2}$.
\item \textbf{SGD.} We tune the learning rate $\rho$ over a grid from $10^{-3}$ to $10^{-1}$ and weight decay $\delta$ from $10^{-5}$ to $5\cdot10^{-3}$. We use momentum $0.9$. The chosen hyperparameters are $\rho=5\cdot10^{-3}$ and $\delta=10^{-5}$.
\item \textbf{IVON.} We tune the learning rate $\rho$ over a grid from $10^{-3}$ to $2\cdot10^{-1}$, weight decay $\delta$ from $10^{-5}$ to $5\cdot10^{-3}$, and initial Hessian scale over $h_0 \in \{0.05, 0.1, 0.5\}$. We use $\beta_1=0.9$, $\beta_2=0.99999$, effective sample size $54000$ (equivalent to the number of training samples) and one MC sample during training. The chosen hyperparameters are $\rho=2\cdot10^{-1}$, $\delta=2\cdot10^{-3}$, and $h_0=0.5$. We also apply the learning-rate rescaling procedure from \citet{shen2024variational}.
\item \textbf{MC-Dropout.} We use the LeNet MC-Dropout variant with dropout probability $p=0.05$, keeping dropout active at test time. The model is trained with SGD using $\rho=5\cdot10^{-3}$ and $\delta=10^{-5}$.
\item \textbf{Laplace.} We first train a MAP LeNet model using SGD with momentum $0.9$, $\rho=5\cdot10^{-3}$, and $\delta=10^{-5}$. We then fit a post-hoc last-layer Laplace approximation with a Kronecker-factored Hessian structure. Predictions use the GLM predictive approximation with a probit link and 100 samples; the prior precision is initialized to $0.54$ (corresponding to the same prior setup as SGD) and optimized by marginal likelihood.
\item \textbf{SWAG.} We first perform standard SGD training with momentum $0.9$ using cosine annealing from an initial learning rate of $5\cdot10^{-3}$ to a final learning rate of $10^{-3}$ over the first 60 epochs. We then run 40 additional SWAG epochs with a constant learning rate of $5\cdot10^{-3}$, collecting SWAG statistics every epoch and maintaining a low-rank approximation of the posterior. We use weight decay $\delta=10^{-5}$.
\item \textbf{\texttt{uCBOpt}.} We tune the learning rate $\rho$ over a grid from $10^{-3}$ to $10^{-1}$, weight decay $\delta$ from $10^{-5}$ to $5\cdot10^{-3}$, initial Hessian scale over $h_0 \in \{0.05, 0.1, 0.5, 1.0\}$, and curvature parameter $\vartheta$ over a fine-grained grid from $0$ to $10^{-4}$. We use $\beta_1=0.9$ and $\beta_2=0.99999$. The chosen hyperparemeters are $\rho=10^{-2}$, $\delta=2\cdot10^{-3}$, $h_0=0.05$, and $\vartheta=8 \cdot10^{-6}$.
\item \textbf{\texttt{uCBOpt-adapt}.} We tune the learning rate $\rho$ over a grid from $10^{-3}$ to $10^{-1}$, weight decay $\delta$ from $10^{-5}$ to $5\cdot10^{-3}$, initial Hessian scale over $h_0 \in \{0.05, 0.1, 0.2\}$, curvature scaling parameter $\gamma$ from $0.0$ to $0.95$, and curvature tracking parameter $\beta_3$ over $\{1.00001, 1.0001, 1.001, 1.01\}$. We use $\beta_1=0.9$ and $\beta_2=0.99999$. The chosen hyperparameters are $\rho=10^{-2}$, $\delta=2\cdot 10^{-3}$, $h_0=0.05$, $\gamma=0.9$, and $\beta_3=1.001$.
\item \textbf{\texttt{lCBOpt-adapt}.} We tune the learning rate $\rho$ over a grid from $10^{-3}$ to $10^{-1}$, weight decay $\delta$ from $10^{-5}$ to $5\cdot10^{-3}$, initial Hessian scale over $h_0 \in \{0.05, 0.1, 0.2\}$, curvature scaling parameter $\gamma$ from $1.0001$ to $2.0$, and curvature tracking parameter $\beta_3$ over $\{0.9, 0.99, 0.999, 0.9999, 0.99999\}$. We use $\beta_1=0.9$ and $\beta_2=0.99999$. The chosen hyperparameters are $\rho=2\cdot 10^{-3}$, $\delta=2\cdot 10^{-3}$, $h_0=0.1$, $\gamma=1.02$, and $\beta_3=0.999$.

\end{itemize}

\textbf{CIFAR-10 \& CIFAR-100.} We train all models for 200 epochs. Details of the hyperparameters are provided below.
\begin{itemize}[leftmargin=*, noitemsep]
\item \textbf{AdamW.} We use $\rho=2\cdot10^{-3}$ and $\delta=2\cdot10^{-4}$, with $\beta_1=0.9$ and $\beta_2=0.999$.
\item \textbf{SGD.} We use $\rho=10^{-1}$, momentum $0.9$, and $\delta=2\cdot10^{-4}$.
\item \textbf{IVON.} We use $\rho=0.2$, $\delta=2\cdot10^{-4}$, $h_0=0.5$, $\beta_1=0.9$, $\beta_2=0.99999$, effective sample size $45000$ (equivalent to the number of training samples) and one MC sample during training. We also apply the learning-rate rescaling procedure from \citet{shen2024variational}.
\item \textbf{MC-Dropout.} We use the ResNet-20 (for CIFAR-10) and DenseNet-121 (for CIFAR-100) MC-Dropout variant with dropout probability $p=0.05$, keeping dropout active at test time. The model is trained with SGD using momentum $0.9$, $\rho=10^{-1}$, and $\delta=2\times10^{-4}$.
\item \textbf{Laplace.} We first train MAP ResNet-20 (for CIFAR-10) and DenseNet-121 (for CIFAR-100) models using SGD with momentum $0.9$, $\rho=10^{-1}$, and weight decay $\delta=2\cdot10^{-4}$. We then fit a post-hoc last-layer Laplace approximation with a Kronecker-factored Hessian structure. Predictions use the GLM predictive approximation with a probit link and 100 samples; the prior precision is initialized to $9$ (corresponding to the same prior setup as SGD) and optimized by marginal likelihood.
\item \textbf{SWAG.} We first perform standard SGD training with momentum $0.9$ using cosine annealing from an initial learning rate of $5\cdot10^{-2}$ to a final learning rate of $10^{-2}$ over the first 160 epochs. We then run 40 additional SWAG epochs with a constant learning rate of $5\cdot10^{-3}$, collecting SWAG statistics every epoch and maintaining a low-rank approximation of the posterior. We use weight decay $\delta=2\cdot10^{-4}$.
\item \textbf{\texttt{uCBOpt}.} We use $\rho=10^{-1}$, $\delta=2\cdot10^{-4}$, $h_0=0.05$, $\vartheta=5 \cdot10^{-5}$, $\beta_1=0.9$, and $\beta_2=0.99999$.
\item \textbf{\texttt{uCBOpt-adapt}.} We use $\rho=5\cdot10^{-3}$, $\delta=2\cdot 10^{-4}$, $h_0=0.05$, $\gamma=0.9$, $\beta_1=0.9$, $\beta_2=0.99999$, and $\beta_3=1.001$.
\item \textbf{\texttt{lCBOpt-adapt}.} We use $\rho=5\cdot 10^{-3}$, $\delta=2\cdot 10^{-4}$, $h_0=0.1$, $\gamma=1.05$, $\beta_1=0.9$, $\beta_2=0.99999$, and $\beta_3=0.999$.

\end{itemize}

\subsection{Performance Metrics}
\label{sec:performance_metrics}

\begin{figure*}[t!]
\centering

\begin{subfigure}{0.48\linewidth}
    \centering
    \includegraphics[width=\linewidth]{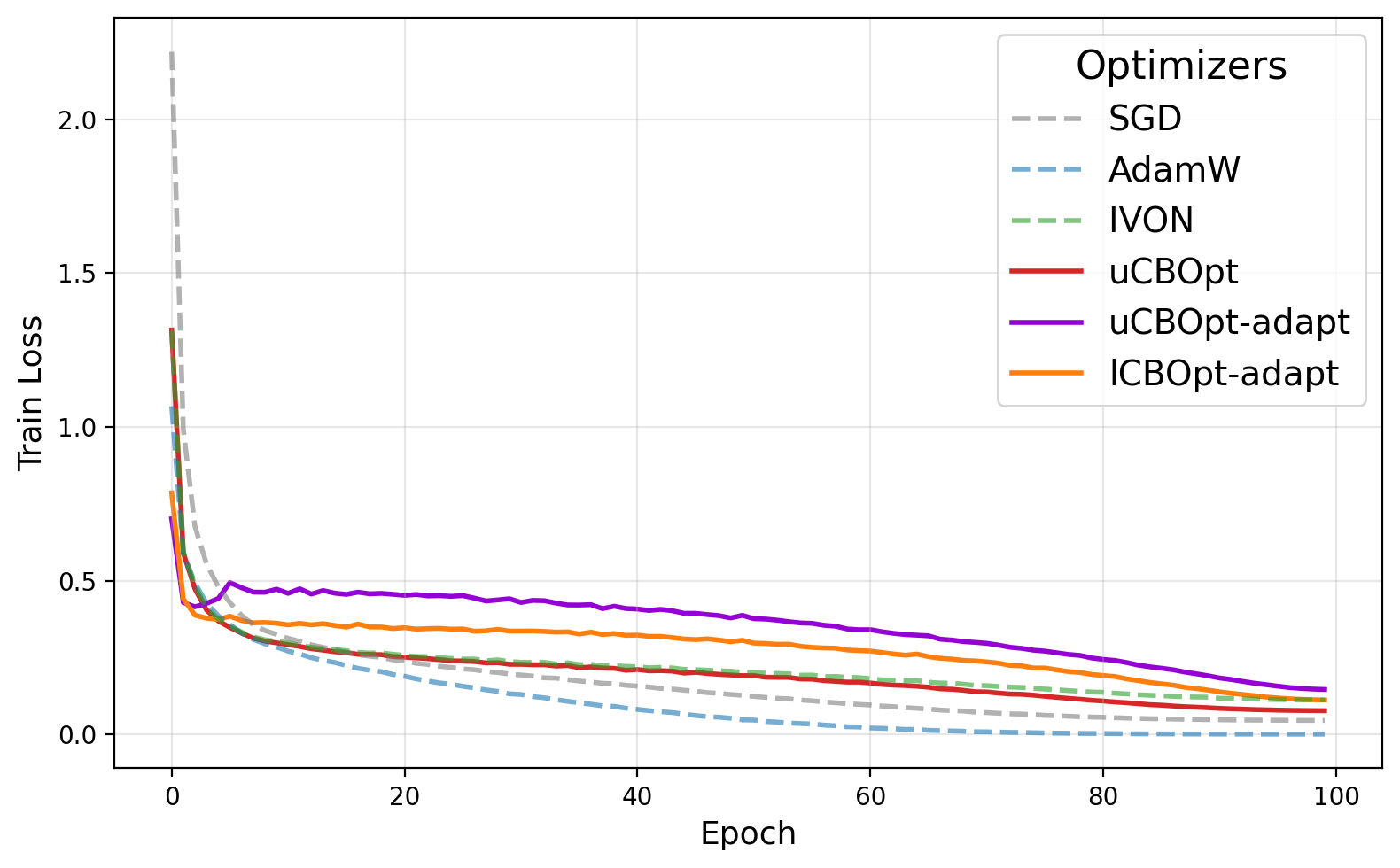}
    \caption{Training loss}
\end{subfigure}
\hfill
\begin{subfigure}{0.48\linewidth}
    \centering
    \includegraphics[width=\linewidth]{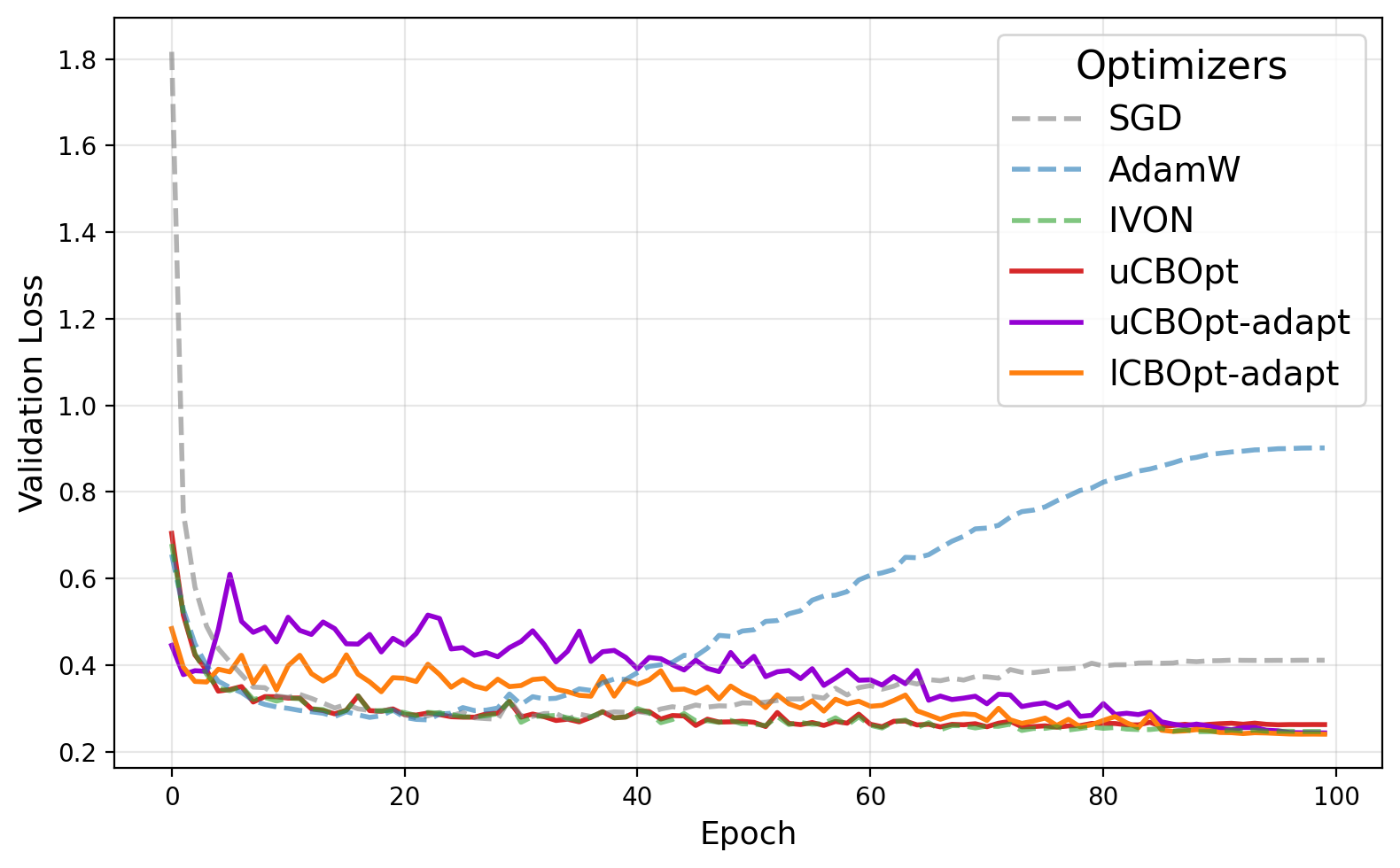}
    \caption{Validation loss}
\end{subfigure}

\vspace{0.6em}

\begin{subfigure}{0.48\linewidth}
    \centering
    \includegraphics[width=\linewidth]{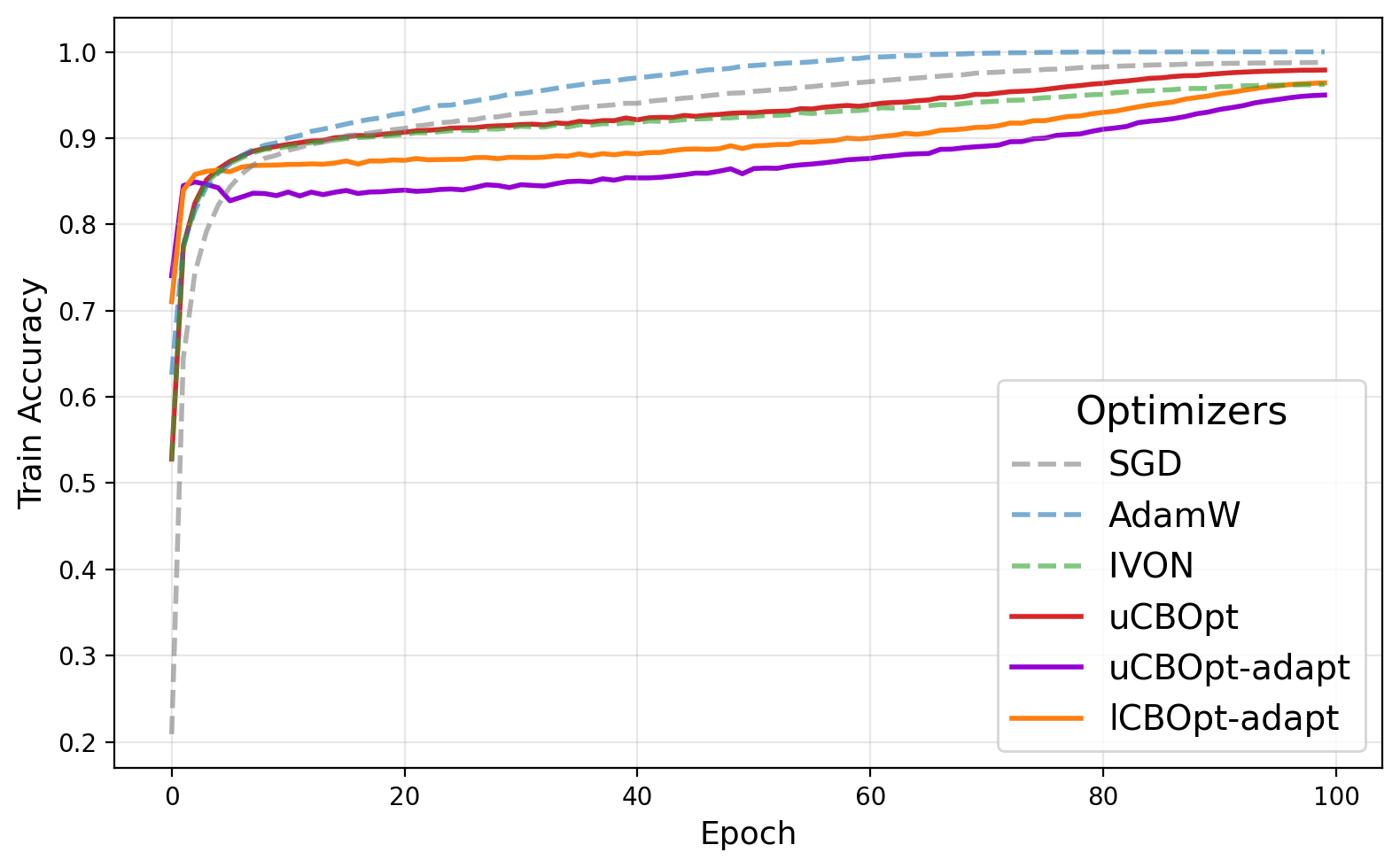}
    \caption{Training accuracy}
\end{subfigure}
\hfill
\begin{subfigure}{0.48\linewidth}
    \centering
    \includegraphics[width=\linewidth]{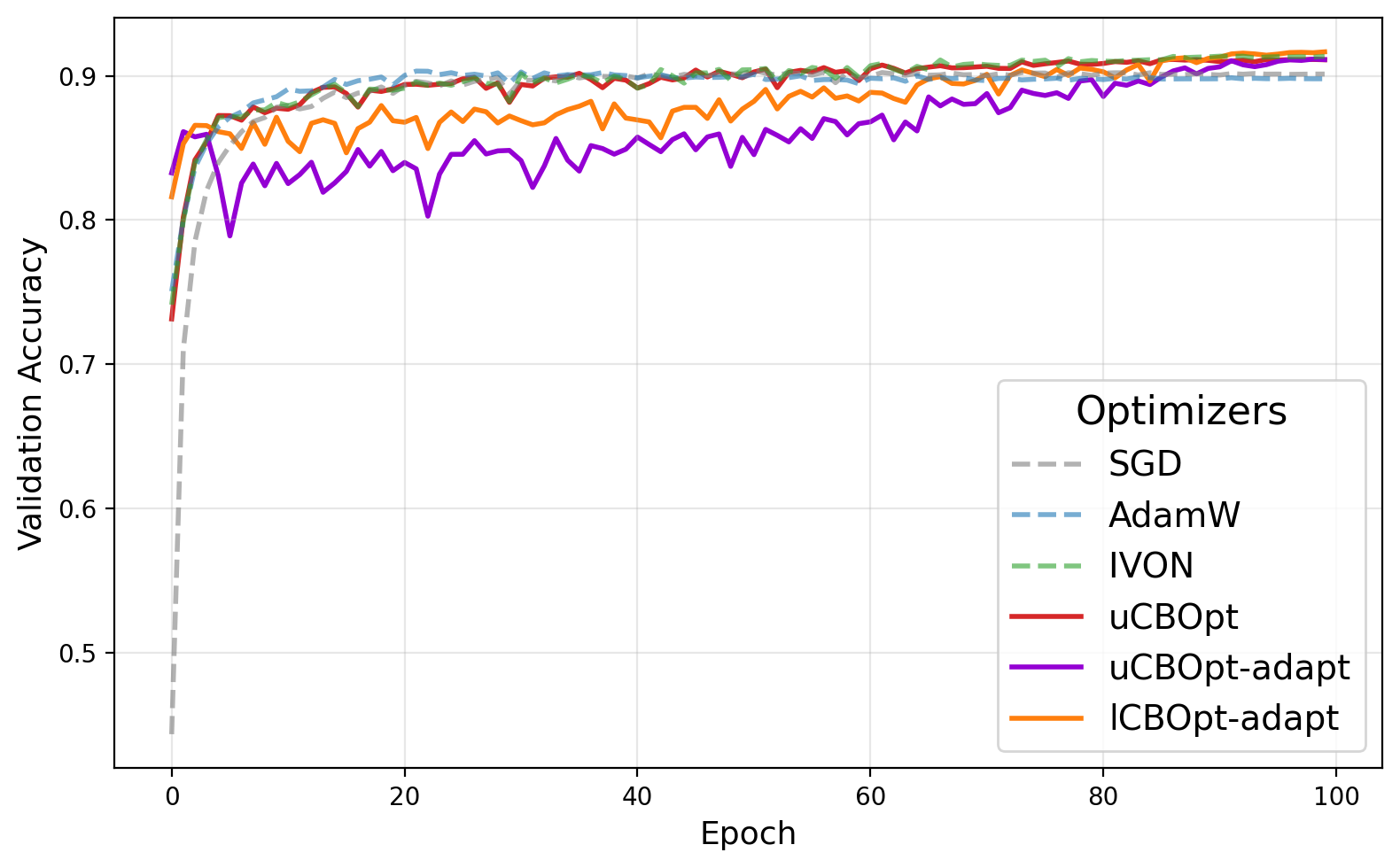}
    \caption{Validation accuracy}
\end{subfigure}

\caption{Training and validation loss and accuracy curves for all methods on the Fashion-MNIST/LeNet experiment.}
\label{fig:fmnist-lenet-training-curves}
\end{figure*}

\textbf{In-domain tasks.} For the in-domain tasks, we report metrics that capture classification performance, predictive uncertainty, and calibration quality. Classification performance is measured using top-1 (Acc.) and top-5 accuracy (Top-5 Acc.). We also report negative log likelihood (NLL) and Brier score to evaluate the quality of probabilistic predictions. Calibration is assessed using expected calibration error (ECE), while AUROC is computed using predictive confidence to measure how well the model separates correct from incorrect predictions. Although top-5 accuracy is less selective for 10-class datasets, we include it for consistency across datasets.

\textbf{Out-of-domain tasks.} For the out-of-domain tasks, we report metrics that evaluate how well the model separates in-domain from out-of-domain examples using uncertainty scores. We use the false positive rate at $95\%$ true positive rate (FPR@95) and detection error to assess threshold-based OOD detection performance. We also report the area under the receiver operating characteristic curve (AUROC), which measures how well the uncertainty scores distinguish between in-domain and out-of-domain examples across all detection thresholds. Lastly, we report the area under the precision--recall curve with in-domain examples treated as positive (AUPR-In) and with out-of-domain examples treated as positive (AUPR-Out). These metrics measure how well the model identifies in-domain and out-of-domain examples, respectively, under precision--recall trade-offs.

\section{Full Experimental Results}
\label{sec:full_results}

The complete results corresponding to \autoref{tab:main-results} are reported in \autoref{tab:in-domain-best-val} for the in-domain tasks and in \autoref{tab:ood-best-val} for the out-of-domain tasks. For completeness, we also report the final-epoch model performance in \autoref{tab:in-domain-final} and \autoref{tab:ood-final}. We note that our final-epoch results differ from those reported by \cite{shen2024variational}, mainly due to two differences in the experimental setup: (1) we use a training batch size of 128, whereas they use 50; and (2) we train on 90\% of the training set and reserve the remaining 10\% for validation, whereas they train on the full training set.

Across the experiments, \texttt{CBOpt} is most beneficial for improving in-domain predictive performance while maintaining competitive uncertainty estimates. The adaptive variants are particularly effective on Fashion-MNIST, where \texttt{uCBOpt-adapt} and \texttt{lCBOpt-adapt} obtain the best or near-best accuracy, NLL, Brier score, and AUROC. On CIFAR-10, all \texttt{CBOpt} variants remain competitive with SGD, Laplace, and IVON, although they do not outperform stronger stochastic baselines such as SWAG and MC dropout. On CIFAR-100, the advantages become clearer: fixed-\(\vartheta\) \texttt{uCBOpt} and \texttt{lCBOpt-adapt} improve over the considered baselines on several in-domain metrics, suggesting that the proposed updates scale well to more challenging classification tasks.

For OOD detection, \texttt{uCBOpt} gives the most consistent behaviour across datasets. It is competitive on Fashion-MNIST/EMNIST and CIFAR-10/SVHN, and achieves strong results on CIFAR-100/TinyImageNet. The adaptive variants are useful in some settings, particularly \texttt{uCBOpt-adapt} on CIFAR-100/TinyImageNet, but they are not uniformly better for OOD detection. Overall, the adaptive variants appear most helpful for in-domain performance, while fixed-\(\vartheta\) \texttt{uCBOpt} provides the most stable OOD performance.

We also present the training and validation loss and accuracy curves for LeNet trained on Fashion-MNIST in \autoref{fig:fmnist-lenet-training-curves}. The results show that \texttt{uCBOpt} exhibits training dynamics similar to IVON, achieving stable optimization and strong validation performance throughout training. In contrast, \texttt{uCBOpt-adapt} and \texttt{lCBOpt-adapt} converge more slowly, with higher training and validation losses in the early and middle stages. However, both adaptive variants continue to improve steadily and eventually reach validation accuracies comparable to the strongest baselines. AdamW achieves the lowest training loss and highest training accuracy, but its validation loss increases in later epochs, indicating stronger overfitting. By comparison, the \texttt{CBOpt} variants maintain more stable validation behaviour, suggesting better generalization.

We present the calibration plots of the \texttt{uCBOpt} family of optimizers in \autoref{fig:reliability-diagrams-cbopt}. Overall, the methods exhibit reasonably good calibration across all datasets, with the empirical accuracies generally following the diagonal reference line. The calibration behaviour is particularly strong on Fashion-MNIST and CIFAR-100, where the confidence and accuracy curves align closely over most confidence bins.

\begin{figure*}[h!]
\centering
\setlength{\tabcolsep}{2pt}

\begin{tabular}{ccc}
\textbf{Fashion-MNIST / LeNet} &
\textbf{CIFAR-10 / ResNet-20} &
\textbf{CIFAR-100 / DenseNet-121} \\

\begin{subfigure}{0.32\linewidth}
    \includegraphics[width=\linewidth]{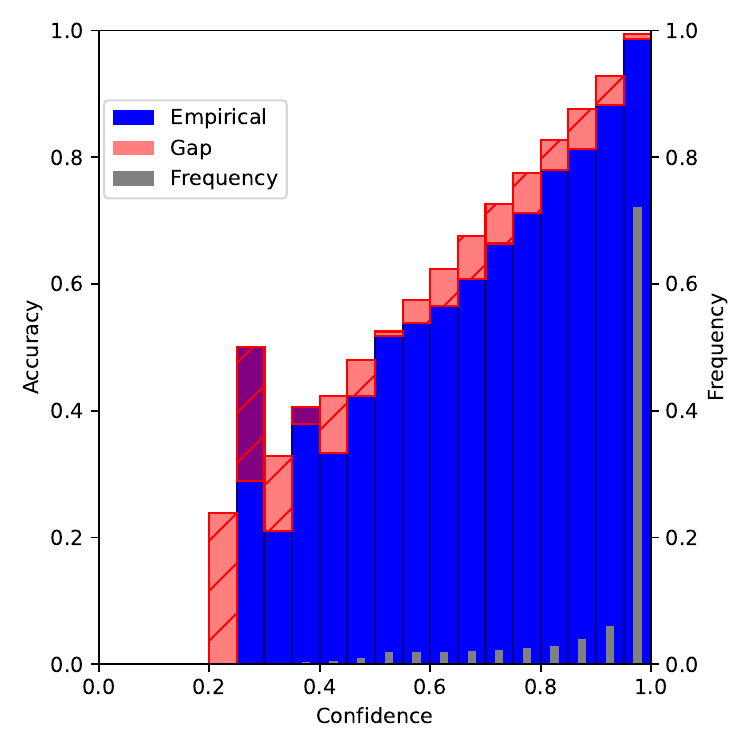}
    \caption{\texttt{uCBOpt} $(\vartheta=0)$}
\end{subfigure}
&
\begin{subfigure}{0.32\linewidth}
    \includegraphics[width=\linewidth]{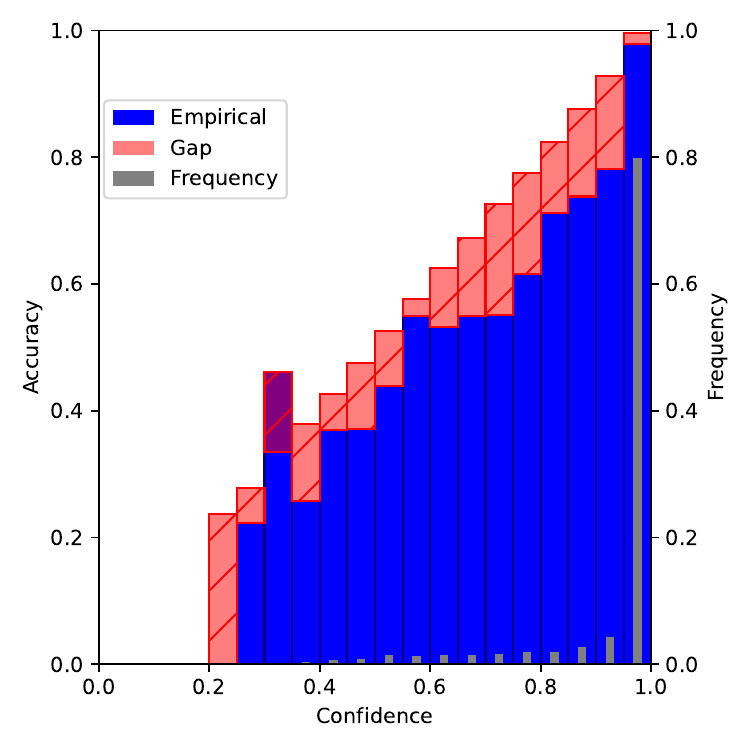}
    \caption{\texttt{uCBOpt} $(\vartheta=0)$}
\end{subfigure}
&
\begin{subfigure}{0.32\linewidth}
    \includegraphics[width=\linewidth]{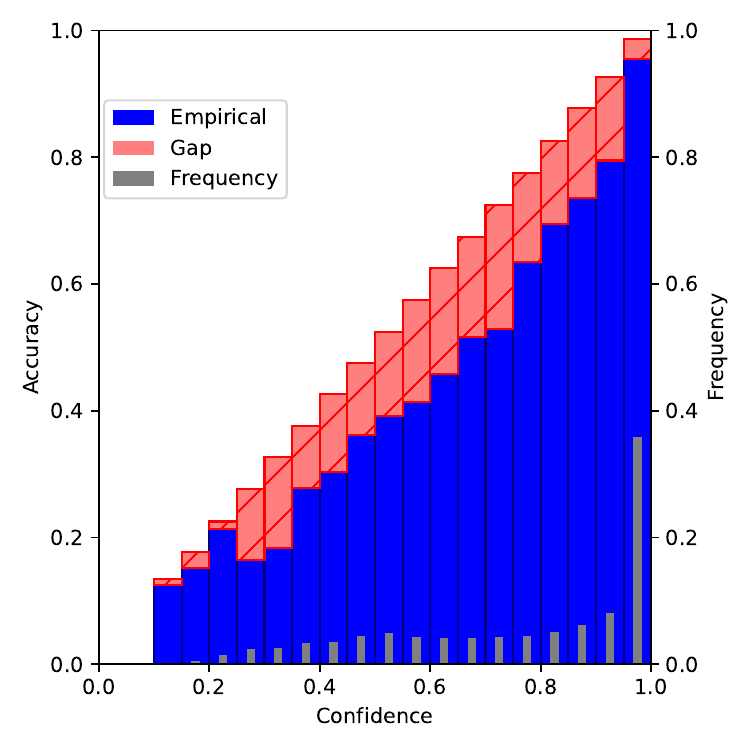}
    \caption{\texttt{uCBOpt} $(\vartheta=0)$}
\end{subfigure}
\\[-0.3em]

\begin{subfigure}{0.32\linewidth}
    \includegraphics[width=\linewidth]{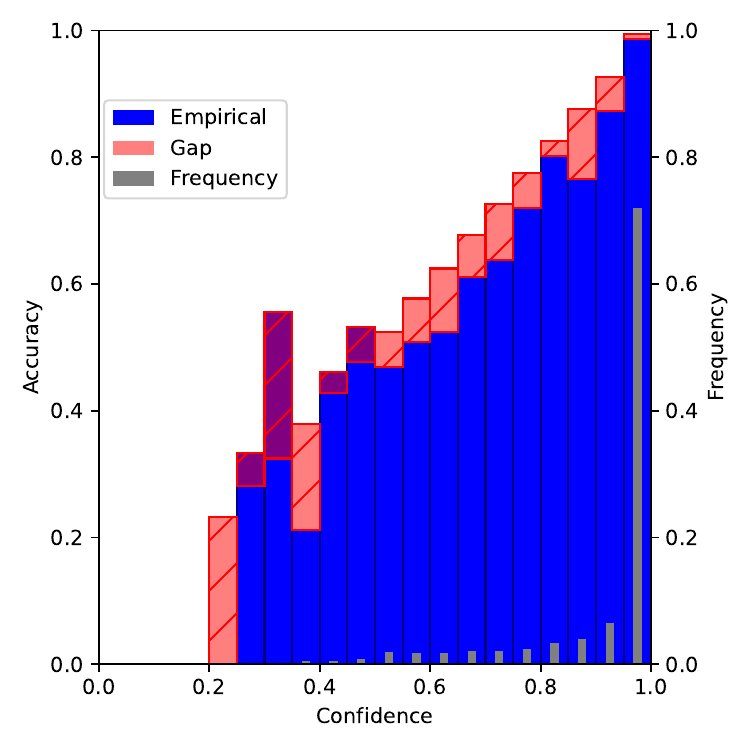}
    \caption{\texttt{uCBOpt} $(\vartheta=8\cdot10^{-6})$}
\end{subfigure}
&
\begin{subfigure}{0.32\linewidth}
    \includegraphics[width=\linewidth]{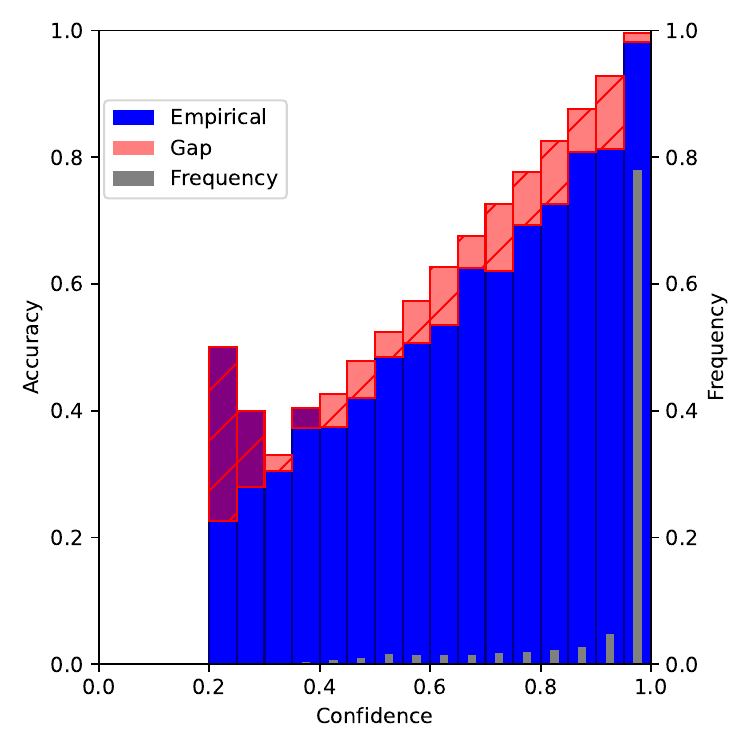}
    \caption{\texttt{uCBOpt} $(\vartheta=5\cdot10^{-5})$}
\end{subfigure}
&
\begin{subfigure}{0.32\linewidth}
    \includegraphics[width=\linewidth]{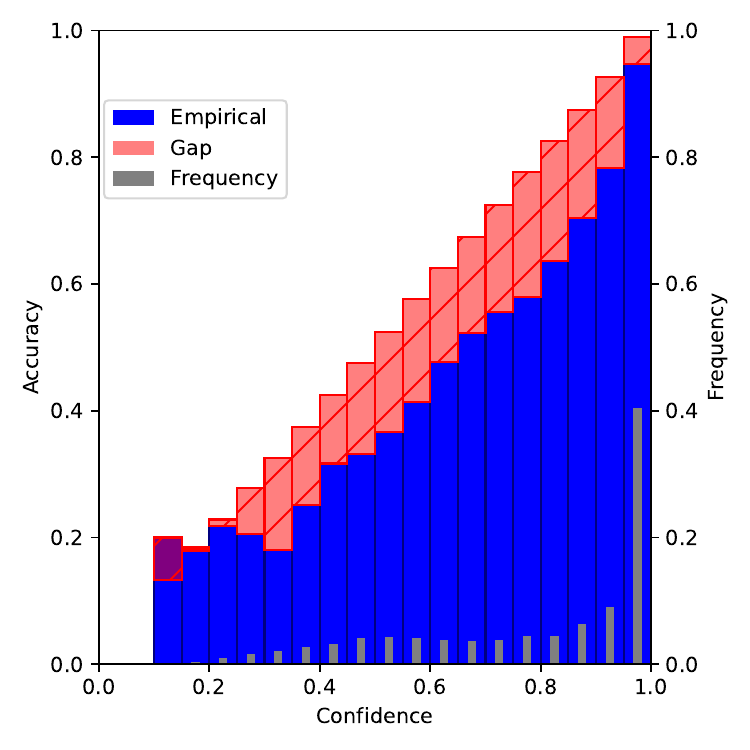}
    \caption{\texttt{uCBOpt} $(\vartheta=5\cdot10^{-5})$}
\end{subfigure}
\\[-0.3em]

\begin{subfigure}{0.32\linewidth}
    \includegraphics[width=\linewidth]{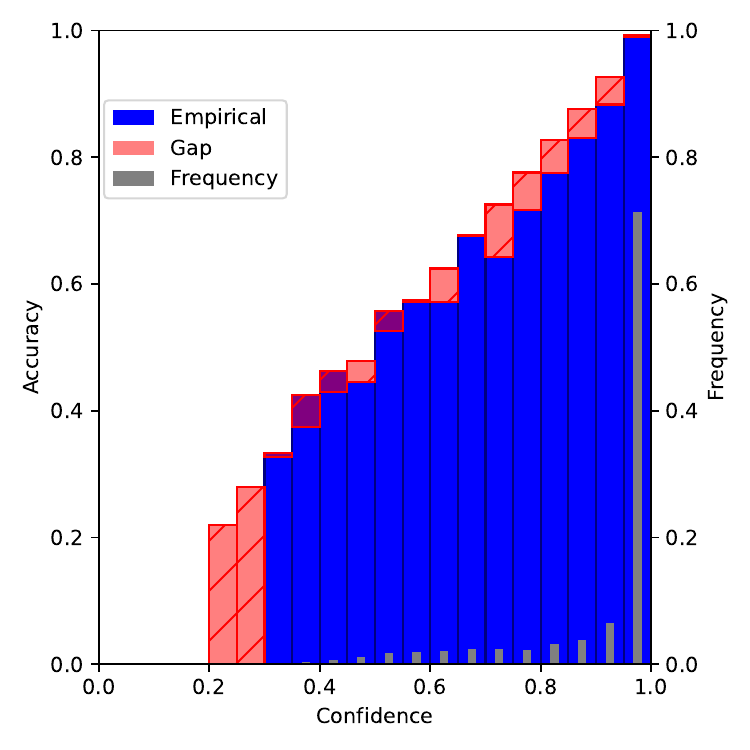}
    \caption{\texttt{uCBOpt-adapt}}
\end{subfigure}
&
\begin{subfigure}{0.32\linewidth}
    \includegraphics[width=\linewidth]{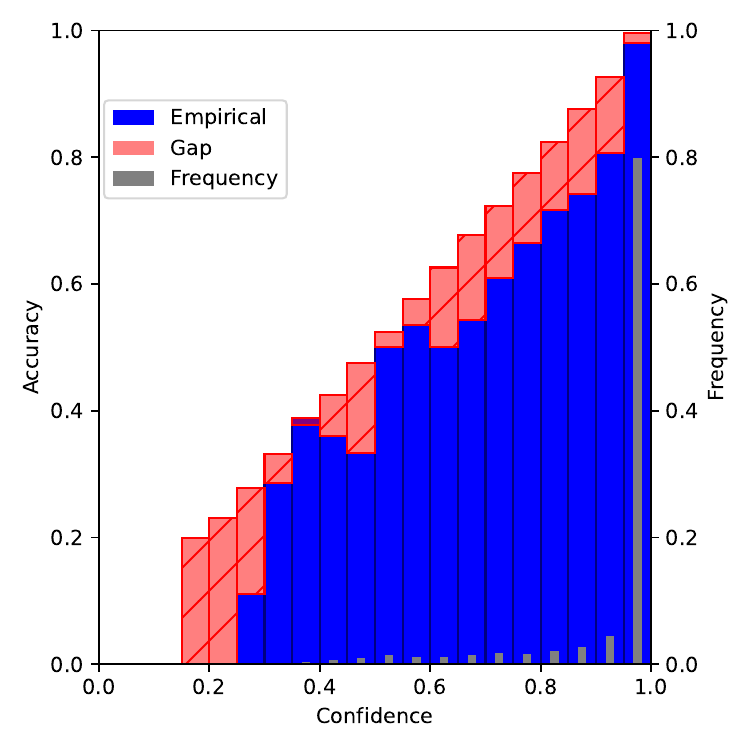}
    \caption{\texttt{uCBOpt-adapt}}
\end{subfigure}
&
\begin{subfigure}{0.32\linewidth}
    \includegraphics[width=\linewidth]{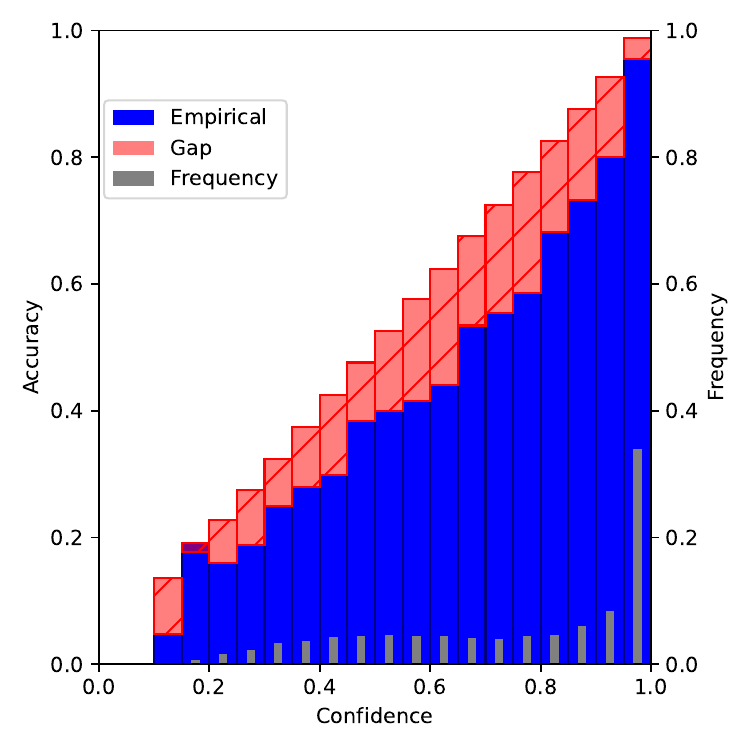}
    \caption{\texttt{uCBOpt-adapt}}
\end{subfigure}
\\[-0.3em]

\begin{subfigure}{0.32\linewidth}
    \includegraphics[width=\linewidth]{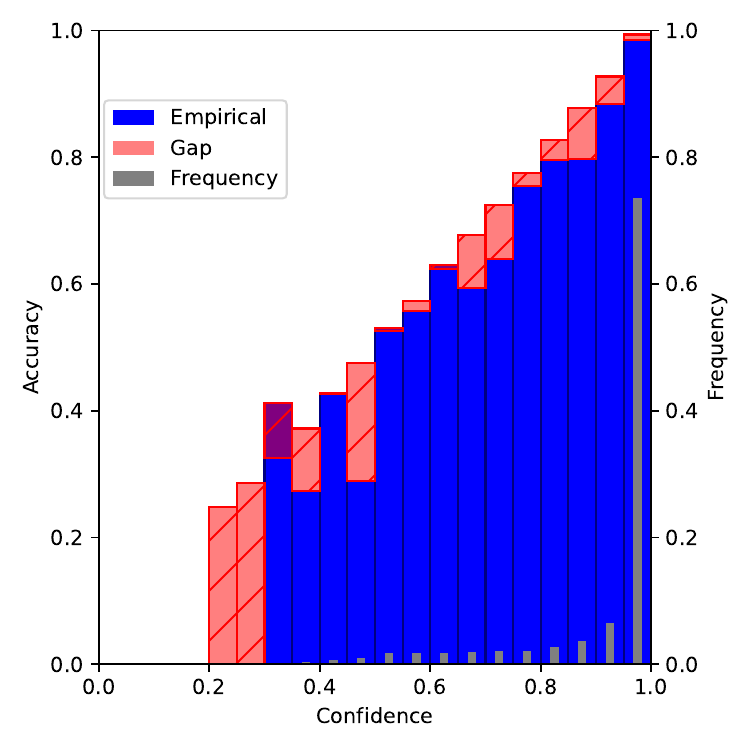}
    \caption{\texttt{lCBOpt-adapt}}
\end{subfigure}
&
\begin{subfigure}{0.32\linewidth}
    \includegraphics[width=\linewidth]{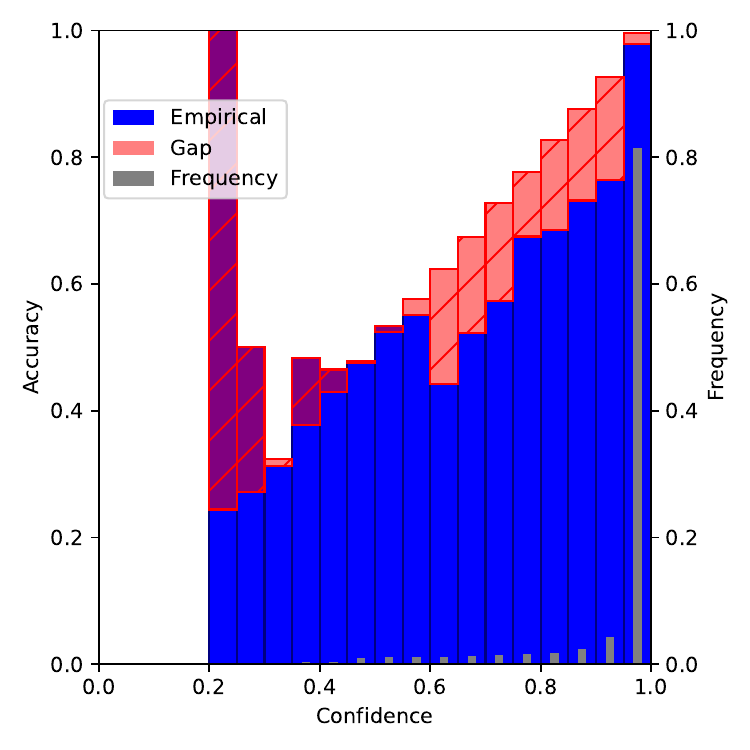}
    \caption{\texttt{lCBOpt-adapt}}
\end{subfigure}
&
\begin{subfigure}{0.32\linewidth}
    \includegraphics[width=\linewidth]{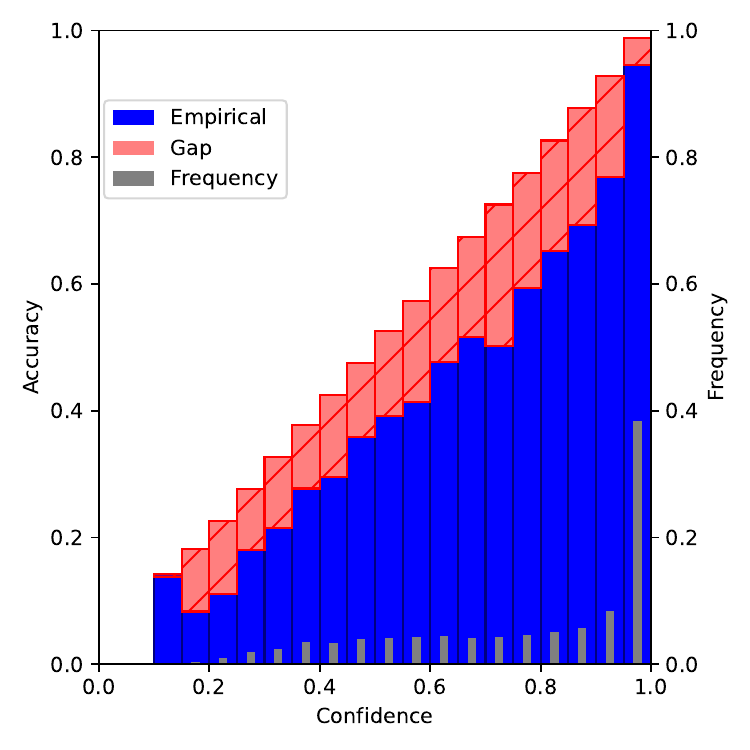}
    \caption{\texttt{lCBOpt-adapt}}
\end{subfigure}
\end{tabular}

\caption{Reliability diagrams for the \texttt{CBOpt} family.}
\label{fig:reliability-diagrams-cbopt}
\end{figure*}

\begin{table*}[h!]
\caption{Experimental results on in-domain tasks over three runs (best validation model).}
\label{tab:in-domain-best-val}
\centering
\setlength{\tabcolsep}{3pt}
        \resizebox{\linewidth}{!}{\begin{tabular}{llcccccc}
    \toprule
        Dataset & Optimizer & Acc. $\uparrow$ & Top-5 Acc. $\uparrow$ & NLL $\downarrow$ & ECE $\downarrow$ & Brier $\downarrow$ & AUROC $\uparrow$\\
        \midrule
        \multirow{11}{*}{Fashion-MNIST}
        & AdamW & $0.900_{\pm 0.005}$ & $0.998_{\pm 0.000}$ & $0.294_{\pm 0.007}$ & $0.026_{\pm 0.004}$ & $0.148_{\pm 0.004}$ & $0.901_{\pm 0.005}$ \\
        & SGD & $0.896_{\pm 0.005}$ & $0.999_{\pm 0.000}$ & $0.293_{\pm 0.005}$ & $0.022_{\pm 0.004}$ & $0.151_{\pm 0.004}$ & $0.904_{\pm 0.002}$ \\
        & IVON@mean & $0.907_{\pm 0.001}$ & $0.999_{\pm 0.000}$ & $0.267_{\pm 0.005}$ & $0.018_{\pm 0.008}$ & $0.136_{\pm 0.001}$ & $0.906_{\pm 0.002}$ \\
        & IVON & $0.911_{\pm 0.003}$ & $0.999_{\pm 0.000}$ & $0.258_{\pm 0.001}$ & $0.013_{\pm 0.004}$ & $0.132_{\pm 0.002}$ & $0.904_{\pm 0.003}$ \\
        & MC-D & $0.897_{\pm 0.003}$ & $0.999_{\pm 0.000}$ & $0.281_{\pm 0.007}$ & $0.013_{\pm 0.003}$ & $0.146_{\pm 0.004}$ & $0.906_{\pm 0.003}$ \\
        & Laplace & $0.895_{\pm 0.004}$ & $0.999_{\pm 0.001}$ & $0.295_{\pm 0.009}$ & $0.024_{\pm 0.003}$ & $0.152_{\pm 0.005}$ & $0.902_{\pm 0.003}$ \\
        & SWAG & $0.900_{\pm 0.003}$ & $0.999_{\pm 0.000}$ & $0.342_{\pm 0.009}$ & $0.043_{\pm 0.003}$ & $0.153_{\pm 0.005}$ & $0.899_{\pm 0.004}$ \\
        \rowcolor{gray!10} \cellcolor{white}& \texttt{uCBOpt} $(\vartheta\to0^{+})$ & $0.905_{\pm 0.002}$ & $0.998_{\pm 0.000}$ & $0.269_{\pm 0.004}$ & $0.020_{\pm 0.006}$ & $0.138_{\pm 0.003}$ & $0.908_{\pm 0.003}$ \\
        \rowcolor{gray!10} \cellcolor{white}& \texttt{uCBOpt} $(\vartheta=8 \cdot 10^{-6})$ & $0.907_{\pm 0.002}$ & $0.999_{\pm 0.000}$ & $0.270_{\pm 0.002}$ & $0.023_{\pm 0.000}$ & $0.137_{\pm 0.003}$ & $0.905_{\pm 0.003}$ \\
         \rowcolor{gray!10} \cellcolor{white}& \texttt{uCBOpt-adapt} & $0.910_{\pm 0.001}$ & $0.999_{\pm 0.000}$ & $0.256_{\pm 0.003}$ & $0.016_{\pm 0.001}$ & $0.132_{\pm 0.001}$ & $0.910_{\pm 0.002}$ \\
        \rowcolor{gray!10} \cellcolor{white}& \texttt{lCBOpt-adapt} & $0.913_{\pm 0.002}$ & $0.998_{\pm 0.000}$ & $0.252_{\pm 0.004}$ & $0.020_{\pm 0.000}$ & $0.129_{\pm 0.001}$ & $0.910_{\pm 0.003}$ \\
        \cmidrule{1-8}
        \multirow{11}{*}{CIFAR-10}
        & AdamW & $0.872_{\pm 0.001}$ & $0.995_{\pm 0.000}$ & $0.406_{\pm 0.002}$ & $0.042_{\pm 0.003}$ & $0.189_{\pm 0.002}$ & $0.893_{\pm 0.001}$ \\
        & SGD & $0.909_{\pm 0.003}$ & $0.997_{\pm 0.001}$ & $0.308_{\pm 0.011}$ & $0.039_{\pm 0.008}$ & $0.139_{\pm 0.002}$ & $0.914_{\pm 0.004}$ \\
        & IVON@mean & $0.892_{\pm 0.003}$ & $0.996_{\pm 0.001}$ & $0.358_{\pm 0.011}$ & $0.045_{\pm 0.008}$ & $0.162_{\pm 0.003}$ & $0.908_{\pm 0.002}$ \\
        & IVON & $0.899_{\pm 0.004}$ & $0.996_{\pm 0.001}$ & $0.317_{\pm 0.004}$ & $0.028_{\pm 0.005}$ & $0.149_{\pm 0.005}$ & $0.910_{\pm 0.006}$ \\
        & MC-D & $0.923_{\pm 0.002}$ & $0.998_{\pm 0.001}$ & $0.230_{\pm 0.006}$ & $0.010_{\pm 0.002}$ & $0.113_{\pm 0.002}$ & $0.923_{\pm 0.002}$ \\
        & Laplace & $0.909_{\pm 0.001}$ & $0.997_{\pm 0.000}$ & $0.303_{\pm 0.008}$ & $0.035_{\pm 0.006}$ & $0.139_{\pm 0.001}$ & $0.911_{\pm 0.002}$ \\
        & SWAG & $0.918_{\pm 0.001}$ & $0.997_{\pm 0.001}$ & $0.251_{\pm 0.004}$ & $0.009_{\pm 0.000}$ & $0.121_{\pm 0.002}$ & $0.920_{\pm 0.002}$ \\
        \rowcolor{gray!10} \cellcolor{white}& \texttt{uCBOpt} $(\vartheta\to0^{+})$ & $0.910_{\pm 0.004}$ & $0.996_{\pm 0.000}$ & $0.296_{\pm 0.009}$ & $0.038_{\pm 0.001}$ & $0.135_{\pm 0.005}$ & $0.918_{\pm 0.003}$ \\
        \rowcolor{gray!10} \cellcolor{white}& \texttt{uCBOpt} $(\vartheta=5\cdot10^{-5})$ & $0.907_{\pm 0.003}$ & $0.997_{\pm 0.000}$ & $0.302_{\pm 0.004}$ & $0.033_{\pm 0.004}$ & $0.140_{\pm 0.004}$ & $0.914_{\pm 0.000}$ \\
         \rowcolor{gray!10} \cellcolor{white} & \texttt{uCBOpt-adapt} & $0.909_{\pm 0.005}$ & $0.997_{\pm 0.001}$ & $0.300_{\pm 0.016}$ & $0.038_{\pm 0.004}$ & $0.138_{\pm 0.008}$ & $0.917_{\pm 0.001}$ \\
        \rowcolor{gray!10} \cellcolor{white}& \texttt{lCBOpt-adapt} & $0.909_{\pm 0.005}$ & $0.997_{\pm 0.000}$ & $0.299_{\pm 0.002}$ & $0.037_{\pm 0.002}$ & $0.138_{\pm 0.004}$ & $0.915_{\pm 0.003}$ \\
        \cmidrule{1-8}
        \multirow{8}{*}{CIFAR-100}
        & SGD & $0.649_{\pm 0.004}$ & $0.891_{\pm 0.005}$ & $1.334_{\pm 0.006}$ & $0.096_{\pm 0.006}$ & $0.483_{\pm 0.002}$ & $0.838_{\pm 0.002}$ \\
        & IVON@mean & $0.599_{\pm 0.001}$ & $0.864_{\pm 0.002}$ & $1.521_{\pm 0.008}$ & $0.092_{\pm 0.006}$ & $0.538_{\pm 0.002}$ & $0.828_{\pm 0.001}$ \\
        & IVON & $0.607_{\pm 0.002}$ & $0.870_{\pm 0.003}$ & $1.458_{\pm 0.012}$ & $0.071_{\pm 0.004}$ & $0.523_{\pm 0.002}$ & $0.830_{\pm 0.003}$ \\
        & MC-D & $0.649_{\pm 0.016}$ & $0.894_{\pm 0.003}$ & $1.298_{\pm 0.016}$ & $0.074_{\pm 0.020}$ & $0.476_{\pm 0.013}$ & $0.838_{\pm 0.006}$ \\
        & \graycell{}\texttt{uCBOpt} $(\vartheta\to0^{+})$ & \graycell{}$0.667_{\pm 0.007}$ & \graycell{}$0.905_{\pm 0.001}$ & \graycell{}$1.247_{\pm 0.001}$ & \graycell{}$0.094_{\pm 0.007}$ & \graycell{}$0.457_{\pm 0.005}$ & \graycell{}$0.845_{\pm 0.001}$\\
        \rowcolor{gray!10} \cellcolor{white}& \texttt{uCBOpt} $(\vartheta=5\cdot10^{-5})$ & $0.671_{\pm 0.014}$ & $0.906_{\pm 0.004}$ & $1.231_{\pm 0.005}$ & $0.095_{\pm 0.014}$ & $0.453_{\pm 0.009}$ & $0.847_{\pm 0.002}$ \\
         \rowcolor{gray!10} \cellcolor{white}& \texttt{uCBOpt-adapt} & $0.655_{\pm 0.011}$ & $0.897_{\pm 0.006}$ & $1.298_{\pm 0.011}$ & $0.096_{\pm 0.021}$ & $0.475_{\pm 0.008}$ & $0.841_{\pm 0.003}$ \\
        \rowcolor{gray!10} \cellcolor{white}& \texttt{lCBOpt-adapt} & $0.674_{\pm 0.007}$ & $0.907_{\pm 0.003}$ & $1.268_{\pm 0.018}$ & $0.110_{\pm 0.002}$ & $0.457_{\pm 0.008}$ & $0.845_{\pm 0.005}$ \\
    \bottomrule
    \end{tabular}}
\vskip -0.1in
\end{table*}

\begin{table*}[h!]
\caption{Experimental results on out-of-domain tasks over three runs (best validation model).}
\label{tab:ood-best-val}
\centering
\setlength{\tabcolsep}{3pt}
    \resizebox{0.92\linewidth}{!}{\begin{tabular}{llccccc}
    \toprule
        Dataset & Optimizer & FPR@95\% $\downarrow$ & Det. Err. $\downarrow$ & AUROC $\uparrow$ & AUPR-In $\uparrow$ & AUPR-Out $\uparrow$\\
        \midrule
        \multirow{11}{*}{\makecell[l]{\textbf{In:} Fashion-MNIST \\ \textbf{Out:} EMNIST}}
        & AdamW & $0.321_{\pm 0.008}$ & $0.312_{\pm 0.009}$ & $0.745_{\pm 0.009}$ & $0.627_{\pm 0.022}$ & $0.834_{\pm 0.003}$ \\
        & SGD & $0.297_{\pm 0.020}$ & $0.293_{\pm 0.018}$ & $0.767_{\pm 0.022}$ & $0.676_{\pm 0.043}$ & $0.845_{\pm 0.012}$ \\
        & IVON@mean & $0.266_{\pm 0.024}$ & $0.279_{\pm 0.026}$ & $0.787_{\pm 0.029}$ & $0.694_{\pm 0.047}$ & $0.861_{\pm 0.019}$ \\
        & IVON & $0.249_{\pm 0.024}$ & $0.265_{\pm 0.028}$ & $0.804_{\pm 0.030}$ & $0.719_{\pm 0.048}$ & $0.873_{\pm 0.020}$ \\
        & MC-D & $0.244_{\pm 0.011}$ & $0.250_{\pm 0.007}$ & $0.821_{\pm 0.008}$ & $0.752_{\pm 0.017}$ & $0.885_{\pm 0.007}$ \\
        & Laplace & $0.296_{\pm 0.047}$ & $0.290_{\pm 0.030}$ & $0.770_{\pm 0.039}$ & $0.675_{\pm 0.046}$ & $0.848_{\pm 0.031}$ \\
        & SWAG & $0.378_{\pm 0.005}$ & $0.324_{\pm 0.001}$ & $0.731_{\pm 0.005}$ & $0.601_{\pm 0.027}$ & $0.820_{\pm 0.002}$ \\
        \rowcolor{gray!10} \cellcolor{white}& \texttt{uCBOpt} $(\vartheta\to0^{+})$ & $0.262_{\pm 0.012}$ & $0.270_{\pm 0.013}$ & $0.796_{\pm 0.014}$ & $0.708_{\pm 0.023}$ & $0.867_{\pm 0.010}$ \\
        \rowcolor{gray!10} \cellcolor{white}& \texttt{uCBOpt} $(\vartheta=8 \cdot 10^{-6})$ & $0.245_{\pm 0.010}$ & $0.259_{\pm 0.007}$ & $0.810_{\pm 0.012}$ & $0.720_{\pm 0.037}$ & $0.879_{\pm 0.010}$ \\
        \rowcolor{gray!10} \cellcolor{white}& \texttt{uCBOpt-adapt} & $0.272_{\pm 0.011}$ & $0.290_{\pm 0.015}$ & $0.773_{\pm 0.015}$ & $0.655_{\pm 0.016}$ & $0.857_{\pm 0.009}$ \\
        \rowcolor{gray!10} \cellcolor{white}& \texttt{lCBOpt-adapt} & $0.269_{\pm 0.015}$ & $0.278_{\pm 0.016}$ & $0.784_{\pm 0.019}$ & $0.674_{\pm 0.049}$ & $0.859_{\pm 0.011}$ \\
        \cmidrule{1-7}
        \multirow{11}{*}{\makecell[l]{\textbf{In:} CIFAR-10 \\ \textbf{Out:} SVHN}}
        & AdamW & $0.271_{\pm 0.033}$ & $0.227_{\pm 0.027}$ & $0.820_{\pm 0.031}$ & $0.769_{\pm 0.038}$ & $0.881_{\pm 0.021}$ \\
        & SGD & $0.207_{\pm 0.019}$ & $0.195_{\pm 0.012}$ & $0.864_{\pm 0.013}$ & $0.808_{\pm 0.016}$ & $0.918_{\pm 0.011}$ \\
        & IVON@mean & $0.255_{\pm 0.013}$ & $0.221_{\pm 0.010}$ & $0.830_{\pm 0.011}$ & $0.779_{\pm 0.013}$ & $0.890_{\pm 0.010}$ \\
        & IVON & $0.252_{\pm 0.020}$ & $0.214_{\pm 0.011}$ & $0.836_{\pm 0.014}$ & $0.791_{\pm 0.016}$ & $0.894_{\pm 0.012}$ \\
        & MC-D & $0.230_{\pm 0.030}$ & $0.206_{\pm 0.021}$ & $0.847_{\pm 0.024}$ & $0.796_{\pm 0.027}$ & $0.902_{\pm 0.017}$ \\
        & Laplace & $0.230_{\pm 0.021}$ & $0.204_{\pm 0.015}$ & $0.852_{\pm 0.019}$ & $0.797_{\pm 0.021}$ & $0.909_{\pm 0.013}$ \\
        & SWAG & $0.197_{\pm 0.012}$ & $0.175_{\pm 0.004}$ & $0.879_{\pm 0.007}$ & $0.841_{\pm 0.006}$ & $0.924_{\pm 0.007}$ \\
        \rowcolor{gray!10} \cellcolor{white}& \texttt{uCBOpt} $(\vartheta\to0^{+})$& $0.208_{\pm 0.016}$ & $0.186_{\pm 0.008}$ & $0.869_{\pm 0.012}$ & $0.824_{\pm 0.013}$ & $0.918_{\pm 0.010}$ \\
        \rowcolor{gray!10} \cellcolor{white}& \texttt{uCBOpt} $(\vartheta=5\cdot10^{-5})$ & $0.221_{\pm 0.009}$ & $0.200_{\pm 0.001}$ & $0.854_{\pm 0.004}$ & $0.804_{\pm 0.001}$ & $0.908_{\pm 0.006}$ \\
         \rowcolor{gray!10} \cellcolor{white} & \texttt{uCBOpt-adapt} & $0.210_{\pm 0.016}$ & $0.187_{\pm 0.012}$ & $0.866_{\pm 0.010}$ & $0.821_{\pm 0.020}$ & $0.915_{\pm 0.007}$ \\
        \rowcolor{gray!10} \cellcolor{white}& \texttt{lCBOpt-adapt} & $0.213_{\pm 0.012}$ & $0.195_{\pm 0.006}$ & $0.861_{\pm 0.006}$ & $0.809_{\pm 0.007}$ & $0.913_{\pm 0.003}$ \\
        \cmidrule{1-7}
        \multirow{8}{*}{\makecell[l]{\textbf{In:} CIFAR-100 \\ \textbf{Out:} TinyImageNet}}
        & SGD & $0.357_{\pm 0.005}$ & $0.337_{\pm 0.006}$ & $0.720_{\pm 0.008}$ & $0.744_{\pm 0.005}$ & $0.684_{\pm 0.006}$ \\
        & IVON@mean & $0.368_{\pm 0.006}$ & $0.354_{\pm 0.006}$ & $0.702_{\pm 0.008}$ & $0.726_{\pm 0.008}$ & $0.663_{\pm 0.007}$ \\
        & IVON & $0.368_{\pm 0.006}$ & $0.354_{\pm 0.006}$ & $0.702_{\pm 0.008}$ & $0.726_{\pm 0.008}$ & $0.663_{\pm 0.007}$ \\
        & MC-D & $0.353_{\pm 0.008}$ & $0.341_{\pm 0.005}$ & $0.718_{\pm 0.008}$ & $0.740_{\pm 0.007}$ & $0.681_{\pm 0.009}$ \\
        & \graycell{}\texttt{uCBOpt} $(\vartheta\to0^{+})$ & \graycell{}$0.350_{\pm 0.009}$ & \graycell{}$0.330_{\pm 0.006}$ & \graycell{}$0.727_{\pm 0.009}$ & \graycell{}$0.751_{\pm 0.007}$ & \graycell{}$0.686_{\pm 0.011}$ \\
        \rowcolor{gray!10} \cellcolor{white}& \texttt{uCBOpt} $(\vartheta=5\cdot10^{-5})$ & $0.344_{\pm 0.004}$ & $0.331_{\pm 0.004}$ & $0.727_{\pm 0.003}$ & $0.749_{\pm 0.003}$ & $0.693_{\pm 0.003}$ \\
         \rowcolor{gray!10} \cellcolor{white}& \texttt{uCBOpt-adapt} & $0.343_{\pm 0.004}$ & $0.332_{\pm 0.003}$ & $0.729_{\pm 0.002}$ & $0.749_{\pm 0.005}$ & $0.696_{\pm 0.001}$ \\
        \rowcolor{gray!10} \cellcolor{white}& \texttt{lCBOpt-adapt} & $0.346_{\pm 0.004}$ & $0.330_{\pm 0.004}$ & $0.729_{\pm 0.005}$ & $0.754_{\pm 0.005}$ & $0.692_{\pm 0.006}$ \\
    \bottomrule
    \end{tabular}}
\vskip -0.1in
\end{table*}

\begin{table*}[h!]
\caption{Experimental results on in-domain tasks over three runs (final epoch model).}
\label{tab:in-domain-final}
\centering
\setlength{\tabcolsep}{3pt}
    \resizebox{\linewidth}{!}{\begin{tabular}{llcccccc}
    \toprule
        Dataset & Optimizer & Acc. $\uparrow$ & Top-5 Acc. $\uparrow$ & NLL $\downarrow$ & ECE $\downarrow$ & Brier $\downarrow$ & AUROC $\uparrow$\\
        \midrule
        \multirow{11}{*}{Fashion-MNIST}
        & AdamW & $0.897_{\pm 0.000}$ & $0.997_{\pm 0.001}$ & $0.981_{\pm 0.091}$ & $0.086_{\pm 0.002}$ & $0.184_{\pm 0.001}$ & $0.866_{\pm 0.010}$ \\
        & SGD & $0.899_{\pm 0.003}$ & $0.998_{\pm 0.001}$ & $0.441_{\pm 0.025}$ & $0.061_{\pm 0.004}$ & $0.163_{\pm 0.005}$ & $0.893_{\pm 0.003}$ \\
        & IVON@mean & $0.908_{\pm 0.003}$ & $0.999_{\pm 0.000}$ & $0.271_{\pm 0.007}$ & $0.026_{\pm 0.003}$ & $0.137_{\pm 0.004}$ & $0.905_{\pm 0.002}$ \\
        & IVON & $0.913_{\pm 0.001}$ & $0.999_{\pm 0.000}$ & $0.255_{\pm 0.005}$ & $0.130_{\pm 0.002}$ & $0.018_{\pm 0.002}$ & $0.906_{\pm 0.001}$ \\
        & MC-D & $0.904_{\pm 0.003}$ & $0.999_{\pm 0.000}$ & $0.298_{\pm 0.007}$ & $0.034_{\pm 0.001}$ & $0.143_{\pm 0.003}$ & $0.903_{\pm 0.001}$ \\
        & Laplace & $0.897_{\pm 0.005}$ & $0.999_{\pm 0.000}$ & $0.447_{\pm 0.010}$ & $0.064_{\pm 0.003}$ & $0.165_{\pm 0.007}$ & $0.895_{\pm 0.003}$ \\
        & SWAG & $0.900_{\pm 0.003}$ & $0.999_{\pm 0.000}$ & $0.342_{\pm 0.009}$ & $0.043_{\pm 0.003}$ & $0.153_{\pm 0.005}$ & $0.899_{\pm 0.004}$ \\
        \rowcolor{gray!10} \cellcolor{white}& \texttt{uCBOpt} $(\vartheta\to0^{+})$ & $0.909_{\pm 0.003}$ & $0.998_{\pm 0.000}$ & $0.279_{\pm 0.004}$ & $0.031_{\pm 0.004}$ & $0.136_{\pm 0.003}$ & $0.905_{\pm 0.001}$ \\
        \rowcolor{gray!10} \cellcolor{white}& \texttt{uCBOpt} $(\vartheta=8\cdot 10^{-6})$ & $0.910_{\pm 0.002}$ & $0.998_{\pm 0.001}$ & $0.278_{\pm 0.007}$ & $0.032_{\pm 0.002}$ & $0.135_{\pm 0.002}$ & $0.905_{\pm 0.003}$ \\
         \rowcolor{gray!10} \cellcolor{white}& \texttt{uCBOpt-adapt} & $0.910_{\pm 0.001}$ & $0.999_{\pm 0.000}$ & $0.256_{\pm 0.003}$ & $0.016_{\pm 0.001}$ & $0.132_{\pm 0.002}$ & $0.910_{\pm 0.002}$ \\
        \rowcolor{gray!10} \cellcolor{white}& \texttt{lCBOpt-adapt} & $0.913_{\pm 0.002}$ & $0.998_{\pm 0.000}$ & $0.252_{\pm 0.003}$ & $0.021_{\pm 0.001}$ & $0.128_{\pm 0.001}$ & $0.911_{\pm 0.005}$ \\
        \cmidrule{1-8}
        \multirow{11}{*}{CIFAR-10}
        & AdamW & $0.888_{\pm 0.001}$ & $0.995_{\pm 0.000}$ & $0.638_{\pm 0.013}$ & $0.083_{\pm 0.002}$ & $0.190_{\pm 0.001}$ & $0.903_{\pm 0.003}$ \\
        & SGD & $0.917_{\pm 0.003}$ & $0.997_{\pm 0.000}$ & $0.311_{\pm 0.014}$ & $0.044_{\pm 0.003}$ & $0.130_{\pm 0.005}$ & $0.916_{\pm 0.002}$    \\
        & IVON@mean & $0.901_{\pm 0.002}$ & $0.995_{\pm 0.000}$ & $0.390_{\pm 0.003}$ & $0.159_{\pm 0.001}$ & $0.055_{\pm 0.002}$ & $0.905_{\pm 0.004}$ \\
        & IVON & $0.910_{\pm 0.001}$ & $0.997_{\pm 0.000}$ & $0.299_{\pm 0.004}$ & $0.029_{\pm 0.002}$ & $0.135_{\pm 0.002}$ & $0.915_{\pm 0.002}$ \\
        & MC-D & $0.924_{\pm 0.001}$ & $0.998_{\pm 0.001}$ & $0.229_{\pm 0.005}$ & $0.009_{\pm 0.005}$ & $0.112_{\pm 0.002}$ & $0.923_{\pm 0.002}$ \\
        & Laplace & $0.917_{\pm 0.002}$ & $0.997_{\pm 0.000}$ & $0.304_{\pm 0.009}$ & $0.043_{\pm 0.002}$ & $0.130_{\pm 0.004}$ & $0.918_{\pm 0.004}$ \\
        & SWAG & $0.918_{\pm 0.001}$ & $0.997_{\pm 0.001}$ & $0.251_{\pm 0.004}$ & $0.009_{\pm 0.000}$ & $0.121_{\pm 0.002}$ & $0.920_{\pm 0.002}$ \\
        \rowcolor{gray!10} \cellcolor{white}& \texttt{uCBOpt} $(\vartheta\to0^{+})$ & $0.915_{\pm 0.000}$ & $0.997_{\pm 0.001}$ & $0.294_{\pm 0.005}$ & $0.040_{\pm 0.002}$ & $0.131_{\pm 0.001}$ & $0.919_{\pm 0.001}$ \\
        \rowcolor{gray!10} \cellcolor{white}& \texttt{uCBOpt} $(\vartheta=5\cdot10^{-5})$ & $0.916_{\pm 0.003}$ & $0.997_{\pm 0.000}$ & $0.299_{\pm 0.015}$ & $0.040_{\pm 0.004}$ & $0.131_{\pm 0.005}$ & $0.919_{\pm 0.002}$ \\
         \rowcolor{gray!10} \cellcolor{white} & \texttt{uCBOpt-adapt} & $0.917_{\pm 0.002}$ & $0.997_{\pm 0.001}$ & $0.310_{\pm 0.023}$ & $0.045_{\pm 0.002}$ & $0.131_{\pm 0.004}$ & $0.920_{\pm 0.006}$ \\
        \rowcolor{gray!10} \cellcolor{white}& \texttt{lCBOpt-adapt} & $0.911_{\pm 0.002}$ & $0.997_{\pm 0.000}$ & $0.301_{\pm 0.003}$ & $0.039_{\pm 0.002}$ & $0.136_{\pm 0.002}$ & $0.916_{\pm 0.004}$ \\
        \cmidrule{1-8}
        \multirow{8}{*}{CIFAR-100}
        & SGD & $0.697_{\pm 0.001}$ & $0.907_{\pm 0.005}$ & $1.365_{\pm 0.021}$ & $0.150_{\pm 0.002}$ & $0.448_{\pm 0.000}$ & $0.856_{\pm 0.002}$ \\
        & IVON@mean & $0.647_{\pm 0.008}$ & $0.878_{\pm 0.004}$ & $1.684_{\pm 0.062}$ & $0.182_{\pm 0.007}$ & $0.527_{\pm 0.012}$ & $0.836_{\pm 0.001}$ \\
        & IVON & $0.674_{\pm 0.008}$ & $0.897_{\pm 0.007}$ & $1.307_{\pm 0.048}$ & $0.093_{\pm 0.004}$ & $0.451_{\pm 0.010}$ & $0.849_{\pm 0.003}$ \\
        & MC-D & $0.699_{\pm 0.006}$ & $0.909_{\pm 0.003}$ & $1.354_{\pm 0.016}$ & $0.149_{\pm 0.003}$ & $0.447_{\pm 0.007}$ & $0.854_{\pm 0.002}$ \\
        & \graycell{}\texttt{uCBOpt} $(\vartheta\to0^{+})$ & \graycell{}$0.701_{\pm 0.003}$ & $0.914_{\pm 0.001}$ & \graycell{}$1.278_{\pm 0.007}$ & \graycell{}$0.133_{\pm 0.002}$ & \graycell{}$0.438_{\pm 0.004}$ & \graycell{}$0.848_{\pm 0.001}$ \\
        \rowcolor{gray!10} \cellcolor{white}& \texttt{uCBOpt} $(\vartheta=5\cdot10^{-5})$ & $0.696_{\pm 0.010}$ & $0.912_{\pm 0.004}$ & $1.276_{\pm 0.032}$ & $0.136_{\pm 0.005}$ & $0.441_{\pm 0.011}$ & $0.854_{\pm 0.002}$ \\
         \rowcolor{gray!10} \cellcolor{white}& \texttt{uCBOpt-adapt} & $0.705_{\pm 0.002}$ & $0.913_{\pm 0.004}$ & $1.315_{\pm 0.023}$ & $0.145_{\pm 0.000}$ & $0.438_{\pm 0.001}$ & $0.854_{\pm 0.002}$ \\
        \rowcolor{gray!10} \cellcolor{white}& \texttt{lCBOpt-adapt} & $0.689_{\pm 0.002}$ & $0.912_{\pm 0.001}$ & $1.277_{\pm 0.022}$ & $0.127_{\pm 0.004}$ & $0.446_{\pm 0.004}$ & $0.850_{\pm 0.005}$ \\
    \bottomrule
    \end{tabular}}
\vskip -0.1in
\end{table*}

\begin{table*}[h!]
\caption{Experimental results on out-of-domain tasks over three runs (final epoch model).}
\label{tab:ood-final}
\centering
\setlength{\tabcolsep}{3pt}
    \resizebox{0.92\linewidth}{!}{\begin{tabular}{llccccc}
    \toprule
        Dataset & Optimizer & FPR@95\% $\downarrow$ & Det. Err. $\downarrow$ & AUROC $\uparrow$ & AUPR-In $\uparrow$ & AUPR-Out $\uparrow$\\
        \midrule
        \multirow{11}{*}{\makecell[l]{\textbf{In:} Fashion-MNIST \\ \textbf{Out:} EMNIST}}
        & AdamW & $0.686_{\pm 0.010}$ & $0.369_{\pm 0.008}$ & $0.647_{\pm 0.007}$ & $0.429_{\pm 0.008}$ & $0.761_{\pm 0.004}$ \\
        & SGD & $0.520_{\pm 0.016}$ & $0.368_{\pm 0.017}$ & $0.674_{\pm 0.022}$ & $0.508_{\pm 0.029}$ & $0.780_{\pm 0.014}$ \\
        & IVON@mean & $0.261_{\pm 0.021}$ & $0.268_{\pm 0.021}$ & $0.796_{\pm 0.024}$ & $0.708_{\pm 0.043}$ & $0.866_{\pm 0.015}$ \\
        & IVON & $0.236_{\pm 0.018}$ & $0.249_{\pm 0.022}$ & $0.820_{\pm 0.023}$ & $0.743_{\pm 0.039}$ & $0.883_{\pm 0.014}$ \\
        & MC-D & $0.268_{\pm 0.003}$ & $0.246_{\pm 0.003}$ & $0.813_{\pm 0.002}$ & $0.734_{\pm 0.013}$ & $0.875_{\pm 0.002}$ \\
        & Laplace & $0.519_{\pm 0.006}$ & $0.363_{\pm 0.012}$ & $0.680_{\pm 0.015}$ & $0.517_{\pm 0.014}$ & $0.784_{\pm 0.011}$ \\
        & SWAG & $0.378_{\pm 0.005}$ & $0.324_{\pm 0.001}$ & $0.731_{\pm 0.005}$ & $0.601_{\pm 0.027}$ & $0.820_{\pm 0.002}$ \\
        \rowcolor{gray!10} \cellcolor{white}& \texttt{uCBOpt} $(\vartheta\to0^{+})$ & $0.266_{\pm 0.017}$ & $0.267_{\pm 0.018}$ & $0.796_{\pm 0.019}$ & $0.700_{\pm 0.024}$ & $0.865_{\pm 0.015}$ \\
        \rowcolor{gray!10} \cellcolor{white}& \texttt{uCBOpt} $(\vartheta=8 \cdot 10^{-6})$ & $0.266_{\pm 0.007}$ & $0.271_{\pm 0.002}$ & $0.794_{\pm 0.005}$ & $0.694_{\pm 0.023}$ & $0.865_{\pm 0.002}$ \\
         \rowcolor{gray!10} \cellcolor{white}& \texttt{uCBOpt-adapt} & $0.272_{\pm 0.011}$ & $0.290_{\pm 0.015}$ & $0.773_{\pm 0.015}$ & $0.655_{\pm 0.016}$ & $0.857_{\pm 0.009}$ \\
        \rowcolor{gray!10} \cellcolor{white}& \texttt{lCBOpt-adapt} & $0.269_{\pm 0.016}$ & $0.278_{\pm 0.016}$ & $0.784_{\pm 0.019}$ & $0.674_{\pm 0.049}$ & $0.859_{\pm 0.011}$ \\
        \cmidrule{1-7}
        \multirow{11}{*}{\makecell[l]{\textbf{In:} CIFAR-10 \\ \textbf{Out:} SVHN}}
        & AdamW & $0.359_{\pm 0.052}$ & $0.224_{\pm 0.024}$ & $0.830_{\pm 0.027}$ & $0.755_{\pm 0.042}$ & $0.893_{\pm 0.017}$ \\
        & SGD & $0.231_{\pm 0.007}$ & $0.200_{\pm 0.003}$ & $0.858_{\pm 0.005}$ & $0.800_{\pm 0.007}$ & $0.913_{\pm 0.005}$ \\
        & IVON@mean & $0.251_{\pm 0.015}$ & $0.213_{\pm 0.011}$ & $0.842_{\pm 0.013}$ & $0.785_{\pm 0.018}$ & $0.900_{\pm 0.008}$ \\
        & IVON & $0.223_{\pm 0.010}$ & $0.193_{\pm 0.010}$ & $0.860_{\pm 0.010}$ & $0.816_{\pm 0.014}$ & $0.911_{\pm 0.007}$ \\
        & MC-D & $0.233_{\pm 0.025}$ & $0.210_{\pm 0.020}$ & $0.844_{\pm 0.021}$ & $0.792_{\pm 0.026}$ & $0.900_{\pm 0.013}$ \\
        & Laplace & $0.244_{\pm 0.019}$ & $0.200_{\pm 0.016}$ & $0.859_{\pm 0.015}$ & $0.801_{\pm 0.022}$ & $0.915_{\pm 0.010}$ \\
        & SWAG & $0.197_{\pm 0.012}$ & $0.175_{\pm 0.004}$ & $0.879_{\pm 0.007}$ & $0.841_{\pm 0.006}$ & $0.924_{\pm 0.007}$ \\
        \rowcolor{gray!10} \cellcolor{white}& \texttt{uCBOpt} $(\vartheta\to0^{+})$& $0.212_{\pm 0.022}$ & $0.187_{\pm 0.014}$ & $0.868_{\pm 0.016}$ & $0.820_{\pm 0.022}$ & $0.918_{\pm 0.011}$ \\
        \rowcolor{gray!10} \cellcolor{white}& \texttt{uCBOpt} $(\vartheta=5\cdot10^{-5})$ & $0.219_{\pm 0.008}$ & $0.196_{\pm 0.004}$ & $0.860_{\pm 0.004}$ & $0.804_{\pm 0.006}$ & $0.914_{\pm 0.003}$ \\
         \rowcolor{gray!10} \cellcolor{white} & \texttt{uCBOpt-adapt} & $0.239_{\pm 0.033}$ & $0.204_{\pm 0.017}$ & $0.854_{\pm 0.019}$ & $0.791_{\pm 0.026}$ & $0.911_{\pm 0.012}$ \\
        \rowcolor{gray!10} \cellcolor{white}& \texttt{lCBOpt-adapt} & $0.218_{\pm 0.008}$ & $0.196_{\pm 0.005}$ & $0.859_{\pm 0.005}$ & $0.806_{\pm 0.006}$ & $0.912_{\pm 0.004}$ \\
        \cmidrule{1-7}
        \multirow{8}{*}{\makecell[l]{\textbf{In:} CIFAR-100 \\ \textbf{Out:} TinyImageNet}}
        & SGD & $0.322_{\pm 0.003}$ & $0.315_{\pm 0.003}$ & $0.744_{\pm 0.003}$ & $0.767_{\pm 0.003}$ & $0.705_{\pm 0.004}$ \\
        & IVON@mean & $0.346_{\pm 0.001}$ & $0.336_{\pm 0.002}$ & $0.718_{\pm 0.002}$ & $0.744_{\pm 0.001}$ & $0.677_{\pm 0.003}$ \\
        & IVON & $0.342_{\pm 0.006}$ & $0.321_{\pm 0.003}$ & $0.739_{\pm 0.005}$ & $0.765_{\pm 0.003}$ & $0.698_{\pm 0.005}$ \\
        & MC-D & $0.326_{\pm 0.002}$ & $0.316_{\pm 0.005}$ & $0.738_{\pm 0.005}$ & $0.762_{\pm 0.006}$ & $0.697_{\pm 0.004}$ \\
        & \graycell{}\texttt{uCBOpt} $(\vartheta\to0^{+})$ & \graycell{}$0.331_{\pm 0.004}$ & \graycell{}$0.319_{\pm 0.004}$ & \graycell{}$0.739_{\pm 0.004}$ & \graycell{}$0.762_{\pm 0.003}$ & \graycell{}$0.700_{\pm 0.004}$ \\
        \rowcolor{gray!10} \cellcolor{white}& \texttt{uCBOpt} $(\vartheta=5\cdot10^{-5})$ & $0.330_{\pm 0.005}$ & $0.318_{\pm 0.003}$ & $0.739_{\pm 0.004}$ & $0.762_{\pm 0.002}$ & $0.701_{\pm 0.003}$ \\
         \rowcolor{gray!10} \cellcolor{white}& \texttt{uCBOpt-adapt} & $0.322_{\pm 0.001}$ & $0.312_{\pm 0.004}$ & $0.746_{\pm 0.002}$ & $0.769_{\pm 0.005}$ & $0.707_{\pm 0.002}$ \\
        \rowcolor{gray!10} \cellcolor{white}& \texttt{lCBOpt-adapt} & $0.341_{\pm 0.003}$ & $0.326_{\pm 0.003}$ & $0.732_{\pm 0.002}$ & $0.757_{\pm 0.001}$ & $0.692_{\pm 0.001}$ \\
    \bottomrule
    \end{tabular}}
\vskip -0.1in
\end{table*}

\newpage

\section{Ablation Studies}
\label{sec:ablation}

\subsection{Varying $\vartheta$ in \texttt{uCBOpt}}

We perform an ablation over the fixed candidate's curvature parameter, denoted by $\vartheta$ in \texttt{uCBOpt}. The $\vartheta\to0^{+}$ configuration is practically implemented as $\vartheta = 0$. The model is a LeNet trained on Fashion-MNIST. The results are summarized in the following \autoref{tab:ablation_vartheta}.

We observe that \texttt{uCBOpt} is relatively stable across a broad range of $\vartheta$ values, with only small variations in in-domain accuracy and calibration metrics. However, the choice of $\vartheta$ has a more noticeable effect on OOD detection performance. In particular, intermediate values around $\vartheta = 8 \cdot 10^{-6}$ achieve the strongest overall OOD performance, yielding the lowest FPR@95 and highest AUROC and AUPR-Out. Very small or larger values of $\vartheta$ tend to slightly degrade OOD detection quality, suggesting that an appropriate positive curvature floor can improve uncertainty estimation without significantly affecting in-domain predictive performance.

\begin{table*}[h!]
\caption{Effect of the $\vartheta$ parameter in \texttt{uCBOpt} on performance averaged over three runs.}
\label{tab:ablation_vartheta}
\centering
\setlength{\tabcolsep}{4pt}
\resizebox{0.85\linewidth}{!}{
\begin{tabular}{lcccccc}
\toprule
& \multicolumn{3}{c}{In-domain} & \multicolumn{3}{c}{Out-of-domain} \\
\cmidrule(lr){2-4}\cmidrule(lr){5-7}
$\vartheta$
& Acc. $\uparrow$ & NLL $\downarrow$ & ECE $\downarrow$ & FPR@95 $\downarrow$ & AUROC $\uparrow$ & AUPR-Out $\uparrow$ \\
\midrule
$0^{+}$ & $0.905_{\pm 0.002}$ & $0.269_{\pm 0.004}$ & $0.020_{\pm 0.006}$ & $0.262_{\pm 0.012}$ & $0.796_{\pm 0.014}$ & $0.867_{\pm 0.010}$ \\
$1 \cdot 10^{-7}$ & $0.905_{\pm 0.002}$ & $0.267_{\pm 0.001}$ & $0.019_{\pm 0.006}$ & $0.265_{\pm 0.008}$ & $0.786_{\pm 0.011}$ & $0.278_{\pm 0.007}$ \\
$2 \cdot 10^{-7}$ & $0.905_{\pm 0.002}$ & $0.271_{\pm 0.002}$ & $0.021_{\pm 0.007}$ & $0.255_{\pm 0.008}$ & $0.799_{\pm 0.009}$ & $0.871_{\pm 0.005}$ \\
$5 \cdot 10^{-7}$ & $0.906_{\pm 0.005}$ & $0.272_{\pm 0.006}$ & $0.021_{\pm 0.004}$ & $0.265_{\pm 0.023}$ & $0.786_{\pm 0.027}$ & $0.863_{\pm 0.020}$ \\
$8 \cdot 10^{-7}$ & $0.905_{\pm 0.001}$ & $0.269_{\pm 0.005}$ & $0.021_{\pm 0.002}$ & $0.272_{\pm 0.020}$ & $0.781_{\pm 0.024}$ & $0.856_{\pm 0.019}$ \\
$1 \cdot 10^{-6}$ & $0.907_{\pm 0.002}$ & $0.270_{\pm 0.002}$ & $0.023_{\pm 0.000}$ & $0.262_{\pm 0.012}$ & $0.796_{\pm 0.015}$ & $0.868_{\pm 0.010}$ \\
$2 \cdot 10^{-6}$ & $0.906_{\pm 0.003}$ & $0.273_{\pm 0.003}$ & $0.021_{\pm 0.002}$ & $0.278_{\pm 0.028}$ & $0.773_{\pm 0.031}$ & $0.852_{\pm 0.020}$ \\
$5 \cdot 10^{-6}$ & $0.908_{\pm 0.004}$ & $0.267_{\pm 0.002}$ & $0.020_{\pm 0.006}$ & $0.259_{\pm 0.010}$ & $0.795_{\pm 0.012}$ & $0.867_{\pm 0.009}$ \\
\rowcolor{gray!10}
$8 \cdot 10^{-6}$ & $0.906_{\pm 0.003}$ & $0.263_{\pm 0.006}$ & $0.018_{\pm 0.006}$ & $0.245_{\pm 0.010}$ & $0.810_{\pm 0.012}$ & $0.879_{\pm 0.010}$ \\
$1 \cdot 10^{-5}$ & $0.907_{\pm 0.003}$ & $0.269_{\pm 0.002}$ & $0.020_{\pm 0.005}$ & $0.275_{\pm 0.006}$ & $0.777_{\pm 0.009}$ & $0.284_{\pm 0.010}$ \\
\bottomrule
\end{tabular}}
\vskip -0.1in
\end{table*}

\subsection{Varying $\gamma$ and $\beta_3$ in \texttt{uCBOpt-adapt}}

We perform an ablation over the $\gamma$ and $\beta_3$ parameters in \texttt{uCBOpt-adapt}. The model is a LeNet trained on Fashion-MNIST. The results are summarized in the following \autoref{tab:ablation_gamma_ucbo} and \autoref{tab:ablation_beta3_ucbo}.

For the $\gamma$ ablation, we observe that moderate-to-large values of $\gamma$ generally improve in-domain performance, with $\gamma = 0.9$ achieving the best overall trade-off between accuracy, NLL, and calibration. However, when $\gamma$ becomes too large ($\gamma = 0.95$), performance deteriorates substantially, indicating that overly aggressive curvature correction can destabilize optimization. We also observe that OOD performance is relatively stable across most $\gamma$ values, though the best OOD metrics are obtained at smaller $\gamma$ values.

For the $\beta_3$ ablation, the results are comparatively stable across a broad range of values, suggesting that \texttt{uCBOpt-adapt} is not highly sensitive to the precise choice of $\beta_3$. Nevertheless, $\beta_3 = 1.001$ achieves the strongest overall performance, providing the best balance between in-domain accuracy, calibration, and OOD detection metrics. Extremely large values such as $\beta_3 = 1.1$ can negatively affect calibration, as reflected by the substantially higher ECE.

\begin{table*}[h!]
\caption{Effect of the $\gamma$ parameter in \texttt{uCBOpt-adapt} on performance averaged over three runs.}
\label{tab:ablation_gamma_ucbo}
\centering
\setlength{\tabcolsep}{4pt}
\resizebox{0.85\linewidth}{!}{
\begin{tabular}{lcccccc}
\toprule
& \multicolumn{3}{c}{In-domain} & \multicolumn{3}{c}{Out-of-domain} \\
\cmidrule(lr){2-4}\cmidrule(lr){5-7}
$\gamma$
& Acc. $\uparrow$ & NLL $\downarrow$ & ECE $\downarrow$ & FPR@95 $\downarrow$ & AUROC $\uparrow$ & AUPR-Out $\uparrow$ \\
\midrule
$0.0$ & $0.905_{\pm 0.002}$ & $0.269_{\pm 0.004}$ & $0.020_{\pm 0.006}$ & $0.262_{\pm 0.012}$ & $0.796_{\pm 0.014}$ & $0.867_{\pm 0.010}$ \\
$0.05$ & $0.907_{\pm 0.002}$ & $0.268_{\pm 0.004}$ & $0.021_{\pm 0.003}$ & $0.274_{\pm 0.013}$ & $0.790_{\pm 0.013}$ & $0.862_{\pm 0.008}$ \\
$0.1$ & $0.903_{\pm 0.002}$ & $0.276_{\pm 0.004}$ & $0.022_{\pm 0.001}$ & $0.295_{\pm 0.003}$ & $0.753_{\pm 0.003}$ & $0.841_{\pm 0.004}$ \\
$0.3$ & $0.905_{\pm 0.001}$ & $0.273_{\pm 0.007}$ & $0.022_{\pm 0.004}$ & $0.273_{\pm 0.007}$ & $0.778_{\pm 0.009}$ & $0.858_{\pm 0.006}$ \\
$0.5$ & $0.904_{\pm 0.001}$ & $0.266_{\pm 0.004}$ & $0.017_{\pm 0.003}$ & $0.279_{\pm 0.009}$ & $0.769_{\pm 0.010}$ & $0.853_{\pm 0.007}$ \\
$0.7$ & $0.909_{\pm 0.002}$ & $0.265_{\pm 0.004}$ & $0.025_{\pm 0.002}$ & $0.280_{\pm 0.015}$ & $0.773_{\pm 0.016}$ & $0.855_{\pm 0.009}$ \\
\rowcolor{gray!10}
$0.9$ & $0.910_{\pm 0.001}$ & $0.256_{\pm 0.003}$ & $0.016_{\pm 0.001}$ & $0.272_{\pm 0.011}$ & $0.773_{\pm 0.015}$ & $0.857_{\pm 0.009}$ \\
$0.95$ & $0.832_{\pm 0.005}$ & $0.460_{\pm 0.011}$ & $0.017_{\pm 0.009}$ & $0.299_{\pm 0.073}$ & $0.743_{\pm 0.092}$ & $0.849_{\pm 0.056}$ \\
\bottomrule
\end{tabular}}
\vskip -0.1in
\end{table*}

\begin{table*}[h!]
\caption{Effect of the $\beta_3$ parameter in \texttt{uCBOpt-adapt} on performance averaged over three runs.}
\label{tab:ablation_beta3_ucbo}
\centering
\setlength{\tabcolsep}{4pt}
\resizebox{0.85\linewidth}{!}{
\begin{tabular}{lcccccc}
\toprule
& \multicolumn{3}{c}{In-domain} & \multicolumn{3}{c}{Out-of-domain} \\
\cmidrule(lr){2-4}\cmidrule(lr){5-7}
$\beta_3$
& Acc. $\uparrow$ & NLL $\downarrow$ & ECE $\downarrow$ & FPR@95 $\downarrow$ & AUROC $\uparrow$ & AUPR-Out $\uparrow$ \\
\midrule
$1.00001$ & $0.908_{\pm 0.002}$ & $0.256_{\pm 0.008}$ & $0.017_{\pm 0.003}$ & $0.285_{\pm 0.014}$ & $0.767_{\pm 0.019}$ & $0.851_{\pm 0.009}$ \\
$1.0001$  & $0.909_{\pm 0.002}$ & $0.256_{\pm 0.001}$ & $0.020_{\pm 0.001}$ & $0.288_{\pm 0.004}$ & $0.753_{\pm 0.006}$ & $0.847_{\pm 0.003}$ \\
\rowcolor{gray!10}
$1.001$ & $0.910_{\pm 0.001}$ & $0.256_{\pm 0.003}$ & $0.016_{\pm 0.001}$ & $0.272_{\pm 0.011}$ & $0.773_{\pm 0.015}$ & $0.857_{\pm 0.009}$ \\
$1.01$ & $0.909_{\pm 0.002}$ & $0.256_{\pm 0.003}$ & $0.018_{\pm 0.003}$ & $0.277_{\pm 0.010}$ & $0.771_{\pm 0.013}$ & $0.853_{\pm 0.007}$ \\
$1.1$ & $0.909_{\pm 0.002}$ & $0.256_{\pm 0.003}$ & $0.133_{\pm 0.002}$ & $0.277_{\pm 0.010}$ & $0.771_{\pm 0.013}$ & $0.853_{\pm 0.007}$ \\
\bottomrule
\end{tabular}}
\vskip -0.1in
\end{table*}

\subsection{Varying $\gamma$ and $\beta_3$ in \texttt{lCBOpt-adapt}}

We perform an ablation over the $\gamma$ and $\beta_3$ parameters in \texttt{lCBOpt-adapt}. The model is a LeNet trained on Fashion-MNIST. The results are summarized in the following \autoref{tab:ablation_gamma_lcbo} and \autoref{tab:ablation_beta3_lcbo}.

For the $\gamma$ ablation, we observe that \texttt{lCBOpt-adapt} is sensitive to the choice of $\gamma$. Very small values close to $1$ (e.g., $\gamma = 1.005$) lead to substantially degraded optimization performance, while moderate values between $1.02$ and $1.10$ provide consistently strong in-domain and OOD results. In particular, $\gamma = 1.02$ achieves the best overall in-domain performance, whereas slightly larger values such as $\gamma = 1.05$ and $\gamma = 1.10$ provide stronger OOD detection metrics. Larger $\gamma$ values beyond this range gradually reduce performance.

For the $\beta_3$ ablation, all metrics remain effectively identical across the tested values, indicating that \texttt{lCBOpt-adapt} is largely insensitive to $\beta_3$ in this setting. This behaviour is likely due to the large $\beta_2$ value used in the optimizer, causing the curvature estimates to evolve very slowly and limiting the influence of the $\beta_3$ parameter.

\begin{table*}[h!]
\caption{Effect of the $\gamma$ parameter in \texttt{lCBOpt-adapt} on performance averaged over three runs.}
\label{tab:ablation_gamma_lcbo}
\centering
\setlength{\tabcolsep}{4pt}
\resizebox{0.85\linewidth}{!}{
\begin{tabular}{lcccccc}
\toprule
& \multicolumn{3}{c}{In-domain} & \multicolumn{3}{c}{Out-of-domain} \\
\cmidrule(lr){2-4}\cmidrule(lr){5-7}
$\gamma$
& Acc. $\uparrow$ & NLL $\downarrow$ & ECE $\downarrow$ & FPR@95 $\downarrow$ & AUROC $\uparrow$ & AUPR-Out $\uparrow$ \\
\midrule
$1.005$ & $0.838_{\pm 0.009}$ & $0.460_{\pm 0.036}$ & $0.020_{\pm 0.005}$ & $0.299_{\pm 0.016}$ & $0.744_{\pm 0.016}$ & $0.850_{\pm 0.007}$ \\
$1.01$  & $0.909_{\pm 0.002}$ & $0.253_{\pm 0.003}$ & $0.017_{\pm 0.001}$ & $0.282_{\pm 0.015}$ & $0.767_{\pm 0.019}$ & $0.849_{\pm 0.013}$ \\
\rowcolor{gray!10}
$1.02$  & $0.913_{\pm 0.002}$ & $0.252_{\pm 0.004}$ & $0.020_{\pm 0.000}$ & $0.269_{\pm 0.015}$ & $0.784_{\pm 0.019}$ & $0.859_{\pm 0.011}$ \\
$1.03$  & $0.911_{\pm 0.001}$ & $0.262_{\pm 0.001}$ & $0.024_{\pm 0.002}$ & $0.273_{\pm 0.009}$ & $0.780_{\pm 0.010}$ & $0.857_{\pm 0.004}$ \\
$1.05$  & $0.908_{\pm 0.001}$ & $0.267_{\pm 0.007}$ & $0.026_{\pm 0.005}$ & $0.261_{\pm 0.014}$ & $0.796_{\pm 0.014}$ & $0.865_{\pm 0.009}$ \\
$1.07$  & $0.907_{\pm 0.002}$ & $0.272_{\pm 0.005}$ & $0.022_{\pm 0.005}$ & $0.282_{\pm 0.004}$ & $0.769_{\pm 0.002}$ & $0.851_{\pm 0.004}$ \\
$1.10$  & $0.906_{\pm 0.002}$ & $0.268_{\pm 0.003}$ & $0.023_{\pm 0.004}$ & $0.258_{\pm 0.012}$ & $0.795_{\pm 0.015}$ & $0.868_{\pm 0.010}$ \\
$1.20$  & $0.907_{\pm 0.002}$ & $0.272_{\pm 0.006}$ & $0.020_{\pm 0.006}$ & $0.267_{\pm 0.014}$ & $0.790_{\pm 0.019}$ & $0.862_{\pm 0.013}$ \\
$1.50$  & $0.904_{\pm 0.003}$ & $0.274_{\pm 0.005}$ & $0.019_{\pm 0.001}$ & $0.279_{\pm 0.006}$ & $0.775_{\pm 0.005}$ & $0.852_{\pm 0.004}$ \\
$2.0$   & $0.901_{\pm 0.002}$ & $0.282_{\pm 0.005}$ & $0.014_{\pm 0.002}$ & $0.283_{\pm 0.008}$ & $0.763_{\pm 0.011}$ & $0.847_{\pm 0.009}$ \\
\bottomrule
\end{tabular}}
\vskip -0.1in
\end{table*}

\begin{table*}[h!]
\caption{Effect of the $\beta_3$ parameter in \texttt{lCBOpt-adapt} on performance averaged over three runs.}
\label{tab:ablation_beta3_lcbo}
\centering
\setlength{\tabcolsep}{4pt}
\resizebox{0.85\linewidth}{!}{
\begin{tabular}{lcccccc}
\toprule
& \multicolumn{3}{c}{In-domain} & \multicolumn{3}{c}{Out-of-domain} \\
\cmidrule(lr){2-4}\cmidrule(lr){5-7}
$\beta_3$
& Acc. $\uparrow$ & NLL $\downarrow$ & ECE $\downarrow$ & FPR@95 $\downarrow$ & AUROC $\uparrow$ & AUPR-Out $\uparrow$ \\
\midrule
$0.9$ & $0.913_{\pm 0.002}$ & $0.252_{\pm 0.004}$ & $0.020_{\pm 0.000}$ & $0.269_{\pm 0.015}$ & $0.784_{\pm 0.019}$ & $0.859_{\pm 0.011}$ \\
$0.99$  & $0.913_{\pm 0.002}$ & $0.252_{\pm 0.004}$ & $0.020_{\pm 0.000}$ & $0.269_{\pm 0.015}$ & $0.784_{\pm 0.019}$ & $0.859_{\pm 0.011}$ \\
\rowcolor{gray!10}
$0.999$ & $0.913_{\pm 0.002}$ & $0.252_{\pm 0.004}$ & $0.020_{\pm 0.000}$ & $0.269_{\pm 0.015}$ & $0.784_{\pm 0.019}$ & $0.859_{\pm 0.011}$ \\
$0.9999$ & $0.913_{\pm 0.002}$ & $0.252_{\pm 0.004}$ & $0.020_{\pm 0.000}$ & $0.269_{\pm 0.015}$ & $0.784_{\pm 0.019}$ & $0.859_{\pm 0.011}$ \\
$0.99999$ & $0.913_{\pm 0.002}$ & $0.252_{\pm 0.004}$ & $0.020_{\pm 0.000}$ & $0.269_{\pm 0.015}$ & $0.784_{\pm 0.019}$ & $0.859_{\pm 0.011}$ \\
\bottomrule
\end{tabular}}
\vskip -0.1in
\end{table*}

\newpage

\section{Additional Experiment: Combining \texttt{CBOpt} with MC Dropout}
\label{sec:mcdrop-cbopt}

We further evaluate whether \texttt{CBOpt} can be combined with MC Dropout. The experiment is conducted on ResNet-20 trained on CIFAR-10, with SVHN used as the out-of-domain dataset. The results are reported in \autoref{tab:mcdrop-cbopt}.

Combining \texttt{CBOpt} with MC Dropout generally preserves the strong in-domain performance of MC Dropout while improving several OOD detection metrics (bolded results). In particular, \texttt{uCBOpt-adapt} + MC-D achieves the best in-domain accuracy, NLL, and ECE among the reported methods, suggesting that the adaptive update remains beneficial even when stochastic dropout inference is used. For OOD detection, the \texttt{CBOpt} + MC-D variants improve AUROC and AUPR-Out over standard MC Dropout, with \texttt{lCBOpt-adapt} + MC-D achieving the highest AUROC and AUPR-Out. The fixed-$\vartheta$ \texttt{uCBOpt} variant also substantially reduces FPR@95, although its AUROC improvement is more modest. Overall, these results suggest that \texttt{CBOpt} and MC Dropout are complementary: MC Dropout improves calibration via stochastic inference, whereas the \texttt{CBOpt} update can further improve OOD separability.

\begin{table*}[h!]
\caption{Experimental results for ResNet-20/CIFAR-10 with \texttt{CBOpt} combined with MC dropout (best validation model).}
\label{tab:mcdrop-cbopt}
\centering
\setlength{\tabcolsep}{4pt}
\resizebox{\linewidth}{!}{
\begin{tabular}{lcccccc}
\toprule
& \multicolumn{3}{c}{In-domain} & \multicolumn{3}{c}{Out-of-domain} \\
\cmidrule(lr){2-4}\cmidrule(lr){5-7}
Optimizer
& Acc. $\uparrow$ & NLL $\downarrow$ & ECE $\downarrow$ & FPR@95 $\downarrow$ & AUROC $\uparrow$ & AUPR-Out $\uparrow$ \\
\midrule
MC-D & $0.923_{\pm 0.002}$ & $0.230_{\pm 0.006}$ & $0.010_{\pm 0.002}$ & $0.230_{\pm 0.030}$ & $0.847_{\pm 0.024}$ & $0.902_{\pm 0.017}$ \\
\rowcolor{gray!10}
\texttt{uCBOpt} $(\vartheta\to0^{+})$ + MC-D & $0.921_{\pm 0.003}$ & $0.236_{\pm 0.005}$ & $0.010_{\pm 0.002}$ & $0.237_{\pm 0.023}$ & $0.859_{\pm 0.017}$ & $\bm{0.906}_{\pm 0.013}$ \\
\rowcolor{gray!10}
\texttt{uCBOpt} $(\vartheta=5\cdot10^{-5})$ + MC-D & $0.921_{\pm 0.005}$ & $0.239_{\pm 0.007}$ & $0.010_{\pm 0.003}$ & $\bm{0.192}_{\pm 0.019}$ & $0.854_{\pm 0.023}$ & $0.903_{\pm 0.015}$ \\
\rowcolor{gray!10}
\texttt{uCBOpt-adapt} + MC-D & $0.924_{\pm 0.002}$ & $0.228_{\pm 0.003}$ & $0.007_{\pm 0.001}$ & $0.220_{\pm 0.013}$ & $\bm{0.861}_{\pm 0.009}$ & $\bm{0.908}_{\pm 0.008}$ \\
\rowcolor{gray!10}
\texttt{lCBOpt-adapt} + MC-D & $0.918_{\pm 0.001}$ & $0.241_{\pm 0.002}$ & $0.010_{\pm 0.001}$ & $0.230_{\pm 0.012}$ & $\bm{0.867}_{\pm 0.007}$ & $\bm{0.911}_{\pm 0.007}$ \\
\bottomrule
\end{tabular}}
\vskip -0.1in
\end{table*}

\end{document}